\pdfoutput=1
\documentclass{article}

\usepackage{iclr2027_conference,times}
\iclrfinalcopy

\usepackage[utf8]{inputenc}
\usepackage[T1]{fontenc}
\usepackage{hyperref}
\usepackage{url}
\input{glyphtounicode}
\usepackage{amsmath}
\usepackage{amssymb}
\usepackage{amsfonts}
\usepackage{amsthm}
\usepackage{mathtools}
\usepackage{booktabs}
\usepackage{graphicx}
\usepackage{array}
\usepackage{xcolor}
\usepackage{microtype}
\usepackage{enumitem}
\usepackage{algorithm}
\usepackage{algpseudocode}
\makeatletter
\newcommand{\algmargin}{\the\ALG@thistlm}
\makeatother
\algnewcommand{\LongState}[1]{\State\parbox[t]{\dimexpr\linewidth-\algmargin\relax}{\strut #1\strut}}
\newcommand{\InComment}[1]{\hspace*{\fill}\(\triangleright\)~#1}
\usepackage{tikz}
\usepackage[font=small,labelfont=bf,skip=4pt]{caption}
\usepackage{placeins}
\usetikzlibrary{positioning,arrows.meta,shapes.geometric,fit,calc,backgrounds}

\theoremstyle{plain}
\newtheorem{theorem}{Theorem}[section]
\newtheorem{lemma}[theorem]{Lemma}
\newtheorem{proposition}[theorem]{Proposition}
\newtheorem{corollary}[theorem]{Corollary}
\theoremstyle{definition}

\newtheorem{assumption}{Assumption}
\theoremstyle{remark}
\newtheorem{remark}[theorem]{Remark}

\newcommand{\E}{\mathbb{E}}
\newcommand{\Pbb}{\mathbb{P}}
\newcommand{\X}{\mathcal{X}}
\newcommand{\Y}{\mathcal{Y}}
\newcommand{\F}{\mathcal{F}}
\newcommand{\Fpre}{\F_t^{\mathrm{pre}}}
\newcommand{\Hcal}{\mathcal{H}}
\newcommand{\Kcal}{\mathcal{K}}
\newcommand{\Ccal}{\mathcal{C}}
\newcommand{\one}{\mathbf{1}}
\newcommand{\caldelta}{\delta_{\mathrm{cal}}}
\newcommand{\edelta}{\delta_{e}}
\newcommand{\rcal}{\delta_r^{\mathrm{cal}}}
\newcolumntype{P}[1]{>{\raggedright\arraybackslash}p{#1}}

\newcommand{\rev}[1]{{#1}}
\newcommand{\KL}{\operatorname{KL}}
\newcommand{\Bern}{\operatorname{Bern}}
\newcommand{\Q}{\mathbb{Q}}

\title{Certified Adaptive Refresh: Anytime-Valid\\ Monitoring for Federated Conformal RAG}
\author{Prasanjit Dubey \& Xiaoming Huo \\
H. Milton Stewart School of Industrial and Systems Engineering,\\
Georgia Institute of Technology, Atlanta, GA 30332, U.S.A.}
\hypersetup{
  pdftitle={Certified Adaptive Refresh: Anytime-Valid Monitoring for Federated Conformal RAG},
  pdfauthor={Prasanjit Dubey and Xiaoming Huo}
}

\begin{document}

\maketitle
\fancyhead{}
\renewcommand{\headrulewidth}{0pt}

\begin{abstract}
Question-answering services built on retrieval-augmented generation (RAG), in which a language model answers from retrieved documents, are inspected continuously and upgraded repeatedly, so their reliability guarantee must survive both. We study federated conformal RAG: retrieval nodes holding private corpora score a shared list of candidate answers with a shared language model and send only compressed scores to a hub, which returns an answer set; a \emph{miss} is a set that omits the true answer. We formulate monitoring as a sequential test: raise an alarm when misses exceed a certified upper bound on the miss rate (the \emph{allowance}), while keeping the probability of any false alarm over the whole deployment below a preset budget; an alarm is false when no deployed configuration has drifted above its allowance. Existing tools solve only parts of this problem. A conformal certificate bounds the miss rate of one frozen configuration at one look fixed in advance, so it supplies no \emph{anytime-valid} alarm, one that stays valid however often it is checked; a threshold calibrated for one model certifies nothing about its replacement; and re-certifying each upgrade at the full error level compounds failures. We propose Anytime-FC-RAG, a protocol with three auditable rules: pick the next configuration \emph{before} drawing its fresh calibration data, charge every certificate attempt to one trajectory-wide budget $\caldelta$, and fix the deployed threshold, allowance and monitoring bet before each query arrives. One betting wealth process $E_t$, evidence that grows in expectation only when misses exceed the allowance, then runs across every deployment without reset. Our \emph{certified adaptive-refresh composition theorem} shows that, for an evidence budget $\edelta$, alarming when $E_t$ first reaches $1/\edelta$ has false-alarm probability $\Pbb(\sup_{t\ge1} E_t \ge 1/\edelta) \le \edelta + \caldelta$ under a random number of history-selected, evidence-driven refreshes of the model, retriever, corpus, score or threshold; $\caldelta$ pays for certificates that are wrong by chance. With an exact order-statistic certificate the guarantee is distribution-free and architecture-agnostic; it concerns misses against the allowance, not per-input coverage or distribution shift as such. Replaying \rev{cached} Qwen2.5 \rev{scores} on MMLU-Pro at a $5\%$ total budget, all $210$ \rev{replay paths} whose shift an independent audit rated harmful alarm within $5{,}000$ queries, and none of \rev{the} $222$ benign or unshifted \rev{paths} does. In a swap of the deployed model, calibrating the incoming model on fresh data holds a downgrade's miss rate at \rev{$.099$ on average over redrawn calibration blocks (exactly the conformal value $10/101$)}, under its $.10$ target, whereas calibrating it on the outgoing model's scores lets it rise to \rev{$.164$. On a second benchmark with retriever, corpus and hub-rule refreshes chosen from history, every budget holds, and never resetting detects changes that budgeted restarts miss.}
\end{abstract}

\section{Introduction}
\label{sec:intro}

\paragraph{Motivating application.}
A deployed language-model service needs a reliability guarantee that survives the two things its operators actually do: look at it repeatedly, and change it. In retrieval-augmented generation (RAG), a language model answers a question using documents retrieved from a corpus. Consider a domain-specialized question-answering service whose corpus is split across several \emph{retrieval nodes} that cannot pool their documents, each holding a private corpus and a copy of a shared scoring model. For each question every node scores a common finite list of candidate answers against its locally retrieved context and uploads only compressed scores, not documents, to a central \emph{hub}, which returns an answer set; once the correct answer is adjudicated, the operator learns whether that set \emph{missed} it. Persistent misses invite action (widen the communication budget, re-fit the retriever, swap the scoring model, recalibrate the threshold), and each changes the monitored system.

\paragraph{The challenge.}
The standard guarantee for such a service breaks under both repeated inspection and change. That guarantee comes from \emph{conformal prediction}, which sets a score threshold on held-out labeled questions so that the returned sets omit the true answer at most a target fraction of the time. In its calibration-conditional form it supplies a \emph{certificate}: a high-probability promise that one configuration's miss rate stays under a stated \emph{allowance}. Three difficulties arise. First, a certificate describes one look fixed in advance, not repeated checks that stop when misses seem high, so it supplies no \emph{anytime-valid} alarm, one whose false-alarm guarantee holds however often and whenever the operator looks. Second, a threshold calibrated for one model certifies nothing about the model that replaces it. Third, re-certifying at the full error level after each change compounds failures: with $100$ independent certificates at $5\%$ each, the chance that any is wrong is $1-.95^{100}=.9941$, against $.0488$--$.0490$ when the same $5\%$ is split across them (Figure~\ref{fig:evidence}(a)).

\paragraph{Problem formulation.}
We formulate reliability monitoring as a sequential test of one hypothesis over the whole deployment: every deployed configuration keeps its miss rate within its certified allowance. The operator, who may replace any component between questions using everything observed so far, wants an alarm whose probability of ever firing while this hypothesis holds (a \emph{false alarm}) is at most a preset budget $\delta$; Section~\ref{sec:setup} makes this precise.

\paragraph{State of the art.}
To our knowledge, no existing family of methods solves this problem, because each certifies only a part of it (Table~\ref{tab:capability} in Appendix~\ref{app:related}). Federated conformal RAG (FC-RAG)~\citep{fcrag2026} and other federated conformal methods~\citep{lu2023federated,plassier2023conformal,wen2026gcfcp,xu2025fedccp} calibrate across nodes for one fixed configuration and raise no sequential alarm. Conformal test martingales (CTM)~\citep{vovk2021retrain,prinster2025watch} are anytime-valid but test \emph{exchangeability} of the score stream (roughly, that its distribution has not changed), so they can alarm on score changes that leave every miss intact. Sequential risk monitors~\citep{podkopaev2022tracking,timans2025monitoring} are anytime-valid only for one \emph{fixed} deployed model. Online and adaptive conformal methods~\citep{gibbs2021adaptive,angelopoulos2024online} re-adjust the threshold and so have no certificate to preserve, and post-selection conformal calibration~\citep{bao2025cap} calibrates after an adaptive pick without budgeting a trajectory of refreshes.

\paragraph{What we propose.}
We propose \emph{Anytime-FC-RAG}, a deployment protocol that lets the operator change the system as often as desired, provided it meets three conditions that can be audited after the fact. \emph{Pick, then calibrate}: the next configuration is selected from past history, and only afterwards is a fresh calibration block (new labeled questions) drawn for it. \emph{Budget every attempt}: each certificate attempt, including one computed and then discarded, spends from a single trajectory-wide budget $\caldelta$. \emph{Act before the outcome}: the deployed configuration, threshold, communication action, allowance and bet are all fixed before the next question arrives. In order,
\begin{equation*}
 \text{select configuration}\ \longrightarrow\ \text{fresh calibration}\
 \longrightarrow\ \text{certificate {+} threshold}\ \longrightarrow\
 \text{deploy and monitor} .
\end{equation*}
Monitoring uses one \emph{betting wealth process} $(E_t)_{t\ge0}$, where $E_t$ is the wealth, after the $t$-th question, of a gambler who starts with one unit ($E_0=1$) and at each question stakes part of it on the claim that the system misses more often than its allowance permits. The wealth's expected value cannot grow while the claim is false, so large wealth is evidence of excess misses. The same process runs across every configuration and is never reset, and the alarm fires when $E_t$ first reaches $1/\edelta$, for an \emph{evidence budget} $\edelta$.

\paragraph{Advantages of the proposed method.}
The main advantage is a \rev{false-alarm} guarantee that does not weaken with the number of inspections or refreshes. Our composition theorem shows that the alarm is anytime-valid, with its false-alarm probability split between the evidence budget and the calibration budget:
\begin{equation}
 \Pbb\!\left(\sup_{t\ge1}E_t\ \ge\ 1/\edelta\right)\ \le\ \edelta+\caldelta ,
 \label{eq:headline}
\end{equation}
however many times, and at whatever data-dependent moments, the operator refreshes the system. The number of refreshes may be random and unbounded, each new configuration may depend on everything observed so far, and no independence across refreshes is assumed; an operator who changes nothing pays $\caldelta$ once, however long the deployment runs. Three further advantages follow. The certificate we use is exact in finite samples and distribution-free, and it assumes nothing about how the scoring model was built, so any model or retriever can be swapped in. The monitor targets misses rather than scores, so changes that leave reliability intact keep its false-alarm rate within budget. And the guarantee sees only a certificate, an allowance and a budget, so it transfers to any sequential conformal protocol that supplies them before each query. On \rev{replayed} Qwen2.5 scores for the MMLU-Pro benchmark the alarm \rev{fired on every audited-harmful shift and no benign one} at a $5\%$ budget\rev{, and on a second benchmark every budget held while every component was refreshed} (Section~\ref{sec:evidence}).

\paragraph{What is inherited, what we use, and what is new.}
The composition is new. We inherit FC-RAG's fixed-configuration inference and communication layer~\citep{fcrag2026} and use betting e-processes (nonnegative evidence processes whose mean is at most one at any stopping time), Ville's inequality (Section~\ref{sec:theorem}), conditional union bounds and mixture boundaries~\citep{ville1939,howard2020time,waudby2024betting,ramdas2023game}. One frozen certified configuration monitored by one e-process is a special case of our result (Corollary~\ref{cor:design}(a)), and we do not claim it as new. What is new is the \emph{certified composition} of these pieces across history-selected configurations, which is why Anytime-FC-RAG is the only row of Table~\ref{tab:capability} with every column checked.

\rev{\paragraph{The pieces do not compose by themselves.}
Conditioning on every certificate being correct and then applying Ville's inequality is invalid when attempts depend on earlier misses (Proposition~\ref{prop:shortcut}); a budgeted restart at each refresh is valid but discards evidence that one never-reset process keeps (Section~\ref{sec:gain}). The guarantee needs no federation; federated RAG is its application.}

\paragraph{Organization.}
Section~\ref{sec:setup} formulates the problem, Section~\ref{sec:method} presents Anytime-FC-RAG, Section~\ref{sec:theory} establishes its guarantees, and Section~\ref{sec:evidence} tests them in simulation \rev{and on real-model scores}.

\section{Problem formulation}
\label{sec:setup}

\paragraph{One round of bandwidth-constrained federated RAG.}
The monitored system answers a question by pooling compressed scores from the nodes rather than their documents. Let $\X$ be the space of questions and $\Y$ a finite, agreed-upon answer list. A synchronous star network has $K<\infty$ retrieval \emph{nodes}, indexed by $i=1,\dots,K$, and one \emph{hub}. Node $i$ is the tuple $(\Ccal_{i},R_{i},\widehat P,\mathrm{Enc}_{B_i})$: a private corpus $\Ccal_i$, a local retriever $R_i$, an inference copy $\widehat P$ of the shared \emph{student} (following FC-RAG's terminology, the one scoring model deployed at every node), and an encoder $\mathrm{Enc}_{B_i}$ that compresses each score to $B_i$ bits. Corpora and retrievers differ across nodes, but the student does not. Queries are indexed by $t=1,2,\dots$. At query $t$ a question $X_t\in\X$ arrives, whose correct answer $Y_t\in\Y$ is revealed after the answer set is returned. Node $i$ retrieves the context $Z_{i,t}=R_i(X_t,\Ccal_i)$ and scores every candidate answer $y\in\Y$ by a \emph{nonconformity score} $s_{i,t}(y)$, a number that is larger the less compatible $y$ is with $X_t$ and $Z_{i,t}$; it then quantizes the scores with $\mathrm{Enc}_{B_i}$ and uploads them. The hub aggregates the decoded scores into one score $s_t(y)$ per candidate, compares it with the active threshold $q$, and returns the answer set $C_t(X_t)$; once $Y_t$ is known, it records the miss indicator $M_t$:
\begin{equation*}
 C_t(X_t)=\{y\in\Y:\,s_t(y)\le q\},
 \qquad
 M_t=\one\{Y_t\notin C_t(X_t)\}\in\{0,1\}.
\end{equation*}
Here $\one\{\cdot\}$ is the indicator function. The \emph{coverage event} is $\{M_t=0\}$, and $M_t$ is the \emph{miss} the operator monitors. Figure~\ref{fig:protocol} shows the round. Raw documents, training examples, gradients and model weights are not part of this uplink; this is a message-content property of the protocol, not a privacy guarantee.

\paragraph{Quantization costs the guarantee nothing by default.}
By default the compressed, aggregated score is calibrated directly, so quantization is already reflected in the threshold. Alternatively, a certificate can cover a class of bit budgets by adding to its allowance a communication term $\Delta_{\mathrm{comm}}$ that, under FC-RAG's dithered quantizer (one that adds node-independent random noise before rounding), shrinks with more bits and as $K^{-1/2}$ in the number of nodes, whereas one node gains nothing from averaging (Appendix~\ref{app:communication}; a model property, as our system experiments use $K\le8$).

\paragraph{Configurations, attempts and epochs.}
A deployment is a sequence of configurations, each certified on its own calibration data and served until it is replaced. A \emph{score state} $\theta$ makes ``configuration'' precise: it collects every component that determines the deployed score, namely student model, retriever and corpus, candidate rule, score semantics, hub aggregation, threshold rule, and any class of communication actions (e.g.\ per-node bit budgets) the certificate covers uniformly, meaning its allowance holds whichever action in the class is used. \emph{Calibration attempts} are indexed by $r=1,2,\dots$; attempt $r$ proposes a state $\theta_r$, acquires a fresh \emph{calibration block} for it (new labeled questions not used to build or choose the state), and from that block computes a \emph{certificate}, a high-probability bound on the miss rate that state will incur. An \emph{epoch} is a maximal run of queries served by one state $\theta_r$, its realized threshold $q_r$ and its matching certificate, whose allowance $u_r$ is defined in Section~\ref{sec:construct}; $r_t$ is the attempt active at query $t$, and $a_t$ is the communication action used at query $t$, fixed before the query. Table~\ref{tab:glossary} (Appendix~\ref{app:protocol}) collects the notation.

\begin{figure}[t]
\begin{center}
\begin{tikzpicture}[
  >={Stealth[length=1.6mm]},
  font=\scriptsize,
  nodebox/.style={draw, rounded corners=1.5pt, fill=blue!7, inner sep=2.2pt,
                  align=center, minimum width=2.45cm, minimum height=0.40cm},
  hubbox/.style={draw, rounded corners=1.5pt, fill=orange!13, inner sep=3pt,
                 align=center, minimum width=3.15cm},
  outbox/.style={draw, rounded corners=1.5pt, fill=green!10, inner sep=3pt,
                 align=center, minimum width=2.45cm},
  actbox/.style={draw, rounded corners=1.5pt, fill=red!7, inner sep=3pt,
                 align=center},
  obsbox/.style={draw, rounded corners=1.5pt, fill=blue!7, inner sep=3pt,
                 align=center},
  notebox/.style={draw, dashed, rounded corners=1.5pt, fill=black!3,
                  inner sep=2.6pt, align=center},
  arr/.style={->, semithick, black!55},
  lbl/.style={font=\scriptsize\itshape, black!70, inner sep=1pt},
  ptitle/.style={font=\scriptsize\bfseries, anchor=west}
]
\node[nodebox] (n1) at (0.45,0.70)  {node $1$:\ \ $\mathcal C_1,R_1,\widehat P$};
\node[nodebox] (n2) at (0.45,0.10)  {$\vdots$};
\node[nodebox] (nK) at (0.45,-0.50) {node $K$:\ \ $\mathcal C_K,R_K,\widehat P$};
\node[hubbox]  (hub) at (5.05,0.10)
  {\textbf{hub}: aggregate to $s_t(y)$\\[0.6pt]
   $C_t(X_t)=\{y:s_t(y)\le q_{r_t}\}$};
\node[outbox] (out) at (9.30,0.10)
  {answer set $C_t(X_t)$\\[0.6pt] label $Y_t\Rightarrow$ miss $M_t$};
\node[lbl] (qin) at (-2.00,0.10) {question $X_t$};
\draw[arr] (qin.east) -- (n1.west);
\draw[arr] (qin.east) -- (n2.west);
\draw[arr] (qin.east) -- (nK.west);
\draw[arr] (n1.east) -- (hub.north west);
\draw[arr] (n2.east) -- (hub.west);
\draw[arr] (nK.east) -- (hub.south west);
\node[lbl] at (3.40,0.95) {compressed scores ($B_i$ bits)};
\draw[arr] (hub.east) -- (out.west);
\node[obsbox] (e1) at (-1.30,-1.20) {epoch $1$: $\theta_1,q_1,u_1$};
\node[actbox, right=0.40cm of e1] (rf) {refresh: select $\theta_2$ $\to$ fresh block $\to$ certify at $\delta^{\mathrm{cal}}_2$ $\to$ deploy};
\node[obsbox, right=0.40cm of rf] (e2) {epoch $2$: $\theta_2,q_2,u_2$};
\node[right=0.30cm of e2, inner sep=1pt] (al) {$\cdots$};
\draw[arr] (e1.east) -- (rf.west);
\draw[arr] (rf.east) -- (e2.west);
\draw[arr] (e2.east) -- (al.west);
\draw[arr] (-2.45,-1.66) -- node[lbl, fill=white] {one wealth process $E_t$ runs across every epoch, never reset; alarm when $E_t\ge1/\edelta$} (10.45,-1.66);
\end{tikzpicture}
\end{center}
\caption{Top: the monitored system. Every node receives the question and scores
the answer list $\Y$ against its private corpus with the shared scoring model
$\widehat P$, then uploads compressed scores; the hub returns $C_t(X_t)$, and
the adjudicated $Y_t$ yields the miss $M_t$. Bottom: a deployment trajectory;
epoch $r$ serves state $\theta_r$ with threshold $q_r$ and allowance $u_r$,
certified on a fresh block with budget share $\delta^{\mathrm{cal}}_r$
(Section~\ref{sec:method}, Algorithm~\ref{alg:protocol}).}
\label{fig:protocol}
\end{figure}

\paragraph{What is known, and when.}
Three nested collections of information fix what each decision may depend on. All randomness lives on one probability space $(\Omega,\mathcal A,\Pbb)$, and each collection of information is a $\sigma$-field, a family of events whose occurrence is known at that moment. Let $\F_0$ contain the pre-deployment history. For each query $t\ge1$,
\begin{equation}
 \F_{t-1}\ \subseteq\ \Fpre\ \subseteq\ \F_t .
 \label{eq:filtration}
\end{equation}
Here $\F_{t-1}$ holds everything known after query $t-1$; the \emph{pre-query field} $\Fpre$ adds every intervening operator decision, state selection, calibration block, certificate, threshold, action, allowance and bet, but closes \emph{before} $(X_t,Y_t)$; and $\F_t$ adds that pair, the returned set and $M_t$. A quantity is \emph{fixed before query $t$} when it is determined by $\Fpre$ ($\Fpre$-measurable). Evaluating a rule fixed in this sense at $X_t$ is allowed and necessary; choosing any of these quantities after seeing $X_t$ is not.

\paragraph{The deployment null.}
The hypothesis under test is that answer sets keep honoring their certified allowance; breaching it, not any change in the data, is the operational failure. Write $b_t$ for the \emph{allowance} of query $t$, the certified upper bound on that query's miss rate, fixed before the query. The allowance bounds a miss rate; $\edelta$ and $\caldelta$ bound error probabilities. The \emph{deployment null} is that $\E[M_t\mid\Fpre]\le b_t$ for every query $t$ whose certificate is correct (formally, Equation~\eqref{eq:onestep} in Section~\ref{sec:interface}). The data may change so long as this holds, and certificates that are wrong by chance are paid for separately by $\caldelta$. Two consequences follow: coverage is \emph{marginal over a fresh question}, not conditional on each particular question value $X=x$, and a change that leaves certified reliability intact should leave the monitor silent. That silence is correct, not a missed detection.

\paragraph{The goal.}
We seek a certificate for each configuration and an alarm rule such that, under the deployment null, the probability that the alarm \emph{ever} fires is at most a preset total budget $\delta$, for every refresh policy that obeys the timing constraints above, while sustained excess misses are detected. The number of refreshes may be random, and each new state may depend on all history.

\paragraph{Scope.}
The guarantee covers the round just described (finite fixed $K$, synchronous star, finite answer list, marginal coverage for a fresh question) together with a history-selected state followed by fresh state-matched calibration, and actions fixed before each outcome, with \emph{sticky} (never reset) evidence. Each query the monitor bets on needs its adjudicated answer; a query chosen in advance to go unlabeled gets a zero bet, and late labels are covered for a fixed state (Corollary~\ref{cor:design}(e)). Asynchronous, peer-to-peer, Byzantine or privacy-preserving federation, open-ended generation, per-$x$ coverage, general shift detection, recovery after an alarm and optimal control are outside it (Table~\ref{tab:scope}, Appendix~\ref{app:protocol}).

\section{Proposed method: Anytime-FC-RAG}
\label{sec:method}

Sections~\ref{sec:construct}--\ref{sec:interface} implement the three conditions of Section~\ref{sec:intro} in order; Section~\ref{sec:monitor} adds the monitor.

\subsection{Pick, then calibrate: a certificate for each configuration}
\label{sec:construct}

\paragraph{Each state is certified on a block drawn after the state is chosen.}
Before attempt $r$, the decision to attempt, the state $\theta_r$ it proposes, the calibration size, the threshold rule and the attempt's error allocation are all fixed ($\Hcal_r$-measurable) by the \emph{pre-calibration history} $\Hcal_r$, everything observed before the attempt. Only then is a \emph{conditionally honest} block drawn, one not used to train or select $\theta_r$. Absorb a fixed communication action into $\theta_r$ and write $s_{\theta_r}(X,Y)$ for the aggregate score that state $\theta_r$ gives the \emph{true} answer $Y$ of a question $X$, so $M_t=\one\{s_{\theta_{r_t}}(X_t,Y_t)>q_{r_t}\}$; let $F_r$ be its distribution function for a fresh calibration question, given $\Hcal_r$.

\paragraph{The threshold is a split-conformal order statistic, and the allowance is its exact Beta bound.}
Split conformal prediction sets the threshold at the $k$-th smallest of $n$ true-answer scores on held-out labeled questions; a fresh question is then missed with a probability that depends on the calibration draw, and $k=\lceil(n+1)(1-\alpha)\rceil$ (with $\lceil\cdot\rceil$ rounding up) keeps it at most a target miss rate $\alpha$ \emph{on average over the draw}. An average over the draw does not make each bet's payoff $M_t-b_t$ nonpositive in conditional mean (where $b_t$ is the allowance of query $t$, Section~\ref{sec:setup}); the monitor needs a bound that holds with high probability over the draw. One classical construction~\citep{vovk2012conditional} supplies it exactly in finite samples. Let the block contain $n_r$ labeled questions with true-answer scores $S_{r,1},\dots,S_{r,n_r}$, and let $\Hcal_r$ fix the block size $n_r$, an order $k_r\in\{1,\dots,n_r\}$ and an error level $\gamma_r\in(0,1)$, the probability the certificate is allowed to be wrong. Set
\begin{equation}
 q_r=S_{r,(k_r)},
 \qquad
 u_r=\operatorname{Beta}^{-1}_{1-\gamma_r}(n_r+1-k_r,\;k_r),
 \label{eq:certificate}
\end{equation}
where $S_{r,(k)}$ is the $k$-th smallest of the $n_r$ scores and $\operatorname{Beta}^{-1}_{p}(a,b)$ the $p$-quantile of the $\operatorname{Beta}(a,b)$ distribution. The state is deployed with threshold $q_r$ and allowance $b_t=u_r$ throughout its epoch. For instance $\alpha=.10$, $n_r=100$ (so $k_r=91$) and $\gamma_r=.0125$ give $u_r=.175$, the allowance of the model swap in Section~\ref{sec:evidence} (Appendix~\ref{app:liverefresh}). The interface below accepts any certificate that satisfies~\eqref{eq:budget}; this one is exact (Proposition~\ref{prop:beta}).

\subsection{Budget every attempt, and act before the outcome}
\label{sec:interface}

\paragraph{All attempts share one pathwise budget.}
Every certificate attempt, including one computed and then discarded, spends a share of one deterministic total calibration budget $\caldelta\in[0,1)$. Attempt $r$ receives an allocation $\rcal\ge0$, its share of $\caldelta$, fixed by $\Hcal_r$, and yields a \emph{certificate event} $G_r$ determined by its block: the typically unobservable event that its certified bound is correct (for the order-statistic certificate, $G_r=\{p_r\le u_r\}$ with $p_r=1-F_r(q_r)$ the miss probability the threshold leaves, and $\gamma_r\le\rcal$). With $G_r^{c}$ the complement of $G_r$ (the certificate is wrong), the budget condition is
\begin{equation}
 \Pbb(G_r^{c}\mid\Hcal_r)\le\rcal
 \quad\text{and}\quad
 \sum_{r\ge1}\rcal\le\caldelta
 \qquad\text{almost surely.}
 \label{eq:budget}
\end{equation}
Allocations may be random and history-dependent, attempts may be dependent, and their number may be random or countably infinite: a known cap $R$ allows $\rcal=\caldelta/R$ (the Bonferroni split), and $\rcal=6\caldelta/(\pi^2r^2)$ (a summable schedule, since $\sum_r r^{-2}=\pi^2/6$) allows unboundedly many. Charging even discarded attempts prevents an unbudgeted search over calibration outcomes.

\paragraph{The deployed state must come with a matching allowance fixed before the query.}
At query $t$ the active certificate event $G_{r_t}$, the allowance $b_t$ and the bet $\lambda_t$ (the stake the monitor places on query $t$; Section~\ref{sec:monitor}) are fixed before the query ($\Fpre$-measurable), and on that certificate, with $\one_{G}$ denoting the indicator of an event $G$,
\begin{equation}
 \one_{G_{r_t}}\,\E[M_t\mid\Fpre]\ \le\ \one_{G_{r_t}}\,b_t,
 \qquad 0<b_t<1 .
 \label{eq:onestep}
\end{equation}
This is the deployment null of Section~\ref{sec:setup}, the hypothesis the alarm tests. With~\eqref{eq:budget} it makes a high-probability statement \emph{over the calibration draw} about each deployed query's \emph{conditional} miss rate. The timing rule restricts \emph{when} a choice is made, not what is chosen: without it an operator must either freeze the configuration and forfeit the gains a refresh would buy, or react to $M_t$ and forfeit validity.

\subsection{Monitor with one never-reset wealth process}
\label{sec:monitor}

\paragraph{One wealth process runs across every epoch.}
Choose a bet $0\le\lambda_t\le1/b_t$ fixed before query $t$ and set
\begin{equation}
 E_0=1,
 \qquad
 E_t=E_{t-1}\bigl\{1+\lambda_t(M_t-b_t)\bigr\} .
 \label{eq:wealth}
\end{equation}
This is the gambler's wealth of Section~\ref{sec:intro} with stake $\lambda_t$: a miss multiplies it by $1+\lambda_t(1-b_t)$ and a covered query by $1-\lambda_tb_t$. Under the null~\eqref{eq:onestep}, on a correct certificate each factor has conditional mean at most one, and the cap $\lambda_t\le1/b_t$ keeps the wealth nonnegative. Any such bet is valid. Our simulations use the adaptive aGRAPA bettor of~\citet{waudby2024betting}, which sets each stake from past outcomes to approximately maximize the wealth's growth rate (Appendix~\ref{app:power}). \rev{For long deployments our default monitor, an equal mix of aGRAPA at five memory scales and a mixture over start times, keeps detecting after long benign periods (Appendices~\ref{app:delaybound} and~\ref{app:x-restart}).} An alarm is declared the first time $E_t\ge1/\edelta$, and the wealth is not reset at a refresh.

\paragraph{The whole protocol is one loop.}
Algorithm~\ref{alg:protocol} (Appendix~\ref{app:protocol}) assembles the pieces: its pre-query line closes $\Fpre$, its monitor line applies~\eqref{eq:wealth}, and only its post-outcome line may change the configuration.

\section{Theoretical guarantees}
\label{sec:theory}

Three results answer the questions an operator must settle before trusting the alarm. Is each certificate correct with the stated probability? Proposition~\ref{prop:beta} says yes, exactly, in finite samples. Does the alarm stay valid across every inspection and refresh? Theorem~\ref{thm:refresh} says yes, at total cost $\edelta+\caldelta$. Which moves keep that guarantee, and which break it? Corollary~\ref{cor:design} and the paragraph after it answer.

\subsection{Each certificate is exact in finite samples}
\label{sec:certificate}

\begin{proposition}[Finite-sample order-statistic certificate]
\label{prop:beta}
Conditional on $\Hcal_r$, let $S_{r,1},\dots,S_{r,n_r}$ be conditionally i.i.d.\ true-answer scores from $F_r$ on a fresh labeled block (scores include the deployment randomization), and let the integer $n_r\ge1$, the order $k_r\in\{1,\dots,n_r\}$ and $\gamma_r\in(0,\rcal]$ be $\Hcal_r$-measurable. Let $q_r$ and $u_r$ be as in~\eqref{eq:certificate}, and put $p_r=1-F_r(q_r)$. Then $\Pbb(p_r>u_r\mid\Hcal_r)\le\gamma_r$, with equality for continuous $F_r$; ties can make it conservative.
\end{proposition}

\paragraph{Why it matters: the allowance is a finite-sample bound, not an asymptotic one.}
The miss probability $p_r$ that the threshold leaves has an exact Beta law, so its $1-\gamma_r$ quantile $u_r$ bounds it with the stated confidence at every block size $n_r$, whatever the score distribution. Setting $G_r=\{p_r\le u_r\}$ and $b_t=u_r$ through the matching epoch then satisfies \eqref{eq:budget}; \eqref{eq:onestep} needs one further condition.

\begin{assumption}[Epoch score stability]
\label{ass:epoch}
Within the epoch served by attempt $r$, every deployed true-answer score
$s_{\theta_r}(X_t,Y_t)$ has conditional law $F_r$ given $\Fpre$, the law of
the calibration scores; fresh questions drawn i.i.d.\ from the calibration
law, independently of all the operator has observed, suffice.
\end{assumption}

\paragraph{The assumption is the only place a condition on the data enters.}
Assumption~\ref{ass:epoch} carries a certificate from its calibration block to the queries it certifies; Theorem~\ref{thm:refresh} below needs nothing about the data beyond \eqref{eq:budget}--\eqref{eq:onestep}. It is imposed per epoch, not per trajectory, so a refresh is free to change the score law; drift \emph{inside} an epoch breaks it, and the alarm targets drift that pushes misses above the allowance. The monitor tests misses against the \emph{certified} $u_r$ rather than the target $\alpha$, as \eqref{eq:onestep} requires (Appendix~\ref{app:beta_proof}).

\subsection{The composition theorem}
\label{sec:theorem}

\begin{theorem}[Certified adaptive-refresh composition]
\label{thm:refresh}
Assume the information structure~\eqref{eq:filtration}, the budget~\eqref{eq:budget} and the null~\eqref{eq:onestep} for the state and action in force, with $E_t$ as in~\eqref{eq:wealth}; only finitely many calibrations precede any query, and each certificate is completed before its state is deployed. Then, with $\Omega_0=\Omega$ and $\Omega_t=\bigcap_{s\le t}G_{r_s}$ the event that every certificate in force at queries $1,\dots,t$ is correct, the killed wealth $\widetilde E_t=E_t\one_{\Omega_t}$ is a nonnegative supermartingale with respect to $(\F_t)_{t\ge0}$, and for every $\edelta\in(0,1)$,
\[
 \Pbb\!\left(\sup_{t\ge1}E_t\ge1/\edelta\right)\ \le\ \edelta+\caldelta .
\]
\end{theorem}

\paragraph{Why it holds: every choice is fixed before the query it affects.}
Everything the operator may vary (which state, when, how often, on what evidence) enters only through $\Fpre$, and $\Fpre$ closes before the query it is judged on. What makes that work is the \emph{killed} wealth $\widetilde E_t$: the same bets, voided (set to zero) the moment any certificate on the path has failed. When a certificate is wrong, \eqref{eq:onestep} says nothing about the miss rate, and killing discards exactly those paths, so two standard arguments meet at the certificate interface. The killed wealth is a nonnegative supermartingale (its conditional mean never rises), so Ville's inequality (such a process started at $1$ reaches $1/\edelta$ with probability at most $\edelta$) contributes the first term; a conditional union bound over attempts, which assumes neither independence nor a fixed attempt count, contributes the second. On the event that no certificate ever fails the two wealths coincide, so $\caldelta$ pays for the event on which they differ (Appendix~\ref{app:refresh_proof}). \hfill$\square$

\paragraph{Why it matters: one auditable budget covers the whole deployment.}
Because~\eqref{eq:headline} holds at every $t$ at once, an operator may inspect the wealth at \emph{any} stopping time: no fixed horizon, no pre-committed stopping rule, no penalty for repeated peeking. For a target total $\delta$, alarm at $1/\edelta$ with $\edelta+\caldelta\le\delta$; $1/\delta$ would overspend whenever $\caldelta>0$. Later attempts may get smaller shares of $\caldelta$, hence looser allowances unless their blocks grow, never a weaker guarantee (Appendix~\ref{app:familywise}).

\subsection{Corollaries, and the failure modes they exclude}
\label{sec:corollaries}

\paragraph{The same interface delimits a usable design space.}
Corollary~\ref{cor:design} lists six operational moves that keep the guarantee, each costing nothing or one allocation from $\caldelta$ or $\edelta$: parts (a)--(c) say what may be done to the deployed state, (d) what is learned beyond the alarm, (e) how late labels are handled, and (f) how a controller may act. Recalibration or a realized threshold change opens a new epoch unless the active certificate covers it uniformly (Appendix~\ref{app:controller-scope}); Appendix~\ref{app:corollaries} proves each part.

\begin{corollary}[The certified design space]
\label{cor:design}
Under the assumptions of Theorem~\ref{thm:refresh}, with part (e) relaxing immediate feedback, the following hold.
\begin{enumerate}[label=\textup{(\alph*)},leftmargin=2.0em,itemsep=0pt,topsep=1pt]
\item \emph{Frozen reuse.} One state and certificate $G_1$ with deterministic $\delta_1^{\mathrm{cal}}\le\caldelta$ give $\Pbb(\sup_tE_t\ge1/\edelta)\le\edelta+\delta_1^{\mathrm{cal}}$: the calibration term never grows, however many queries or inspections follow.
\item \emph{Post-selection construction.} $\theta_r$ may come from any $\Hcal_r$-measurable randomized, history-dependent training or selection procedure; fresh calibration then needs no multiple-testing correction for the states considered before selection.
\item \emph{Certified library.} To choose \emph{after} seeing the block, predeclare a finite $\Hcal_r$-measurable library of candidate states indexed by $j$, where member $j$ has certificate event $G_{r,j}$ and allocation $\delta_{r,j}$, with allocations totalling at most $\caldelta$; then $G_r=\bigcap_jG_{r,j}$ certifies the attempt and any member may be deployed.
\item \emph{Cumulative envelope.} With $A_t=\sum_{s\le t}(M_s-b_s)$, the cumulative excess of misses over allowances, and a fixed finite mixture of Hoeffding bounds (Appendix~\ref{app:envelope}), there is a boundary $h_t(\edelta)$ with $\Pbb(\exists t:A_t\ge h_t(\edelta))\le\edelta+\caldelta$, and on the complement of that event every almost surely finite stopping time $\tau$ (a query index chosen using only information available up to that query) obeys $\tau^{-1}\sum_{t\le\tau}M_t\le\tau^{-1}\sum_{t\le\tau}b_t+h_\tau(\edelta)/\tau$: where the alarm reports \emph{whether} misses exceed the allowance, the envelope bounds how far they can run above it by chance.
\item \emph{Exogenous delayed feedback.} For one certified state whose misses are, given its calibration block, independent with mean at most $b$, and a delay schedule conditionally independent of them, running the wealth in reveal time with the bet fixed before each batch preserves the guarantee.
\item \emph{Certificate-preserving controller.} A controller may use all of $\F_t$ provided each decision either keeps the state and picks an action the certificate covers uniformly, or completes select--calibrate--certify--deploy before the next affected query. Adaptivity inside this class is unrestricted.
\end{enumerate}
\end{corollary}

\paragraph{Five moves break the guarantee.}
\emph{Same-block selection} chooses among candidate states with the calibration block that then certifies the winner (unless the library was certified simultaneously, part (c)). \emph{Calibrate-then-replace} deploys a new model behind a threshold calibrated for the outgoing one, which would need its own score-transfer bound (Appendix~\ref{app:training}). \emph{Same-query peeking} uses $X_t$ or $M_t$ to choose that query's state, threshold, action, allowance or bet. \emph{Unbudgeted search} tries calibration attempts without charging them to $\caldelta$. \emph{Resetting wealth} without paying for it also breaks validity, though a restart that allocates $\edelta$ across epochs is safe (Appendix~\ref{app:frozen}).

\rev{\subsection{What composition gains, and how it handles uncertified changes}
\label{sec:gain}
One never-reset process dominates every \emph{budgeted restart}, the valid alternative that starts a fresh wealth process at each refresh with a share of $\edelta$: the mixture of the restart's processes, weighted by their shares and never reset, alarms no later on every path (Proposition~\ref{prop:dominate}), whereas a restarted process cannot alarm unless its own segment holds enough misses, whatever its bettor (Proposition~\ref{prop:ceiling}); and a mixture over start times detects an excess beginning at an unknown time $\nu$ with a mean-delay bound that grows only logarithmically in $\nu$ (Proposition~\ref{prop:delaybound}). If a node changes without a new certificate, \eqref{eq:budget} still holds and the theorem needs only~\eqref{eq:onestep}: while the conditional miss rate stays within the allowance the bound holds (between target and allowance the monitor is silent by design), and above it any alarm is a true detection (Remark~\ref{rem:three-cases}). Two runtime checks cost the alarm nothing, since any pre-query event may trigger a certified refresh (Corollary~\ref{cor:design}(f)): \emph{state attestation}, comparing node digests with the certified ones (Remark~\ref{rem:attest}), and \emph{label-free flags} on unlabeled indicators, with their own false-trigger budget (Remark~\ref{rem:flag}).}

\section{\rev{Experiments}}
\label{sec:evidence}

\paragraph{The simulations show that dropping any one condition breaks validity and that the monitor detects harmful change within its false-alarm budget.}
Synthetic studies break one condition of the protocol at a time and check every guarantee at its stated budget; replays of cached Qwen2.5 scores (computed once by the real model, then fed through the protocol) and matched baselines measure detection on simulated deployment streams (\emph{paths}). Protocols were fixed in advance (Appendices~\ref{app:experiments}--\ref{app:repro}). \rev{In Appendices~\ref{app:x-scoreboard}--\ref{app:x-refresh} none of $13{,}536$ measured error rates exceeds its bound at the $1\%$ level.}

\begin{figure}[t]
\begin{center}
\includegraphics[width=\linewidth]{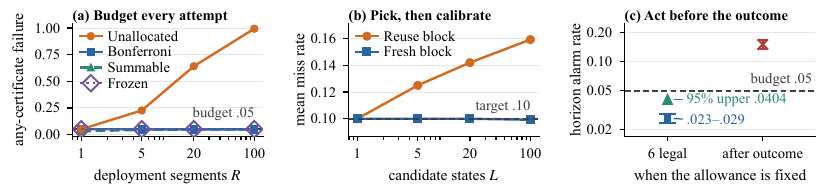}\\[0.2ex]
\includegraphics[width=\linewidth]{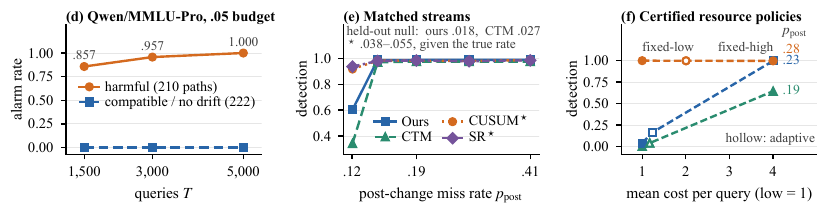}
\end{center}
\caption{\textbf{(a)--(c)} Synthetic: each panel breaks one condition and
prices it. (a) Exact probability that any certificate fails over $R$ segments
of $100$ queries (failure alone; Table~\ref{tab:validity}'s composed rows count
alarm \emph{or} failure\rev{; failure is not an alarm, Appendix~\ref{app:composed}}): $.05$ per certificate compounds toward certain
failure, while the Bonferroni and summable schedules of
Section~\ref{sec:interface} hold. (b) Reusing the block that selected the best
of $L$ candidate states inflates its miss rate above the $.10$ target. (c) A monitor whose
allowance is chosen after the outcome alarms at three times its budget; six legal
schedules fixed before each query hold (triangle: their largest simultaneous $95\%$ upper
bound). \textbf{(d)--(f)} Monitoring at budget: (d) every
audited-harmful Qwen path alarms by horizon $T=5{,}000$ queries, and no compatible or no-drift path
does; (e) without the true post-change rate, ours detects as often as detectors given
it, except near the allowance $.10$ (Table~\ref{tab:baselines_main}); (f) one dashed
line per post-change miss rate, from fixed-low through the adaptive policy
(middle point) to fixed-high bandwidth.}
\label{fig:evidence}
\end{figure}

\paragraph{Each condition is load-bearing, and every budget is respected.}
Breaking any one condition inflates the error far above its budget or target (Figure~\ref{fig:evidence}(a)--(c)): for example, re-certifying each of $100$ refreshes at the full $.05$ leads to an alarm or a failed certificate with probability $.9918$, and a monitor whose allowance is set after the outcome alarms at $.150$ against a $.05$ budget. Every valid design in Table~\ref{tab:validity} stays below its budget (at most $.0401$ against $.05$, and $.0490$ against $.10$ for one never-reset wealth process across $100$ refreshed epochs); allocation, not independence, restores validity (Lemma~\ref{lem:familywise}).

\begin{table}[t]
\caption{Null-behavior claims beside their budgets; italic rows are
deliberately invalid comparators. Composed \rev{and real-score} rows count alarm or certificate
failure (\rev{composed:} Monte Carlo, $5{,}000$ paths; the monitor's own null rate stays
$\le.0262$\rev{, and $\le.0024$ with every certificate correct; real scores: $4{,}000$ to $20{,}000$ paths}); \rev{the flag row counts firings;} the rest use a known allowance. Sentinels are controllers that
bet only on high-bandwidth queries scheduled in advance.}
\label{tab:validity}
\begin{center}
\small
\setlength{\tabcolsep}{5pt}
\begin{tabular}{@{}lcc@{}}
\toprule
Study & Budget & Measured \\
\midrule
Alarm, miss rate at / below allowance & $.05$ & $.0210/.0025$ \\
Envelope breach, $4{,}000$ null paths & $.05$ & $.0015$ \\
Composed, $100$ refreshes (Bonferroni) & $.10$ & $.0490$ \\
\emph{\quad same design, unallocated} & --- & \emph{$.9918$} \\
Six pre-query allowance schedules & $.05$ & $.023$--$.029$ \\
\emph{\quad illegal: allowance set after $M_t$} & --- & \emph{$.150$} \\
Seven exogenous label-delay laws & $.05$ & $.019$--$.022$ \\
Twelve sentinels $\times$ two nulls & $.05$ & $\le.0401$ \\
\rev{Real scores, refreshes along each path, $5$--$131$ attempts} & \rev{$.05$} & \rev{$\le.0455$} \\
\rev{Real scores, blocks from labeled traffic} & \rev{$.05$} & \rev{$\le.0288$} \\
\rev{Two label-free flags, nothing changed} & \rev{$.05$} & \rev{$.0024/.0020$} \\
\bottomrule
\end{tabular}
\end{center}
\end{table}

\paragraph{On Qwen2.5/MMLU-Pro \rev{replays, every audited-harmful shift alarms and no benign one does} at a $5\%$ budget (Figure~\ref{fig:evidence}(d)).}
The open language model Qwen2.5-1.5B-Instruct~\citep{qwen2024qwen25} scored the ten options of questions from the multiple-choice benchmark MMLU-Pro~\citep{wang2024mmlupro} over eight nodes at $\caldelta=\edelta=.025$, each allowance being an independent audit's exact Clopper--Pearson bound. Each path switched after $500$ queries to one of $14$ subjects; a separate audit rated each split--subject cell harmful, compatible with the allowance or (three cells) ambiguous. By $T=5{,}000$, all $210$ paths with an audited-harmful shift alarmed, and none of the $192$ compatible, $30$ no-drift and $18$ ambiguous-cell paths did. A discovery-disjoint confirmation, on questions these paths never used, upheld the strongest harmful shift. In a separate study at a $.10$ budget, removing any one retrieval node raised no alarm on $240$ paths yet left the harmful Computer Science and History shifts detected (Appendices~\ref{app:qwen} and~\ref{app:modular}). \rev{These paths share fixed states and pools; with fresh questions and blocks per replication, detection rises smoothly with the true excess over the allowance (Appendices~\ref{app:qwen} and~\ref{app:x-first}).}

\paragraph{Calibrating a new model on fresh data keeps its miss rate in check, and the alarm stays within budget across the swap.}
When the deployed $1.5$B model is replaced by the $0.5$B model on a two-node system, calibrating the incoming model on a fresh block holds the miss rate at \rev{$.099$ on average over redrawn blocks}, below the $.10$ target, whereas calibrating it on the outgoing model's scores lets it rise to \rev{$.164$}. With one wealth process running across the swap, \rev{with blocks redrawn and every miss rate exact, the alarm with both certificates correct stays at or below $.0140$ (budget $.025$)}; under drift, \rev{the default monitor detects $.88$--$.89$ of paths (aGRAPA $.61$--$.66$)} after the swap, so the refresh did not hide the shift (\rev{Appendices}~\ref{app:liverefresh} \rev{and~\ref{app:x-first}}).

\begin{table}[!tb]
\caption{Matched baselines on paired synthetic streams whose pre-change miss
rate equals the allowance $.10$; $^{\star}$ marks CUSUM and Shiryaev--Roberts
(SR), \emph{given} the true post-change rate $p_{\mathrm{post}}$ and a
null-calibrated threshold; $T$ is the horizon in queries.
Detection ($\uparrow$) is by $T=5{,}000$; $\downarrow$ columns are alarm rates
against a $.05$ budget: a held-out null, $T=20{,}000$ under the frozen
$T=5{,}000$ threshold, and three miss-preserving score transforms, which CTM
is designed to detect (Appendices~\ref{app:baselines}, \ref{app:modular}).}
\label{tab:baselines_main}
\begin{center}
\scriptsize
\setlength{\tabcolsep}{4.5pt}
\begin{tabular}{@{}lcccccc@{}}
\toprule
 & \shortstack{Told the\\alternative} & \shortstack{Held-out\\null $\downarrow$}
 & \shortstack{Detection $\uparrow$\\$p_{\mathrm{post}}{=}.123$} & \shortstack{Detection $\uparrow$\\$p_{\mathrm{post}}{\ge}.146$}
 & \shortstack{Null, $T{=}20{,}000$ $\downarrow$\\frozen threshold} & \shortstack{Miss-preserving\\shifts $\downarrow$} \\
\midrule
Ours \rev{(aGRAPA)} & no & $.0175$ & $.603$ & $.987$--$.988$ & $.036$ & $.003/.003/.003$ \\
CTM & no & $.0270$ & $.343$ & $.970$--$.975$ & $.015$ & $.142/.964/.969$ \\
CUSUM\,/\,SR$^{\star}$ & yes & $.038$--$.055$ & $.917$--$.936$ & $.978$--$.987$ & $.193$--$.274$ & --- \\
\bottomrule
\end{tabular}
\end{center}
\end{table}

\paragraph{Without being told the post-change rate, the monitor detects as often as detectors that are told it, except near the allowance (Figure~\ref{fig:evidence}(e)).}
Table~\ref{tab:baselines_main} compares our monitor with CUSUM and Shiryaev--Roberts~\citep{page1954cusum,roberts1966control}, classical change-point detectors that are here handed the true $p_{\mathrm{post}}$, and with the conformal test martingale (CTM) of~\citet{vovk2021retrain}, which tests exchangeability. Our monitor has the lowest held-out null rate (each against its own null) and detects the stronger alternatives as often as the oracle detectors; knowing the alternative raises detection near the allowance and buys speed elsewhere (Appendix~\ref{app:baselines})\rev{: there they detect $.917$--$.936$ and our default monitor $.758$, level with the same-null windowed rate of \citet{timans2025monitoring} ($.764$; Appendix~\ref{app:x-baselines})}.

\paragraph{The alarm tracks misses rather than scores, and its threshold $1/\edelta$ needs no recalibration.}
Under score shifts that leave every miss unchanged, our alarm rate stays at $.003$ while CTM's reaches $.969$. Our Ville threshold stays within budget at $T=20{,}000$, whereas CUSUM and Shiryaev--Roberts thresholds calibrated for $T=5{,}000$ exceed the budget there. Online conformal prediction, by contrast, re-adjusts its score threshold through a harmful shift, so its own output does not reveal the shift (Appendix~\ref{app:oc}).

\paragraph{\rev{No tuning is needed}, and monitoring is cheap.}
With no tuning, \rev{aGRAPA's} detection rises from $.310$ to $.996$ as misses exceed the allowance by $.02$ to $.04$. Raising bandwidth once evidence builds detects more than fixed-low bandwidth at lower cost than fixed-high (Figure~\ref{fig:evidence}(f)), and one update costs $24.0\,\mu$s (Appendices~\ref{app:modular} and~\ref{app:earlier}).

\rev{\paragraph{Never resetting detects what valid restarts miss.}
With a refresh every $100$ queries and an excess of $.04$, a summable or capped restart at each refresh caps detection at $.122$ or $.524$ when the allowance is known (Corollary~\ref{cor:ceiling}); with certified allowances the restarts detect $.006$ and $.045$, the never-reset monitor $.999$. After $50{,}000$ benign queries (a refresh every $1{,}000$) it detects $.998$ against $.270$ and $.906$, and on QASC $.400$ of harmful paths against $.031$ for the summable restart (Appendix~\ref{app:x-restart}).

\paragraph{On a second benchmark retrieval helps, and every budget holds.}
On QASC~\citep{khot2020qasc} split by topic across eight nodes, a hub rule that lets nodes without a relevant passage stay silent gains $6$ to $10$ points of top-1 accuracy over no retrieval, within two points of pooled retrieval (unavailable to a federation), and the 7B student keeps $2.48$ of $8$ options at exact $90\%$ coverage (Appendix~\ref{app:x-qasc}). Refreshing among $36$ states by evidence, cost, an exogenous clock or the path's own record keeps both budgets over $5$ to $131$ attempts (Appendix~\ref{app:x-refresh}). After an uncertified node change the monitor alarms on $.0004$ of benign changes and detects $.566$ of harmful ones; state attestation cuts the post-change queries above the allowance from $.298$ to $.004$ (Appendix~\ref{app:x-node}).

\paragraph{Refreshes cost little.}
With blocks of $1{,}000$ the allowance grows from $.121$ to $.150$ over $100$ attempts while the 7B student keeps $2.49$ options; blocks taken from labeled traffic need $1{,}000$ extra labels a path over $20$ refreshes, not $21{,}000$, and keep both budgets with one query in ten labeled (Appendix~\ref{app:x-cost}).}

\section{Conclusion}
\label{sec:discussion}

A deployment may change its model, retriever, corpus, score or threshold as often as it likes and still report one trajectory-level false-alarm guarantee, provided it follows select--calibrate--certify--deploy, budgets every certificate attempt and acts before each outcome. On Qwen2.5/MMLU-Pro the monitor alarmed on every audited-harmful shift and no benign one\rev{, and on QASC every budget held}. The guarantee is deliberately narrow (Section~\ref{sec:setup}); Appendix~\ref{app:repro} lists what remains open\rev{, a live federation included}.

\subsubsection*{Reproducibility statement}
The operating model and all assumptions are stated in Section~\ref{sec:setup}, Assumption~\ref{ass:epoch} and Appendix~\ref{app:protocol}, with the covered and excluded regimes in Table~\ref{tab:scope}. Complete proofs of Proposition~\ref{prop:beta}, Theorem~\ref{thm:refresh} and Corollary~\ref{cor:design} are in Appendices~\ref{app:proofs}--\ref{app:corollaries}. Every experimental design (models, dataset revisions, eligibility rules, split and seed construction, audit sizes, budget allocations and horizons) is specified in Appendix~\ref{app:experiments}, and Appendix~\ref{app:repro} records the compute environment, the cache-integrity checks, and how every number reported in Section~\ref{sec:evidence} traces to a completed run. Figures~\ref{fig:evidence} and~\ref{fig:earlier} are rendered by scripts that read the retained numerical summaries directly, so each plotted value is traceable to a completed run. The artifact bundle contains the code, one configuration per experiment (fixing its seeds and any pinned model and dataset revisions), the retained result summaries, the figure scripts, a verifier and a map from every figure and table to its result files. \rev{It also holds the model-score caches behind the Qwen2.5 runs, the score tables of every new experiment, the per-path results of the model swap, the registration record of the new experiments, and scripts that replay results from these files alone on one CPU core (Appendix~\ref{app:repro}).}

\subsubsection*{Use of generative AI and AI-assisted technologies}
We used generative AI tools (ChatGPT, Google Gemini, and Claude) for language
editing and code formatting support. All data, results, and mathematical
derivations are the authors' own work.

\clearpage
\appendix
\let\revOldSection\section
\renewcommand{\section}{\FloatBarrier\revOldSection}
\let\revOldSubsection\subsection
\renewcommand{\subsection}{\FloatBarrier\revOldSubsection}
\section{Operating model and protocol}
\label{app:protocol}

This appendix states the statistical and operational contract in full, so that
a reader can check applicability without reconstructing it from the proofs.
Section~\ref{sec:setup} gives the same model in compressed form.

\subsection{One round of bandwidth-constrained federated RAG}
\label{app:round}

There are $K<\infty$ retrieval nodes and one hub in a synchronous star. The
nodes share one deployed student model but retain separate private corpora and
retrievers. For a question $X_t$ and a finite, agreed-upon answer list $\Y$,
every active node scores every $y\in\Y$ using its locally retrieved context; in
the Qwen study of Appendix~\ref{app:qwen}, that local score is the full-option
mean token negative log-likelihood. Each node quantizes its scalar scores and
uploads them; the hub aggregates them, compares them with the active threshold,
and returns a set $C_t(X_t)$. In symbols, one round is
\[
 \{\text{private corpus } i\}_{i=1}^{K}
 \xrightarrow[\text{shared student}]{\text{local retrieval}}
 \{\text{quantized scores}\}_{i=1}^{K}
 \xrightarrow{\text{hub aggregation}} C_t(X_t),
\]
and once the label is adjudicated the monitored miss is
$M_t=\one\{Y_t\notin C_t(X_t)\}$. Raw documents, training examples, gradients
and model weights are not part of this uplink. The analysis is of the
score-message protocol; it does not supply a privacy guarantee merely because
documents remain local.

The analysis uses a finite-list, full-scoring convention. If a deployment
instead generates a query-dependent candidate list, that generation rule must
be placed inside the score state and the failure to place $Y_t$ on the list
must be included in the miss event. Open-ended generation is not covered by
the present result.

Table~\ref{tab:glossary} collects the notation and terms of the main text.

\begin{table}[ht]
\caption{Notation and terms used in the main text; Sections~\ref{sec:intro}--\ref{sec:theory} define each where it first appears.}
\label{tab:glossary}
\begin{center}
\footnotesize
\begin{tabular}{@{}P{0.25\linewidth}P{0.70\linewidth}@{}}
\toprule
Symbol or term & Meaning \\
\midrule
$X_t,Y_t$; $C_t(X_t)$; $M_t$ & Question and true answer at query $t$; the
returned answer set; the miss $\one\{Y_t\notin C_t(X_t)\}$ \\
\addlinespace[0.3ex]
Score state $\theta$; epoch & Everything determining the deployed score (model, retriever and corpus,
candidate rule, score semantics, hub aggregation, threshold rule, and
any certified within-epoch communication-action class), and the maximal run of
queries sharing one such state, its realized threshold $q_r$
and its matching certificate \\
\addlinespace[0.3ex]
Action $a_t$; covered uniformly & The pre-query communication choice, e.g.\
per-node bit budgets $B_i$; a certificate covers a class of actions uniformly
when its allowance holds whichever member is used \\
\addlinespace[0.3ex]
Attempt $r$; history $\Hcal_r$; certificate event $G_r$; allocation $\rcal$ &
A state and allocation chosen from the pre-calibration history $\Hcal_r$, then
a fresh calibration block yielding the (typically unobservable) event $G_r$ on
which the epoch's risk bound holds; an attempt is charged even if discarded \\
\addlinespace[0.3ex]
$\caldelta$; $\edelta$ & Total calibration budget, bounding the chance that
any certificate is wrong; evidence budget, bounding false alarms when every
certificate is correct \\
\addlinespace[0.3ex]
Allowance $b_t$; $u_r$; $\alpha$ & A pre-query upper bound on
$\E[M_t\mid\Fpre]$ holding on the active certificate; the Beta-quantile
allowance of Proposition~\ref{prop:beta}; the target miss rate that sets the
order $k_r$ \\
\addlinespace[0.3ex]
$\F_{t-1}\subseteq\Fpre\subseteq\F_t$ & Information after query $t-1$; that
plus everything fixed before query $t$; that plus query $t$'s pair, set and miss \\
\addlinespace[0.3ex]
Wealth $E_t$; bet $\lambda_t$; sticky & The evidence
$E_t=E_{t-1}\{1+\lambda_t(M_t-b_t)\}$ against the null, with the predictable
stake $\lambda_t\in[0,1/b_t]$; sticky wealth is never reset \\
\addlinespace[0.3ex]
Calibrate-then-replace; same-block selection & Deploying a new model behind a
threshold calibrated for the outgoing one; selecting a state with the
calibration block that then certifies it. Both break the guarantee \\
\bottomrule
\end{tabular}
\end{center}
\end{table}

\subsection{Information structure and the certified transition}
\label{app:information}

All calibration, deployment, protocol, controller and training randomness is
defined on the common probability space $(\Omega,\mathcal A,\Pbb)$ of
Section~\ref{sec:setup}. Every $\Hcal_r$, $\Kcal_r$, $\Fpre$ and $\F_t$ below
is a sub-$\sigma$-field of $\mathcal A$.

Let $\F_0$ contain the pre-deployment history. For deployment query $t\ge1$,
\[
 \F_{t-1}\subseteq\Fpre\subseteq\F_t .
\]
The first field contains all information after query $t-1$. The pre-query
field additionally contains every intervening state selection, completed
calibration block, certificate, threshold, communication action,
allowance and bet, but is formed \emph{before} the fresh pair $(X_t,Y_t)$. The
last field additionally contains that pair, the returned set and $M_t$. The
certified risk is therefore conditional on the pre-query history but marginal
over a fresh question. Evaluating a pre-committed retrieval, scoring and set
rule at $X_t$ is allowed; using $X_t$ to select the state, threshold,
bandwidth, allowance or bet would require a stronger certificate conditional
on $X_t$.

Fix a deterministic total budget $\caldelta\in[0,1)$, and before calibration
attempt $r$ let $\Hcal_r$ be the available history. The decision to make attempt
$r$, the proposed state $\theta_r$, the calibration size, the threshold rule
and the allocation $\rcal\ge0$ are $\Hcal_r$-measurable. A fresh block $D_r$
is then drawn and $\Kcal_r=\Hcal_r\vee\sigma(D_r,\text{calibration
randomization})$. Because each certificate is completed before its state is
deployed, $\Kcal_r\subseteq\Fpre$ for every query $t$ the attempt serves. The
canonical state-changing operation is:
\begin{enumerate}[leftmargin=1.6em,itemsep=1pt,topsep=2pt]
\item use $\Hcal_r$ to select $\theta_r$ and $\rcal$;
\item acquire a fresh, conditionally honest calibration block that was not
      used to train or select $\theta_r$;
\item construct the state-matched threshold, allowance and event $G_r$; and
\item deploy only after these objects belong to the pre-query information.
\end{enumerate}
Calibrate-then-replace, deploying a new model behind a threshold calibrated
for the outgoing one, does not certify the replacement: selection must precede
the fresh, state-matched calibration block.

The assumptions used by Theorem~\ref{thm:refresh} can now be stated compactly:
the budget condition and the deployment null of
Section~\ref{sec:interface}, repeated here under their original numbers, hold
almost surely:
\begin{align}
 \Pbb(G_r^c\mid\Hcal_r)&\le\rcal,
 &\sum_{r\ge1}\rcal&\le\caldelta,
 \tag{\ref*{eq:budget}}\label{eq:app_budget}\\
 \one_{G_{r_t}}\E[M_t\mid\Fpre]
 &\le \one_{G_{r_t}}b_t,
 &0<b_t&<1 .
 \tag{\ref*{eq:onestep}}\label{eq:app_one_step}
\end{align}
With the information structure and timing above, these two are what
Theorem~\ref{thm:refresh} needs; a third condition,
Assumption~\ref{ass:epoch} of Section~\ref{sec:certificate}, is needed only to
carry a \emph{particular} certificate across the epoch it was drawn for, and is
where any distributional condition on the data enters.
Allocations may be random and history-dependent, but each is chosen before its
block. Attempts need not be independent, and their realized count may be
random or countably infinite, provided only finitely many calibration
operations occur before any one query. To pad a random collection of attempts
to a countable sequence, set $G_r=\Omega$ and $\rcal=0$ for every unattempted
slot; the same convention ($G_r=\Omega$, $\rcal=0$) covers a known,
deterministic allowance.

Table~\ref{tab:scope} sets the covered regime beside the adjacent problems
that this paper does not settle; Section~\ref{sec:setup} states the same split
in prose.

\begin{table}[ht]
\caption{Scope of the formal guarantee. The right column lists what the
theorem does not establish: adjacent problems it leaves open and, in the middle
row, the moves of Section~\ref{sec:corollaries} that break it.}
\label{tab:scope}
\begin{center}
\footnotesize
\begin{tabular}{@{}P{0.45\linewidth}P{0.50\linewidth}@{}}
\toprule
Covered & Not established \\
\midrule
Finite fixed $K$; synchronous star; shared student copies; private corpora;
uplink scalar-score messages; finite answer list; marginal coverage for a fresh
question & Asynchronous, peer-to-peer, Byzantine or privacy-preserving
federation; open-ended generation; coverage conditional on every $X=x$ \\
\addlinespace[0.3ex]
History-selected state then fresh state-matched calibration; pre-outcome
covered actions; sticky evidence & Same-block selection;
calibrate-then-replace; same-query peeking; unbudgeted search; unpaid monitor
restarts \\
\addlinespace[0.3ex]
Evidence for excess conditional miss rate; fixed-state exogenous-delay
corollary & General score or covariate shift detection; guaranteed recovery
after an alarm; optimal control; outcome-dependent missingness \\
\bottomrule
\end{tabular}
\end{center}
\end{table}

\begin{algorithm}[ht]
\caption{Anytime-FC-RAG: one certified deployment trajectory}
\label{alg:protocol}
\begin{algorithmic}[1]
\Require budgets $\edelta+\caldelta\le\delta$, an initial state $\theta_1$ and its allocation $\delta_1^{\mathrm{cal}}$
\LongState{draw a fresh, conditionally honest block for $\theta_1$, compute $q_1$ and its allowance, and set $E_0\gets1$, $r\gets1$ \InComment{Eq.~\eqref{eq:certificate}}}
\For{$t=1,2,\dots$}
  \LongState{\textbf{pre-query:} fix the active attempt $r_t$, threshold $q_{r_t}$, action $a_t$, allowance $b_t$ and bet $\lambda_t\in[0,1/b_t]$ \InComment{closes $\Fpre$; Eq.~\eqref{eq:filtration}}}
  \LongState{every node retrieves locally, scores $\Y$ with the deployed student and uploads quantized scores; the hub returns $C_t(X_t)=\{y:s_t(y)\le q_{r_t}\}$ \InComment{Fig.~\ref{fig:protocol}}}
  \State observe $Y_t$; $M_t\gets\one\{Y_t\notin C_t(X_t)\}$; $E_t\gets E_{t-1}\{1+\lambda_t(M_t-b_t)\}$ \Comment{Eq.~\eqref{eq:wealth}}
  \If{$E_t\ge1/\edelta$} \LongState{raise the alarm; \textbf{do not} reset the wealth unless a share of $\edelta$ pays for the restart (App.~\ref{app:frozen}) \InComment{Thm.~\ref{thm:refresh}}} \EndIf
  \LongState{\textbf{post-outcome:} using all of $\F_t$, either keep $\theta_{r_t}$ and pick an action the active certificate covers uniformly, or select $\theta_{r+1}$ and allocation $\delta^{\mathrm{cal}}_{r+1}$ from $\Hcal_{r+1}$, calibrate it freshly, certify it, deploy it unless discarded, and set $r\gets r+1$ (the allocation is spent either way) \InComment{Cor.~\ref{cor:design}(f)}}
\EndFor
\end{algorithmic}
\end{algorithm}

Algorithm~\ref{alg:protocol} states the loop the
theorem assumes, and the correspondence is line by line: its pre-query step is
the closure of $\Fpre$ above, its monitor step is the wealth
recursion~\eqref{eq:wealth}, and its two post-outcome branches are exactly the
two cases of Appendix~\ref{app:controller-scope}. The four-step
state-changing operation listed above is what its second branch performs.

\subsection{What a controller may change}
\label{app:controller-scope}

After query $t$, a controller may use $\F_t$, including the wealth, the misses, realized costs
and its own policy randomization, to choose later actions. The controller remains covered if it either (i) retains the state and
chooses an action from a communication-action class that the current
certificate covers uniformly, or (ii) completes a new
select--calibrate--certify--deploy transition before the next affected query.
Outside case (i), a change of model, retriever, corpus, candidate rule, score,
aggregation rule or realized threshold starts a new epoch, and so do
recalibration and certificate replacement unless the original event covered
the realized choice uniformly. A certified bound $b_t\ge1$ is
uninformative and inadmissible for the betting theorem, and cannot be clipped
below one while preserving validity.

Concretely, the controller has three main levers, each covered when it is
exercised through case (i) or (ii) before the next query: \emph{recalibration}, which
requests fresh compressed calibration summaries from each node and reconstructs
the threshold at the hub; \emph{bandwidth escalation}, which raises selected
per-node budgets $B_i$ or the calibration-refresh budget; and \emph{student
refresh}, which replaces the deployed model. Student refresh always opens a
new certified epoch, and recalibration does too unless the active certificate
covers the realized threshold uniformly. Bandwidth escalation may stay inside the current epoch when the
active certificate covers the new action uniformly, and otherwise opens one too.

The guarantee therefore survives every certificate-preserving controller, and
the wealth remains valid evidence for its decisions. It does not prove that an alarm-triggered
intervention restores coverage, minimizes communication, or solves an optimal
stochastic-control problem.

\section{Proofs}
\label{app:proofs}

\subsection{Trajectory-level certificate accounting}
\label{app:familywise}

\begin{lemma}[Conditional familywise control]
\label{lem:familywise}
If~\eqref{eq:app_budget} holds, then
$\Pbb\bigl(\bigcup_{r\ge1}G_r^c\bigr)\le\caldelta$.
\end{lemma}

\begin{proof}
For a finite prefix, the union bound and the tower property give
\[
 \Pbb\!\left(\bigcup_{r=1}^{R}G_r^c\right)
 \le \E\!\left[\sum_{r=1}^{R}\one_{G_r^c}\right]
 = \E\!\left[\sum_{r=1}^{R}\Pbb(G_r^c\mid\Hcal_r)\right]
 \le \E\!\left[\sum_{r=1}^{R}\rcal\right]
 \le\caldelta .
\]
Letting $R\to\infty$ and using continuity from below (equivalently, Tonelli's
theorem on the nonnegative indicators) proves the claim. The argument charges
discarded attempts, and it assumes neither a fixed number of attempts nor
independence among them.
\end{proof}

In words: the chance that
\emph{any} certificate the trajectory ever issues is wrong is at most the total
$\caldelta$ fixed in advance, no matter how many attempts are made,
in what order, or how strongly they depend on one another. It is what allows
Theorem~\ref{thm:refresh} to treat an unbounded, adaptively chosen sequence of
refreshes as if it were a single certificate.

When $0<\caldelta<1$, the deterministic schedule
$\rcal=6\caldelta/(\pi^2r^2)$ allows arbitrarily many attempts. For a standard
root-$n_r$ conditional certificate, the allocation-dependent term of the
allowance scales under this schedule as
$O\bigl(\sqrt{(\log(1/\caldelta)+2\log r)/n_r}\bigr)$, so
$n_r\asymp\log r$ with a sufficiently large constant keeps that term
bounded at a fixed level. A known cap $R$ instead permits
$\rcal=\caldelta/R$.

\subsection{Proof of Theorem~\ref{thm:refresh}}
\label{app:refresh_proof}

Define $\Omega_0=\Omega$, $\Omega_t=\bigcap_{s\le t}G_{r_s}$ and
$\widetilde E_t=E_t\one_{\Omega_t}$. Because each certificate is completed
before its state is deployed, $\Omega_t\in\Fpre$ and
$\Omega_{t-1}\in\F_{t-1}$. The constraint $0\le\lambda_t\le1/b_t$ gives
\[
 1+\lambda_t(M_t-b_t)\ \ge\ 1-\lambda_tb_t\ \ge\ 0,
\]
so both $E_t$ and $\widetilde E_t$ are nonnegative. Moreover
$\Omega_t\subseteq\Omega_{t-1}\cap G_{r_t}$, and since
$E_{t-1},\one_{\Omega_t},b_t,\lambda_t$ are all $\Fpre$-measurable,
Equation~\eqref{eq:app_one_step} gives
\[
 \E[\widetilde E_t\mid\Fpre]
 = E_{t-1}\one_{\Omega_t}
   \bigl\{1+\lambda_t(\E[M_t\mid\Fpre]-b_t)\bigr\}
 \le E_{t-1}\one_{\Omega_t}.
\]
Taking the conditional expectation given $\F_{t-1}\subseteq\Fpre$ and using
$\Omega_t\subseteq\Omega_{t-1}$ with $\Omega_{t-1}\in\F_{t-1}$ yields
\[
 \E[\widetilde E_t\mid\F_{t-1}]
 \le E_{t-1}\,\E[\one_{\Omega_t}\mid\F_{t-1}]
 \le E_{t-1}\one_{\Omega_{t-1}}
 =\widetilde E_{t-1},
\]
so the killed process is a nonnegative supermartingale with respect to
$(\F_t)_{t\ge0}$.

Let $\Omega_\infty=\bigcap_{t\ge1}\Omega_t$. On this event
$\widetilde E_t=E_t$ for every $t$, so
\[
 \Bigl\{\sup_tE_t\ge1/\edelta\Bigr\}
 \ \subseteq\
 \Bigl\{\sup_t\widetilde E_t\ge1/\edelta\Bigr\}\ \cup\ \Omega_\infty^c .
\]
Ville's inequality applied to $\widetilde E_t$, which starts at
$\widetilde E_0=1$, bounds the first probability by $\edelta$. Moreover
$\Omega_\infty^c=\bigcup_{t\ge1}G_{r_t}^c\subseteq\bigcup_{r\ge1}G_r^c$, so
Lemma~\ref{lem:familywise} bounds the second by $\caldelta$. Adding the two
bounds proves the theorem. \hfill$\blacksquare$

\begin{remark}[Total target level]\label{rem:target}
For a desired total false-alarm level $\delta$, choose
$\edelta+\caldelta\le\delta$ and use the threshold $1/\edelta$. Using
$1/\delta$ when $\caldelta>0$ would overspend the target. The same
decomposition shows that the union
$\{\sup_tE_t\ge1/\edelta\}\cup\Omega_\infty^c$ itself has probability at most
$\edelta+\caldelta$, which is the quantity measured directly in
Appendix~\ref{app:composed}.
\end{remark}

\begin{remark}[Anytime validity, precisely]
``Anytime-valid'' here refers to this time-uniform false-alarm guarantee for
evidence under the certified deployment null. It claims neither
input-conditional prediction-set coverage nor time-uniform coverage at
every stopping time.
\end{remark}

\subsection{Proof of Proposition~\ref{prop:beta}}
\label{app:beta_proof}

Let the positive integer $n_r$ and the order $k_r\in\{1,\dots,n_r+1\}$ be
$\Hcal_r$-measurable. For $1\le k_r\le n_r$, let $\gamma_r$ also be
$\Hcal_r$-measurable with $0<\gamma_r\le\rcal$. The main text states the
nondegenerate case; the complete definition also allows $k_r=n_r+1$:
\[
 q_r=
 \begin{cases}
  S_{r,(k_r)}, & 1\le k_r\le n_r,\\
  +\infty, & k_r=n_r+1 .
 \end{cases}
\]
When $k_r=n_r+1$ the answer set is all of $\Y$ and the conditional miss
probability is zero; to satisfy the strict requirement $0<b_t<1$ in this
degenerate case, use any predeclared $b_t=\varepsilon\in(0,1)$, the sure event
$G_r=\Omega$ and the allocation $\rcal=0$.

Otherwise, condition on $\Hcal_r$ and let $F_r$ be the distribution function of
the true-answer score $s_{\theta_r}(X,Y)$ of a fresh calibration question,
for the already-selected state. Let
$V_j\stackrel{\mathrm{iid}}{\sim}\operatorname{Unif}(0,1)$ be conditionally
independent of the scores and define the randomized probability-integral
transforms
\[
 U_j=F_r(S_{r,j}^{-})+V_j\bigl\{F_r(S_{r,j})-F_r(S_{r,j}^{-})\bigr\},
 \qquad F_r(s^-):=\lim_{z\uparrow s}F_r(z).
\]
Then the $U_j$ are i.i.d.\ uniform on $(0,1)$ and $U_j\le F_r(S_{r,j})$ for
every $j$. At least $k_r$ of the calibration scores satisfy
$S_{r,j}\le S_{r,(k_r)}$, and $F_r$ is nondecreasing, so at least $k_r$ of the
$U_j$ are at most $F_r(S_{r,(k_r)})$; hence
$U_{(k_r)}\le F_r(S_{r,(k_r)})$. Therefore
\[
 p_r=1-F_r(q_r)\ \le\ 1-U_{(k_r)},
 \qquad
 1-U_{(k_r)}\sim\operatorname{Beta}(n_r+1-k_r,\,k_r),
\]
using $U_{(k)}\sim\operatorname{Beta}(k,n+1-k)$ and
$1-\operatorname{Beta}(a,b)\sim\operatorname{Beta}(b,a)$. Writing $u_r$ for
the $(1-\gamma_r)$ quantile of that Beta law gives the exact certificate
\begin{equation}
 \Pbb(p_r>u_r\mid\Hcal_r)\ \le\ \gamma_r\ \le\ \rcal .
 \label{eq:beta_certificate}
\end{equation}
For continuous $F_r$, $p_r=1-U_{(k_r)}$ and the conditional failure
probability at the stated quantile equals $\gamma_r$; atoms and ties give the
displayed inequality and can make the bound conservative. The Beta quantile has
an equivalent binomial form, usually the more convenient one to compute with:
\[
 \Pbb\bigl\{\operatorname{Beta}(n+1-k,k)>u\bigr\}
 = \sum_{j=k}^{n}\binom{n}{j}(1-u)^{j}u^{\,n-j},
\]
so $u_r$ is the smallest $u$ making the
right-hand side at most $\gamma_r$. Setting $G_r=\{p_r\le u_r\}$ and $b_t=u_r$ throughout the matching
epoch then gives $\Pbb(G_r^c\mid\Hcal_r)\le\rcal$, which is the first condition in
Equation~\eqref{eq:app_budget}. \hfill$\blacksquare$

Two points deserve emphasis. Equation~\eqref{eq:beta_certificate} is a
high-probability bound, over the fresh calibration draw, on the
calibration-conditional risk; a joint marginal coverage statement over both
calibration and one test observation is not enough to make each bet's payoff
nonpositive in conditional mean, which a sequential test needs. And reuse throughout an epoch also requires
Assumption~\ref{ass:epoch}: each future deployed score must obey the selected
score law conditionally on the pre-query history, for which conditionally
i.i.d.\ sampling of calibration and future scores is a simple sufficient
condition. Exchangeability without the
corresponding conditional statement is not sufficient after arbitrary
history-dependent interventions. This view, a bound that holds with high probability over the calibration
draw, is closely related to sample-conditional split-conformal
coverage~\citep{duchi2025sampleconditional}, and for a fixed state it is the
estimand controlled by risk-controlling prediction
sets~\citep{bates2021rcps}; sequential reuse further
needs the pre-query conditional score law, and adaptive refresh
needs the trajectory-level allocation of Lemma~\ref{lem:familywise}.

\rev{%
\subsection{Why the theorem is not a corollary of its ingredients}
\label{app:shortcut}

\paragraph{Conditioning on correct certificates is not a proof.}
A natural shortcut to Theorem~\ref{thm:refresh} is: condition on
$\Omega_\infty$, note that on it every bet has nonpositive conditional mean,
apply Ville's inequality under $\Pbb(\cdot\mid\Omega_\infty)$, and add
$\Pbb(\Omega_\infty^c)\le\caldelta$. Both middle steps are false when attempts
are triggered by outcomes, which Algorithm~\ref{alg:protocol} allows.

\begin{proposition}[Given $\Omega_\infty$ the wealth is not a supermartingale, and Ville's bound fails]
\label{prop:shortcut}
Fix $b\in(0,1)$ and $d\in(0,1)$. There is a deployment that follows
Algorithm~\ref{alg:protocol} and satisfies \eqref{eq:budget} with $\caldelta=d$ and
\eqref{eq:onestep}, in which:
\begin{enumerate}[label=\textup{(\alph*)},leftmargin=2em]
\item $\E[M_1\mid\Omega_\infty]=b/\{1-d(1-b)\}>b$, so that
$\E[E_1\mid\Omega_\infty]>1=E_0$ for every constant bet $\lambda_1\in(0,1/b]$;
\item with $b=\edelta$ and $\lambda_1=1/b$,
\[
 \Pbb\Bigl(\sup_{t}E_t\ge1/\edelta\Bigm|\Omega_\infty\Bigr)
 =\frac{\edelta}{1-d(1-\edelta)}>\edelta .
\]
\end{enumerate}
\end{proposition}

\begin{proof}
Let $\F_0$ and $\F_1^{\mathrm{pre}}$ be trivial. Attempt~1 uses a known
allowance: $G_1=\Omega$, $\delta^{\mathrm{cal}}_1=0$, $b_1=b$, which is the
paper's convention for a deterministic allowance (Appendix~\ref{app:information}). Let
$\Pbb(M_1=1)=b$, so \eqref{eq:onestep} holds at $t=1$ with equality. After query~1
the operator acts on the outcome, as the post-outcome step of Algorithm~\ref{alg:protocol}
permits. If $M_1=0$, it makes attempt~2 with allocation
$\delta^{\mathrm{cal}}_2=d$ and deploys that state from $t=2$ on; the
certificate event satisfies $\Pbb(G_2^c\mid\Hcal_2)=d$, where
$\Hcal_2\supseteq\sigma(M_1)$ (Proposition~\ref{prop:beta} with $\gamma_2=d$ and a
continuous score law gives exactly this). If $M_1=1$, it makes no further
attempt and keeps state~1 for ever. For $t\ge2$ let the misses satisfy
\eqref{eq:onestep} for the state in force. Then $\sum_r\rcal\le d$ on every path,
so \eqref{eq:budget} holds with $\caldelta=d$.

On $\{M_1=1\}$ only attempt~1 is ever in force, so $\Omega_\infty\cap\{M_1=1\}
=\{M_1=1\}$. On $\{M_1=0\}$ attempts~1 and~2 are in force, so
$\Omega_\infty\cap\{M_1=0\}=G_2\cap\{M_1=0\}$. Hence
\[
 \Pbb(\Omega_\infty\cap\{M_1=1\})=b,\qquad
 \Pbb(\Omega_\infty\cap\{M_1=0\})
   =\E\bigl[\one\{M_1=0\}\,\Pbb(G_2\mid\Hcal_2)\bigr]=(1-b)(1-d),
\]
so $\Pbb(\Omega_\infty)=1-d(1-b)$ and
$\E[M_1\mid\Omega_\infty]=b/\{1-d(1-b)\}$, which exceeds $b$ because
$d(1-b)>0$. Since $\lambda_1$ is a constant,
$\E[E_1\mid\Omega_\infty]=1+\lambda_1(\E[M_1\mid\Omega_\infty]-b)>1$. This is
(a).

For (b) take $b=\edelta$ and $\lambda_1=1/b$, so that $E_1=M_1/b$. On
$\{M_1=1\}$, $E_1=1/\edelta$: an alarm. On $\{M_1=0\}$, $E_1=0$ and therefore
$E_t=0$ for all $t\ge1$: no alarm, whatever happens later. The alarm event is
thus $\{M_1=1\}$, which lies in $\Omega_\infty$, and
$\Pbb(\text{alarm}\mid\Omega_\infty)=b/\Pbb(\Omega_\infty)
=\edelta/\{1-d(1-\edelta)\}$.
\end{proof}

\begin{remark}[What the example shows, and what it does not]
\label{rem:shortcut}
The example does not contradict Theorem~\ref{thm:refresh}: there
$\Pbb(\text{alarm})=\edelta\le\edelta+\caldelta$. It shows that the
intermediate claims of the shortcut are false, so the shortcut is not a proof.
The excess in (a) is $b(1-b)d/\{1-d(1-b)\}$, small but positive. Its sign
depends on the refresh policy: if the operator made the second attempt only
after a miss, the same computation would give
$\E[M_1\mid\Omega_\infty]=b(1-d)/(1-bd)<b$. A proof cannot lean on the policy,
which is the operator's to choose.

Theorem~\ref{thm:refresh} avoids the issue by never conditioning on $\Omega_\infty$. It kills
the wealth on $\Omega_t^c$, and $\Omega_t\in\Fpre$ involves only the
certificates in force up to query $t$, never an attempt that a later outcome
may trigger. The killed-process argument gives
$\Pbb(\{\text{alarm}\}\cap\Omega_\infty)\le\edelta$, and part~(b) attains this
with equality, so that statement is sharp while the conditional one is false.
In general $\Pbb(\text{alarm}\mid\Omega_\infty)\le\edelta/(1-\caldelta)$, and
(b) shows that the factor $1/(1-\caldelta)$ cannot be dropped altogether.
\end{remark}

\paragraph{The two budgets genuinely add.}
\begin{proposition}[The bound $\edelta+\caldelta$, and its exact form under Algorithm~\ref{alg:protocol}]
\label{prop:tight}
Let $\edelta,\caldelta\in(0,1)$.
\begin{enumerate}[label=\textup{(\alph*)},leftmargin=2em]
\item \emph{A matching example.} There is a deployment with a single attempt
that satisfies \eqref{eq:budget} and \eqref{eq:onestep} and in which
$\Pbb(\sup_tE_t\ge1/\edelta)=\edelta+\caldelta-\edelta\caldelta$.
\item \emph{An upper bound for every refresh policy.} Suppose that every
attempt is made and completed between two queries, as in Algorithm~\ref{alg:protocol}, with the
information nested in the order of events: $\Hcal_r$ contains the current
wealth, the unspent calibration budget and the certificate events of all
earlier attempts, and every later pre-query field contains the certificate
event of attempt $r$, whether or not its state is deployed. Then
\[
 \Pbb\Bigl(\sup_{t}E_t\ge1/\edelta\Bigr)\ \le\ \edelta+\caldelta-\edelta\caldelta ,
\]
however many attempts are made and however they depend on the past.
\end{enumerate}
\end{proposition}

\begin{proof}
(a) Take one attempt with $\Pbb(G_1^c\mid\Hcal_1)=\caldelta$, allowance
$b_1=\edelta$ and bet $\lambda_1=1/b_1$. Let
$\Pbb(M_1=1\mid\F_1^{\mathrm{pre}})=b_1$ on $G_1$ and $=1$ on $G_1^c$; this is
allowed, since \eqref{eq:onestep} constrains the misses only on $G_1$, and
$G_1\in\F_1^{\mathrm{pre}}$. After a miss, $E_1=1+(1-b_1)/b_1=1/\edelta$, an
alarm; after a covered query, $E_1=0$ and no alarm can follow. So the alarm
probability is $\Pbb(M_1=1)=\caldelta+(1-\caldelta)\edelta$.

(b) Follow the trajectory through its steps, each either an attempt or a
query; every query is reached after finitely many steps, since only finitely
many calibrations precede it. Let $\kappa$ be the unspent calibration budget,
$\caldelta$ minus the allocations of the attempts made so far, and let $B$ be
the event that, so far, the alarm has fired or some attempted certificate is
wrong. Define the potential
\[
 \Phi=\begin{cases}1 & \text{on } B,\\
 \kappa+(1-\kappa)\,\edelta E & \text{otherwise,}\end{cases}
\]
with $E$ the current wealth. Off $B$ we have $\edelta E<1$, so
$0\le\Phi\le1$, and initially $\Phi=\caldelta+(1-\caldelta)\edelta$. By the
nesting assumed, $\Phi$ is $\Hcal_r$-measurable just before attempt $r$ and
$\Fpre$-measurable just before query $t$. We show that $\Phi$ is a
supermartingale along the steps; since $\Phi=1$ once the alarm has fired,
$\Pbb(\text{alarm by step }n)\le\E\Phi_n\le\Phi_0$ for every $n$, which is the
claim. (The argument in fact bounds the probability that the alarm fires or
some attempted certificate is wrong.)

\emph{A query, off $B$.} Every attempted certificate is correct, so
\eqref{eq:onestep} applies and the new wealth $E'$ has
$\E[E'\mid\Fpre]\le E$. Whether or not the alarm fires at this query,
$\Phi'\le\kappa+(1-\kappa)\edelta E'$: if it fires, $\edelta E'\ge1$ and the
right-hand side is at least one. Hence
$\E[\Phi'\mid\Fpre]\le\kappa+(1-\kappa)\edelta E=\Phi$.

\emph{An attempt with allocation $\eta\le\kappa$, off $B$.} Write
$x=\edelta E\in[0,1)$. The certificate is wrong with conditional probability
at most $\eta$ given $\Hcal_r$; then $\Phi'=1$. Otherwise
$\Phi'=(\kappa-\eta)+(1-\kappa+\eta)x\le1$. Since the first value is the
larger,
\[
 \E[\Phi'\mid\Hcal_r]\ \le\ \eta+(1-\eta)\{(\kappa-\eta)+(1-\kappa+\eta)x\}
 \ =\ \Phi-\eta(\kappa-\eta)(1-x)\ \le\ \Phi .
\]
On $B$ the potential stays at one.
\end{proof}

\begin{remark}[Where the sum itself is approached]
\label{rem:t3sum}
Part~(b) needs each attempt to be completed before the next query, with its
certificate event known to everything that follows. Theorem~\ref{thm:refresh} does not:
\eqref{eq:budget} conditions only on $\Hcal_r$, the history before the attempt
begins. If a block may be drawn while queries are
being served, a certificate's failure can fall where the wealth has already
been lost, and then the two budgets do not overlap at all. For example, let
state~1 have a known allowance $b_1$ with $\edelta\le b_1\le1-\caldelta$ and
miss rate exactly $b_1$,
let attempt~2 be allocated $\caldelta$ before query~1 and completed after it,
with $G_2^c\subseteq\{M_1=0\}$, and let the first bet be just below $1/b_1$.
After a miss the operator keeps state~1; the wealth is just below $1/b_1$ and
fair bets take it to $1/\edelta$ with probability approaching $\edelta/b_1$,
which contributes $b_1\cdot\edelta/b_1=\edelta$. After a covered query the
wealth is small but positive, state~2 is deployed, and on $G_2^c$ nothing
constrains its misses, so the wealth grows without bound, which contributes
$\caldelta$. The alarm probability approaches $\edelta+\caldelta$. So the bound of Theorem~\ref{thm:refresh} is
tight in the generality in which it is stated, and under the timing of
Algorithm~\ref{alg:protocol} it improves by exactly the product term.
\end{remark}

\begin{remark}[What is realized in practice]
The examples use the generality of the interface: off the certificate event
the misses are unconstrained. With the certificate of Proposition~\ref{prop:beta} a wrong
certificate leaves an excess over the allowance that is small but persistent.
Such a state is detected slowly, not rarely. In the single-certificate arm of
the composed study of Appendix~\ref{app:composed}, the share of wrong-certificate paths on which the alarm
fired was $.01$, $.07$, $.24$ and $.48$ in runs of 100, 500, 2{,}000 and
10{,}000 queries. A bettor whose wealth grows on average against a fixed
excess alarms eventually with probability one, so over a long deployment the
$\caldelta$ term is largely realized, as alarms that correctly flag a state
whose miss rate exceeds its allowance. The slack is in the $\edelta$ term: in the same arm the alarm
fired with the certificate correct on $.0018$ to $.0024$ of paths, against
$\edelta=.05$.
\end{remark}
}%
 
\section{Corollaries and operational variants}
\label{app:corollaries}

\subsection{Frozen reuse, post-selection construction and certified libraries}
\label{app:frozen}

\paragraph{Corollary~\ref{cor:design}(a): frozen reuse.}
Suppose one state and one certificate $G_1$ are reused for every query. Under
the remaining conditions of Theorem~\ref{thm:refresh} (including \eqref{eq:app_one_step},
admissible predictable bets and sticky wealth), choose a deterministic
allocation $\delta_1^{\mathrm{cal}}\le\caldelta$ with
$\Pbb(G_1^c\mid\Hcal_1)\le\delta_1^{\mathrm{cal}}$. The proof of
Appendix~\ref{app:refresh_proof} applies verbatim with a single attempt and
gives
$\Pbb(\sup_tE_t\ge1/\edelta)\le\edelta+\delta_1^{\mathrm{cal}}\le\edelta+\caldelta$.
The calibration term does not grow with the number of queries or inspections
that retain that state.

\paragraph{Resetting the monitor, legally.}
Sticky continuation is the default because it costs nothing extra and keeps
the evidence gathered before a refresh, but a deployment that wants a fresh
test after each intervention can have one. Split the evidence budget into
per-restart allocations $\edelta^{(m)}$, each fixed before its process starts,
with $\sum_m\edelta^{(m)}\le\edelta$ almost surely, and start a new wealth
process $E^{(m)}$ at each reset. Kill every $E^{(m)}$ on the certificate events
as in Appendix~\ref{app:refresh_proof}: from its start, each killed process is
a nonnegative supermartingale, so Ville's inequality applied conditionally on
the information at that start bounds its crossing probability by
$\edelta^{(m)}$. Summing over $m$, and charging certificate failures once
through Lemma~\ref{lem:familywise}, gives
$\Pbb\bigl(\exists m:\sup_tE^{(m)}_t\ge1/\edelta^{(m)}\bigr)\le\edelta+\caldelta$.
Both variants are safe; what is not safe is restarting \emph{without}
allocating, which is the failure mode Section~\ref{sec:corollaries} lists.
\rev{Appendix~\ref{app:restart} compares the two valid variants: one never-reset
process built from a restart's own processes never alarms later than the
restart, and a restarted process is capped by its own segment.}

\paragraph{Corollary~\ref{cor:design}(b): arbitrary post-selection construction.}
The state $\theta_r$ may be produced by any randomized, history-dependent
training, model-selection, retriever-selection or corpus-selection procedure
that is $\Hcal_r$-measurable. Once that single state is selected, fresh
calibration incurs no multiplicity cost for the states considered before
selection, because Proposition~\ref{prop:beta} conditions on $\Hcal_r$ and the
selection is $\Hcal_r$-measurable.

\paragraph{Corollary~\ref{cor:design}(c): simultaneously certified library.}
If selection is to occur \emph{after} viewing the fresh block, predeclare a
finite $\Hcal_r$-measurable library
$\Theta_r=\{\theta_{r,1},\dots,\theta_{r,L_r}\}$. For each member construct a
matching $\Kcal_r$-measurable event $G_{r,j}$ on which its threshold and
allowance satisfy \eqref{eq:app_one_step} whenever that member is deployed,
and choose $\Hcal_r$-measurable allocations $\delta_{r,j}$ before viewing the
block, with
\[
 \Pbb(G_{r,j}^c\mid\Hcal_r)\le\delta_{r,j},
 \qquad
 \sum_{r}\sum_{j=1}^{L_r}\delta_{r,j}\le\caldelta
 \quad\text{a.s.}
\]
The conditional union bound makes $G_r=\bigcap_{j}G_{r,j}$ a valid certificate
event with allocation $\rcal=\sum_j\delta_{r,j}$, so
Lemma~\ref{lem:familywise} and the killed-process argument apply unchanged.
The controller may then inspect the block and deploy any library member with
its matching certificate. Same-block selection across uncorrected pointwise certificates thus has two
valid alternatives: either select one state before
calibration and pay no candidate multiplicity cost, or certify a predeclared
finite library and select afterward while paying an explicit within-library
allocation.

\subsection{Cumulative envelope}
\label{app:envelope}

\paragraph{Corollary~\ref{cor:design}(d): cumulative envelope.}
Under the assumptions of Theorem~\ref{thm:refresh}, let $Z_t=M_t-b_t$ and
$A_t=\sum_{s=1}^{t}Z_s$. For a fixed $\eta>0$ define
$L_t(\eta)=\exp\{\eta A_t-\eta^2t/8\}$. On the active certificates, $Z_t$ has
conditional mean at most zero and lies in an interval of length one, so the
conditional Hoeffding lemma gives
\[
 \one_{\Omega_t}\,\E\bigl[e^{\eta(M_t-b_t)}\mid\Fpre\bigr]
 \le \one_{\Omega_t}\,e^{\eta^2/8},
\]
and the tower argument of Appendix~\ref{app:refresh_proof} makes
$L_t(\eta)\one_{\Omega_t}$ a nonnegative supermartingale. For a deterministic
finite grid $\eta_1,\dots,\eta_J>0$ and weights $w_j>0$ summing to one, the
mixture $L_t^{\mathrm{mix}}=\sum_jw_jL_t(\eta_j)$ is again a nonnegative
supermartingale when killed. Define the boundary $h_t(\edelta)$ as the unique solution of
\[
 \sum_{j=1}^{J}w_j\exp\Bigl\{\eta_jh-\frac{\eta_j^2t}{8}\Bigr\}=1/\edelta .
\]
On $\Omega_\infty$ a crossing of $A_t$ through $h_t(\edelta)$ is a crossing of
the killed mixture through $1/\edelta$, so Ville's inequality followed by the
split on $\Omega_\infty$ and Lemma~\ref{lem:familywise} gives
$\Pbb\{\exists t\ge1:A_t\ge h_t(\edelta)\}\le\edelta+\caldelta$. Consequently,
on the complementary event, every almost surely finite $(\F_t)$-stopping time
$\tau\ge1$ satisfies
\[
 \frac1\tau\sum_{t=1}^{\tau}M_t
 \ \le\ \frac1\tau\sum_{t=1}^{\tau}b_t+\frac{h_\tau(\edelta)}{\tau}.
\]
Computing this finite-mixture boundary exactly avoids relying on a stitched
closed form; mixture boundaries and time-uniform supermartingale bounds are
standard~\citep{howard2020time,howard2021nonparametric}. The Hoeffding step is,
however, loose: replacing it with an empirical-Bernstein
boundary~\citep{howard2021nonparametric}, which adapts to the realized variance
of $M_t-b_t$ rather than to its range, can tighten the envelope in low-variance
regimes while retaining time-uniform validity. We keep the Hoeffding form because it is
self-contained and its boundary is computed exactly, and flag the substitution as a natural refinement. Where the alarm
reports \emph{whether} the allowance has been exceeded, the envelope bounds, at
every stopping time, how far realized misses can run above it by chance. For a joint guarantee at the same total budget, allocate
$\delta_e^{\mathrm{alarm}}+\delta_e^{\mathrm{env}}\le\edelta$ to their two
crossing events; the shared calibration event still costs $\caldelta$ only once.

\subsection{A limited delayed-feedback result}
\label{app:delay}

\paragraph{Corollary~\ref{cor:design}(e): exogenous delayed feedback.}
Theorem~\ref{thm:refresh} assumes $M_t$ is available before a later action
uses it. A narrower feedback-arrival result covers exogenous delays. Fix one
state, a certificate $G_1\in\Kcal_1$ and a $\Kcal_1$-measurable allowance
$b\in(0,1)$; let $\delta_1^{\mathrm{cal}}\le\caldelta$ be deterministic with
$\Pbb(G_1^c\mid\Hcal_1)\le\delta_1^{\mathrm{cal}}$. Conditional on the
completed calibration field $\Kcal_1$, assume the misses are mutually
independent with $\one_{G_1}\E[M_t\mid\Kcal_1]\le\one_{G_1}b$ for every $t$,
and assume the complete nonnegative integer-valued delay schedule
$(\ell_t)_{t\ge1}$ is conditionally independent of the entire miss sequence
given $\Kcal_1$. Let
\[
 \mathcal B_u=\{t:t+\ell_t=u\},
 \qquad
 \mathcal R_u=\Kcal_1\vee\sigma((\ell_t)_{t\ge1})\vee\sigma(M_t:t+\ell_t\le u).
\]
Before the batch $\mathcal B_u$ is observed, select an
$\mathcal R_{u-1}$-measurable $\lambda_u\in[0,1/b]$ and hold it fixed within
that batch. Define
$E_0^{\mathrm{rev}}=1$,
$E_u^{\mathrm{rev}}=\prod_{v\le u}\prod_{t\in\mathcal B_v}\{1+\lambda_v(M_t-b)\}$
and $\widetilde E_u^{\mathrm{rev}}=E_u^{\mathrm{rev}}\one_{G_1}$. Batch
membership, $G_1$, $b$ and $\lambda_u$ are all measurable before the batch, so
conditional independence factorizes the batch expectation and every factor has
conditional expectation at most one:
\[
 \E[\widetilde E_u^{\mathrm{rev}}\mid\mathcal R_{u-1}]
 = E_{u-1}^{\mathrm{rev}}\one_{G_1}
   \prod_{t\in\mathcal B_u}\bigl\{1+\lambda_u(\E[M_t\mid\Kcal_1]-b)\bigr\}
 \le \widetilde E_{u-1}^{\mathrm{rev}} .
\]
Every factor is nonnegative because $\lambda_u\le1/b$. Hence
$(\widetilde E_u^{\mathrm{rev}})_{u\ge0}$ is a nonnegative supermartingale in
reveal time that is constant between nonempty reveal batches, and Ville's
inequality with the certificate split gives
$\Pbb\{\sup_uE_u^{\mathrm{rev}}\ge1/\edelta\}\le\edelta+\delta_1^{\mathrm{cal}}\le\edelta+\caldelta$.

In words: labels that arrive late do not break the guarantee, provided the
lateness has nothing to do with the outcomes. Run the wealth in the order the
labels arrive rather than the order the queries were served, fix the bet before
each arriving batch, and the same budget covers the trajectory. Delay slows
detection but leaves validity intact, as the delay-law study of
Appendix~\ref{app:modular} shows. The case to avoid is outcome-dependent
missingness, where lateness depends on whether the query was a miss: there the factorization above fails, and
the corollary no longer applies. It also does not cover
adaptive state changes or a bet changed inside a feedback-arrival batch.

\subsection{Certificate-preserving controller}
\label{app:controller-corollary}

\paragraph{Corollary~\ref{cor:design}(f): certificate-preserving controller.}
The statement is immediate from the theorem once the action class is named.
Case (i) of Appendix~\ref{app:controller-scope} leaves $r_t$, $G_{r_t}$ and the epoch unchanged and selects an action
the current certificate covers uniformly, so \eqref{eq:app_one_step} holds for
the realized $a_t$ and no new allocation is spent. Case (ii) opens a new
attempt, which Lemma~\ref{lem:familywise} already budgets, and completes it
before the next affected query so that the new $G_{r}$, $q_r$, $b_t$ and
$\lambda_t$ are $\Fpre$-measurable. In both cases the hypotheses of
Theorem~\ref{thm:refresh} hold at every query, so its conclusion is unchanged.
Inside this class, adaptivity is restricted only by pre-query
measurability, $0<b_t<1$, the pathwise calibration budget, and
state--certificate matching. The practical reading is that the constraint is on
timing rather than on ambition. A controller may use everything it knows (the wealth, the miss history, realized costs, its
own randomization) to decide what to do next, and may change the model, the retriever, the corpus, the threshold
or the bandwidth on that basis; what it may not do is use the query it is about
to answer, or deploy a state before the certificate that covers it exists.

\rev{%
\paragraph{Returning to a certified state.}
\begin{remark}[Returning to a certified state is free]
\label{rem:returning}
As written, the protocol gives the controller two moves after each outcome
(Algorithm~\ref{alg:protocol}; Appendix~\ref{app:controller-scope}): keep the state with an action its certificate
covers, or make a new attempt. A third move is covered by the same proof and
is not listed there: return to a state certified at an earlier attempt.
The proof of Theorem~\ref{thm:refresh} uses nothing about the index $r_t$ beyond its
$\Fpre$-measurability and \eqref{eq:onestep}; in particular $r_t$ need not be the
most recent attempt. If the operator switches at query $t$ to a state certified
at an earlier attempt, and \eqref{eq:onestep} holds for that state on every epoch
in which it is in force, the proof goes through verbatim:
$\Omega_t=\bigcap_{s\le t}G_{r_s}$ is an intersection over the attempts that
have been in force, each counted once however often it recurs, and
$\Pbb(\Omega_\infty^c)\le\Pbb(\bigcup_rG_r^c)\le\caldelta$. No new allocation
is spent and no new labels are needed. What must be checked is the condition
on the data: Assumption~\ref{ass:epoch} is stated for one epoch, and here it is needed on
each return, that is, the state's score law must still be the one it was
calibrated on. With that, the calibration budget and the label cost grow with
the number of \emph{certificates issued}, not with the number of \emph{changes
made}. A library certified once (Corollary~\ref{cor:design}(c)) and a class of actions
covered uniformly (Corollary~\ref{cor:design}(f); case~(i) of Appendix~\ref{app:controller-scope}) are the two
instances the paper already names.
\end{remark}
}%

\rev{%
\subsection{A budgeted restart against a never-reset process}
\label{app:restart}

Appendix~\ref{app:frozen} (``Resetting the monitor, legally'') starts a new wealth process
at each reset, with per-restart allocations fixed before each process starts
and summing to at most $\edelta$. We formalize it as follows. A \emph{budgeted
restart scheme} has restart times $0=\tau_1<\tau_2<\cdots$, which are
$(\F_t)$-stopping times. Process $m$ starts at $\tau_m$ with
$E^{(m)}_{\tau_m}=1$, runs on its segment
$I_m=\{\tau_m+1,\dots,\tau_{m+1}\}$ with admissible bets $\lambda^{(m)}_s$,
$E^{(m)}_t=\prod_{s=\tau_m+1}^{t}\{1+\lambda^{(m)}_s(M_s-b_s)\}$ for
$t\in I_m$, and is discarded at $\tau_{m+1}$. Its budget $\edelta^{(m)}>0$ is
$\F_{\tau_m}$-measurable, with $\sum_m\edelta^{(m)}\le\edelta$ almost surely,
and it alarms if $E^{(m)}_t\ge1/\edelta^{(m)}$ for some $t\in I_m$.

Throughout this subsection the allowance is taken as valid: either it is known, or the
statements are read as bounds on the probability of the alarm \emph{together
with} the event that the certificates in force are correct, which is where
Theorem~\ref{thm:refresh} places its evidence bound as well. They are not statements
conditional on that event (Proposition~\ref{prop:shortcut}).

\begin{proposition}[One never-reset process dominates any budgeted restart]
\label{prop:dominate}
Consider a budgeted restart scheme. Put $E^{(m)}_t=1$ for $t\le\tau_m$ and
$E^{(m)}_t=E^{(m)}_{\tau_{m+1}}$ for $t>\tau_{m+1}$, and let
\[
 W_t\ =\ 1+\sum_{m:\,\tau_m<t}\frac{\edelta^{(m)}}{\edelta}\,
        \bigl(E^{(m)}_t-1\bigr).
\]
Then $W_0=1$, and $W$ is the wealth of a predictable bet
$\bar\lambda_t\in[0,1/b_t]$, that is, $W_t=W_{t-1}\{1+\bar\lambda_t(M_t-b_t)\}$,
so Theorem~\ref{thm:refresh} applies to it with the single threshold $1/\edelta$. Moreover
$W_t\ge1/\edelta$ at every query $t$ at which a restarted process is at or
above its own threshold. Hence on every path the never-reset process $W$ alarms
no later than the restart scheme.
\end{proposition}

\begin{proof}
Write $w_m=\edelta^{(m)}/\edelta$, so $\sum_mw_m\le1$, and let $m$ be the
process whose segment contains $t$. Only $E^{(m)}$ changes at query $t$, so
$W_t-W_{t-1}=w_mE^{(m)}_{t-1}\lambda^{(m)}_t(M_t-b_t)$, which is
$W_{t-1}\bar\lambda_t(M_t-b_t)$ with
$\bar\lambda_t=w_mE^{(m)}_{t-1}\lambda^{(m)}_t/W_{t-1}$ (and $\bar\lambda_t=0$
if $W_{t-1}=0$). Every quantity in $\bar\lambda_t$ is $\Fpre$-measurable,
because $\tau_m<t$ and $w_m$ is $\F_{\tau_m}$-measurable. Since every
$E^{(j)}\ge0$,
\[
 W_{t-1}=1-\sum_{j\le m}w_j+\sum_{j\le m}w_jE^{(j)}_{t-1}
 \ \ge\ w_mE^{(m)}_{t-1},
\]
so $0\le\bar\lambda_t\le\lambda^{(m)}_t\le1/b_t$. If
$E^{(m)}_t\ge1/\edelta^{(m)}$, the same display at time $t$ gives
$W_t\ge1-\sum_{j\le m}w_j+w_m/\edelta^{(m)}\ge1/\edelta$, because
$w_m/\edelta^{(m)}=1/\edelta$ and $\sum_{j\le m}w_j\le1$.
\end{proof}

\begin{remark}[What the domination says]
\label{rem:mixture}
A budgeted restart is a particular, wasteful bettor inside the never-reset
composition: the mixture $W$ with the alarm raised only when a single component
has grown by $1/\edelta^{(m)}$ on its own, whereas $W$ credits what all
components have gained together. Any fixed weighting of processes that start at
different times, or that differ in the half-life of their estimates (like the
five-scale mixture of Appendix~\ref{app:modular}), is likewise one never-reset
wealth process. The statement concerns the wealth, not a bettor's estimates:
$W$ restarts its estimates at every refresh, while a single bettor that never
forgets is a different monitor (Appendix~\ref{app:x-restart} compares them).
\end{remark}

\paragraph{The same mixture with every process kept running.}
In Proposition~\ref{prop:dominate} a process is frozen when the next one
starts. Let it go on betting instead, by its own rule: for every $t>\tau_m$,
\[
 \tilde E^{(m)}_t=\prod_{s=\tau_m+1}^{t}\bigl\{1+\lambda^{(m)}_s(M_s-b_s)\bigr\},
 \qquad \lambda^{(m)}_s\in[0,1/b_s]\ \text{and}\ \F^{\mathrm{pre}}_s\text{-measurable},
\]
with $\tilde E^{(m)}_t=1$ for $t\le\tau_m$, and let
$\tilde W_t=1+\sum_{m:\,\tau_m<t}w_m\,(\tilde E^{(m)}_t-1)$ with the same
weights $w_m=\edelta^{(m)}/\edelta$, so that $\tilde W_0=1$.

\begin{proposition}[The same mixture with every process kept running]
\label{prop:kept}
Keep every process of a budgeted restart scheme betting after its segment, and weight the processes by $w_m=\edelta^{(m)}/\edelta$. The process $\tilde W$ defined above is the wealth of a predictable bet in $[0,1/b_t]$, it alarms no later than the restart scheme on every path on which the budgets sum to at most $\edelta$, and in general neither $\tilde W$ nor the frozen mixture $W$ of Proposition~\ref{prop:dominate} dominates the other.
\end{proposition}

\begin{proof}
\emph{(i) $\tilde W$ is the wealth of a predictable bet.} At query $t$ every
started process moves. A process with $\tau_j=t-1$ has
$\tilde E^{(j)}_{t-1}=1$, so it may be counted among the started ones at
$t-1$ at no cost, and
\[
 \tilde W_t-\tilde W_{t-1}=\tilde W_{t-1}\,\tilde\lambda_t\,(M_t-b_t),
 \qquad
 \tilde\lambda_t=\frac{\sum_{j:\,\tau_j<t}w_j\,\tilde E^{(j)}_{t-1}\,\lambda^{(j)}_t}
 {\tilde W_{t-1}}
\]
(and $\tilde\lambda_t=0$ if $\tilde W_{t-1}=0$). Every quantity in
$\tilde\lambda_t$ is $\Fpre$-measurable: the event $\{\tau_j<t\}$ lies in
$\F_{t-1}$, on it $w_j$ and $\tilde E^{(j)}_{t-1}$ are $\F_{t-1}$-measurable,
and $\lambda^{(j)}_t$ is $\Fpre$-measurable by assumption. Since
$\sum_jw_j\le1$ and every $\tilde E^{(j)}\ge0$,
\[
 \tilde W_{t-1}=\Bigl(1-\sum_{j:\,\tau_j<t}w_j\Bigr)
 +\sum_{j:\,\tau_j<t}w_j\,\tilde E^{(j)}_{t-1}
 \ \ge\ \sum_{j:\,\tau_j<t}w_j\,\tilde E^{(j)}_{t-1},
\]
and every $\lambda^{(j)}_t\le1/b_t$, so $0\le\tilde\lambda_t\le1/b_t$. Hence
Theorem~\ref{thm:refresh} applies to $\tilde W$ with the single threshold $1/\edelta$.

\emph{(ii) $\tilde W$ alarms no later than the restart scheme}, on every path
on which the budgets sum to at most $\edelta$, as in
Proposition~\ref{prop:dominate}. On its own segment a process is the
restart's: $\tilde E^{(m)}_t=E^{(m)}_t$ for $t\in I_m$, because the bets are
the same there. The last display, written at time $t$, is a sum of
nonnegative terms, the bracket being one of them, so
$\tilde W_t\ge w_m\tilde E^{(m)}_t$ for every $t>\tau_m$. If
$E^{(m)}_t\ge1/\edelta^{(m)}$ for a $t\in I_m$, then
$\tilde W_t\ge w_mE^{(m)}_t\ge w_m/\edelta^{(m)}=1/\edelta$.

\emph{(iii) In general neither of $W$ and $\tilde W$ dominates the other.}
($W$ is $\tilde W$ for a scheme whose processes bet nothing after their
segments.) A kept process may lose after its segment what a frozen one
retains, for instance once a refresh has repaired the harm, and it may gain
what a frozen one forgoes, when the harm outlasts the refresh. For example,
take a constant $b$, $\tau_2=1$, $w_1=w_2=\tfrac12$, $\lambda^{(1)}_s=1/b$ and
$\lambda^{(2)}_2=0$. After $M_1=1$ the first process stands at $1/b$, and
$W_2=\tfrac12+\tfrac1{2b}$ whatever $M_2$ is, while
$\tilde W_2=\tfrac12$ if $M_2=0$ and $\tilde W_2=\tfrac12+\tfrac1{2b^2}$ if
$M_2=1$.
\end{proof}

The start-time mixture of Proposition~\ref{prop:delaybound} is $\tilde W$ for the scheme with
$\tau_m=(m-1)L$, budgets $\edelta^{(m)}=\edelta/\{m(m+1)\}$ and, as the bettor
of each process, the equal mix of the $J$ fixed stakes, which is one
predictable bet; its weights sum to one, so there
$\tilde W_t=\sum_mw_m\tilde E^{(m)}_t$.

For $b<p<1$ and an integer $L\ge1$ let $\beta_L(\delta)=\beta_L(\delta;b,p)$ be
the power of the most powerful level-$\delta$ test of $\Bern(b)$ against
$\Bern(p)$ from $L$ independent observations, randomization allowed. The
likelihood ratio increases with the number of misses $N$, so
$\beta_L(\delta)=\Pbb_p(N>k)+\rho\,\Pbb_p(N=k)$, where $N$ is binomial and $k$
and $\rho\in[0,1)$ solve $\Pbb_b(N>k)+\rho\,\Pbb_b(N=k)=\delta$. It is an
explicit binomial quantity.

\begin{proposition}[The evidence a restarted process can gather is capped by its own segment]
\label{prop:ceiling}
Consider a budgeted restart scheme and one of its processes, $m$.
\begin{enumerate}[label=\textup{(\alph*)},leftmargin=2em]
\item \emph{Pathwise cap.} On every outcome sequence and for all admissible
bets,
\[
 \sup_{t\in I_m}E^{(m)}_t\ \le\ \prod_{s\in I_m:\,M_s=1}\frac1{b_s} .
\]
Hence, if $b_s\ge\underline b$ on $I_m$, the process cannot alarm unless the
number $N_m$ of misses in its segment satisfies
$N_m\log(1/\underline b)\ge\log(1/\edelta^{(m)})$.
In particular a process whose segment has at most $L$ queries cannot alarm,
whatever the data, once $\edelta^{(m)}<\underline b^{\,L}$.
\item \emph{A ceiling for every bettor.} Suppose that, given $\F_{\tau_m}$, the
segment has a fixed length and a constant allowance $b$; that
$\E[M_s\mid\F_s^{\mathrm{pre}}]\le b$ at its first $L_0\ge0$ queries; and that
$\Pbb(M_s=1\mid\F_s^{\mathrm{pre}})=p$ at each of its remaining $L\ge1$
queries, for a constant $p\in(b,1)$. Then, for every admissible bettor,
\[
 \Pbb\bigl(\text{process $m$ alarms}\,\bigm|\,\F_{\tau_m}\bigr)
 \ \le\ \beta_L\bigl(\edelta^{(m)}\bigr).
\]
When $L_0=0$ and $\edelta^{(m)}\ge b^{L}$ the bound is attained, up to the
randomization of the Neyman--Pearson test at its boundary, by a bettor that
knows $L$ but not $p$.
\item \emph{An explicit form.} With $c=\log\frac{p(1-b)}{b(1-p)}$, for every
$a>0$,
\[
 \beta_L(\delta)\ \le\ \delta\exp\{L\,\KL(p\,\|\,b)+a\}
        +\exp\Bigl\{-\frac{2a^2}{Lc^2}\Bigr\}.
\]
\end{enumerate}
\end{proposition}

\begin{proof}
(a) A miss multiplies the wealth by
$1+\lambda_s(1-b_s)\le1+(1-b_s)/b_s=1/b_s$, since $\lambda_s\le1/b_s$; a
covered query multiplies it by $1-\lambda_sb_s\in[0,1]$. Multiplying over
$s\in(\tau_m,t]$ gives the bound for each $t\in I_m$, because every factor
$1/b_s$ is at least one. If $b_s\ge\underline b$, the right-hand side is at
most $\underline b^{-N_m}$, and $\underline b^{-N_m}\ge1/\edelta^{(m)}$ is the
stated condition. Since $N_m\le L$, no alarm is possible when
$\underline b^{-L}<1/\edelta^{(m)}$.

(b) Work conditionally on $\F_{\tau_m}$, so that $\delta:=\edelta^{(m)}$ is a
constant, and write $A$ for the event that process $m$ alarms. Let $S$ be the
set of the last $L$ queries of the segment, $V=(M_s)_{s\in S}$, and
\[
 D\ =\ \prod_{s\in S}\Bigl(\frac bp\Bigr)^{M_s}
       \Bigl(\frac{1-b}{1-p}\Bigr)^{1-M_s}
 \ =\ \frac1{\Lambda(V)},\qquad
 \Lambda(v)=\prod_{s\in S}\Bigl(\frac pb\Bigr)^{v_s}
       \Bigl(\frac{1-p}{1-b}\Bigr)^{1-v_s}.
\]
Because $\Pbb(M_s=1\mid\F_s^{\mathrm{pre}})=p$ on $S$, each factor of $D$ has
conditional mean $p\cdot\frac bp+(1-p)\cdot\frac{1-b}{1-p}=1$ given
$\F_s^{\mathrm{pre}}$. So $D$ is a probability density; let $\Q$ be the law
with $d\Q/d\Pbb=D$. Under $\Q$ the conditional laws at the first $L_0$ queries
are unchanged, and $\Q(M_s=1\mid\F_s^{\mathrm{pre}})=p\cdot\frac bp=b$ for
$s\in S$. Hence under $\Q$ every bet has nonpositive conditional mean, the
wealth $E^{(m)}$ is a nonnegative supermartingale on the segment, started at
one, and Ville's inequality gives $\Q(A)\le\delta$. Moreover $V$ is i.i.d.\
$\Bern(b)$ under $\Q$.

Let $\varphi(v)=\Q(A\mid V=v)$, a randomized test based on $V$. Then
$\E_{\Bern(b)^L}\varphi=\Q(A)\le\delta$, and since $d\Pbb/d\Q=\Lambda(V)$,
\[
 \Pbb(A)=\E_\Q[\Lambda(V)\one_A]=\E_\Q[\Lambda(V)\varphi(V)]
 =\sum_{v}b^{|v|}(1-b)^{L-|v|}\Lambda(v)\varphi(v)
 =\E_{\Bern(p)^L}\varphi ,
\]
where $|v|$ is the number of misses in $v$. The Neyman--Pearson lemma gives
$\E_{\Bern(p)^L}\varphi\le\beta_L(\delta)$. No independence between the bets
and the past is used: the bettor may use anything in $\F_s^{\mathrm{pre}}$.

For the attainment let $L_0=0$, let $N_t$ be the number of misses among the
first $t$ queries of the segment, and let $A^\star=\{N_L>k\}$ with $k$ as in
the definition of $\beta_L(\delta)$, so that $\Q(A^\star)\le\delta$ and
$\Pbb(A^\star)\ge\beta_L(\delta)-\Pbb_p(N=k)$. The region $A^\star$ does not
depend on $p$, and $\delta\ge b^L$ makes $k\le L-1$, so it is nonempty. (When
$\delta<b^L$ part~(a) already rules out an alarm, and $\beta_L(\delta)$ is not
sharp.) Put $g_t(n)=\Q(A^\star\mid N_t=n)$, which is nondecreasing
in $n$ and satisfies $g_{t-1}(n)=b\,g_t(n+1)+(1-b)\,g_t(n)$. Consider the
wealth $E_t=g_t(N_t)/g_0(0)$. It starts at one, and it is the wealth of the bet
$\lambda_t=\{1-g_t(n)/g_{t-1}(n)\}/b$ at $N_{t-1}=n$ (any bet if
$g_{t-1}(n)=0$), which lies in $[0,1/b]$ because
$0\le g_t(n)\le g_{t-1}(n)$. On $A^\star$,
$E_L=1/\Q(A^\star)\ge1/\delta$: the process alarms.

(c) Let $\ell_s=M_s\log\frac pb+(1-M_s)\log\frac{1-p}{1-b}$, so that
$\Lambda=\exp(\sum_s\ell_s)$ is the likelihood ratio of $\Bern(p)^{L}$ to
$\Bern(b)^{L}$. For any test $\varphi$ of level $\delta$ and any real $y$,
\[
 \E_p\varphi=\E_b[\Lambda\varphi]
 \le e^{y}\,\E_b\varphi+\E_b[\Lambda\one\{\Lambda>e^{y}\}]
 \le\delta e^{y}+\Pbb_p\Bigl(\sum_{s}\ell_s>y\Bigr).
\]
Take $y=L\,\KL(p\,\|\,b)+a$. Under $\Bern(p)$ the $\ell_s$ are i.i.d.\ with
mean $\KL(p\,\|\,b)$ and take two values $c$ apart, so Hoeffding's inequality
bounds the last probability by $\exp\{-2a^2/(Lc^2)\}$.
\end{proof}

\begin{corollary}[Across segments]
\label{cor:ceiling}
Let the restart times and the budgets $\edelta^{(m)}$ be deterministic and the
allowance a constant $b$. Let $\E[M_s\mid\F_s^{\mathrm{pre}}]\le b$ for
$s\le T_0$ and $\Pbb(M_s=1\mid\F_s^{\mathrm{pre}})=p$ for $s>T_0$, with a
constant $p\in(b,1)$. Then, for all admissible bettors and every $h\ge1$,
\[
 \Pbb\bigl(\text{some restarted process alarms at a query in }(T_0,T_0+h]\bigr)
 \ \le\ 1-\prod_{m}\Bigl\{1-\beta_{L_m(h)}\bigl(\edelta^{(m)}\bigr)\Bigr\},
\]
where $L_m(h)$ is the number of queries of $I_m$ in $(T_0,T_0+h]$ and the
product runs over the processes with $L_m(h)\ge1$. For one never-reset process
the corresponding bound on an alarm by query $T_0+h$ is $\beta_h(\edelta)$,
and it is never below the right-hand side.
\end{corollary}

\begin{proof}
Let $A_m$ be the event that process $m$ alarms at a query in $(T_0,T_0+h]$. It
is contained in the alarm event of the process stopped at $T_0+h$, to which
Proposition~\ref{prop:ceiling}(b) applies with $L=L_m(h)$, so
$\Pbb(A_m\mid\F_{\tau_m})\le\beta_{L_m(h)}(\edelta^{(m)})=:\beta_m$. Since
$A_j\in\F_{\tau_m}$ for $j<m$,
$\Pbb(\bigcap_{j\le m}A_j^c)=\E[\one_{\bigcap_{j<m}A_j^c}\,
\Pbb(A_m^c\mid\F_{\tau_m})]\ge(1-\beta_m)\,\Pbb(\bigcap_{j<m}A_j^c)$, and the
first claim follows by induction.

For the never-reset process, apply Proposition~\ref{prop:ceiling}(b) to the single
segment of queries $1,\dots,T_0+h$, with $L_0=T_0$, $L=h$ and budget
$\edelta$.

For the comparison, let $\varphi_m$ be the most powerful test of level
$\edelta^{(m)}$ on the $L_m(h)$ post-change observations of segment $m$. The
segments are disjoint, so $\varphi=1-\prod_m(1-\varphi_m)$ is a test on the
$h$ post-change observations with level
$1-\prod_m(1-\edelta^{(m)})\le\sum_m\edelta^{(m)}\le\edelta$ and power
$1-\prod_m(1-\beta_m)$. A test of level at most $\edelta$ on $h$ observations
has power at most $\beta_h(\edelta)$.
\end{proof}

\begin{remark}[Reading the results]
\label{rem:ceiling}
(i) The ceiling is small once $\log(1/\edelta^{(m)})$ exceeds the segment's
mean log-likelihood ratio $L\,\KL(p\,\|\,b)$ by a margin that is large both in
nats and in standard deviations; part~(c) is the Hoeffding form of this, and the
tables below use the exact $\beta$. Under any spending rule that lets the number
of restarts grow, $\edelta^{(m)}\to0$, so with segments of bounded length the
ceiling tends to zero: with the summable rule and segments of at most $L$
queries, part~(a) makes every process with
$m>\sqrt{6\edelta/\pi^2}\;\underline b^{-L/2}$ powerless.

(ii) The never-reset process has no segment: its wealth multiplies its factors
over all queries, so evidence gathered in one epoch is kept in the next and is
tested against the single threshold $1/\edelta$.

(iii) When a segment is several times longer than
$\log(1/\edelta^{(m)})/\KL(p\,\|\,b)$ the ceiling is close to one
(Table~\ref{tab:ceiling}). The cost of restarting arises when the system is
refreshed more often than the evidence scale; there, with the same component
bettors, keeping the wealth is never worse (Proposition~\ref{prop:dominate})
and much better.
\end{remark}

\begin{table}[ht]\rev{%
\caption{Exact ceilings for a budgeted restart: the bound of
Corollary~\ref{cor:ceiling} on the probability that \emph{any} restarted bettor has
alarmed within $h$ post-change queries. Known allowance $b=.15$, post-change
miss rate $b+\text{excess}$, $\edelta=.025$, change after query 500, a restart
after every \emph{cadence} queries, under the summable rule and under
Bonferroni with cap $\lceil5{,}500/\text{cadence}\rceil$. Nothing is simulated;
the numbers are exact binomial computations.}
\label{tab:ceiling}
\begin{center}
\small
\begin{tabular}{@{}rr rrr rrr@{}}
\toprule
& & \multicolumn{3}{c}{Summable}
  & \multicolumn{3}{c}{Bonferroni} \\
\cmidrule(lr){3-5}\cmidrule(l){6-8}
Excess & Cadence & $h{=}1{,}000$ & $2{,}000$ & $5{,}000$
       & $1{,}000$ & $2{,}000$ & $5{,}000$ \\
\midrule
.02 & 100 & .014 & .018 & .022 & .031 & .062 & .147 \\
 & 250 & .051 & .069 & .088 & .062 & .120 & .274 \\
 & 500 & .122 & .171 & .232 & .111 & .209 & .444 \\
 & 1,000 & .243 & .359 & .509 & .159 & .318 & .636 \\
 & 2,500 & .338 & .615 & .828 & .260 & .526 & .850 \\
\addlinespace
.04 & 100 & .072 & .095 & .122 & .138 & .257 & .524 \\
 & 250 & .294 & .399 & .525 & .334 & .557 & .869 \\
 & 500 & .589 & .759 & .907 & .569 & .815 & .985 \\
 & 1,000 & .761 & .936 & .996 & .660 & .924 & .999 \\
 & 2,500 & .889 & .995 & 1.000 & .842 & .991 & 1.000 \\
\addlinespace
.08 & 100 & .552 & .696 & .833 & .730 & .927 & .999 \\
 & 250 & .966 & .996 & 1.000 & .976 & .999 & 1.000 \\
 & 500 & .999 & 1.000 & 1.000 & .999 & 1.000 & 1.000 \\
 & 1,000 & 1.000 & 1.000 & 1.000 & .999 & 1.000 & 1.000 \\
 & 2,500 & 1.000 & 1.000 & 1.000 & 1.000 & 1.000 & 1.000 \\
\bottomrule
\end{tabular}
\end{center}
}%
\end{table}

\begin{table}[ht]\rev{%
\caption{One never-reset process in the setting of Table~\ref{tab:ceiling}, at
unit wealth when the change occurs. \emph{Ceiling}: $\beta_h(\edelta)$, which
bounds the probability of an alarm within $h$ post-change queries for every
bettor and is attained up to the randomization of the test. \emph{Oracle bet}: the exact probability for the bet $\lambda^\star$
of Remark~\ref{rem:oracle}, which needs $p$ but no horizon, with its exact mean
and median delay in queries.}
\label{tab:never}
\begin{center}
\small
\begin{tabular}{@{}r rrr rrr rr@{}}
\toprule
& \multicolumn{3}{c}{Ceiling} & \multicolumn{3}{c}{Oracle bet}
& \multicolumn{2}{c}{Oracle delay} \\
\cmidrule(lr){2-4}\cmidrule(lr){5-7}\cmidrule(l){8-9}
Excess & $h{=}1{,}000$ & $2{,}000$ & $5{,}000$
       & $1{,}000$ & $2{,}000$ & $5{,}000$ & mean & median \\
\midrule
.02 & .415 & .689 & .971 & .164 & .515 & .913 & 2,465 & 1,947 \\
.04 & .923 & .997 & 1.000 & .832 & .978 & 1.000 & 643 & 507 \\
.08 & 1.000 & 1.000 & 1.000 & .999 & 1.000 & 1.000 & 173 & 136 \\
\bottomrule
\end{tabular}
\end{center}
}%
\end{table}

\paragraph{What a never-reset mixture pays for where its processes start.}
In the mixture $\tilde W$ of Proposition~\ref{prop:kept}, process~$m$ enters
with weight $w_m=\edelta^{(m)}/\edelta$, and $\tilde W_t$ is a sum of
nonnegative terms one of which is $w_m\tilde E^{(m)}_t$ (part~(ii) there). So
$\tilde W$ alarms no later than process~$m$ reaches its own restart threshold
$1/\edelta^{(m)}$: keeping process~$m$ costs at most $\log(1/w_m)$ in log-wealth over
the threshold $1/\edelta$ (exactly that much when the other processes stand at or below one). These costs cannot all be small, because the weights
sum to at most one. With $J$ processes some costs at least $\log J$, with
equality only for equal weights $1/J$ (a declared cap); with infinitely many,
for every constant $C$ infinitely many cost more than $\log m+C$ (the summable
weights cost $2\log m+\log(\pi^2/6)$). No choice of weights is therefore the
cheapest at every process, in keeping with the finding that no monitor of
Appendix~\ref{app:x-restart} is the soonest in every cell. When each restart starts a fresh copy of the same bettor and the weights are summable,
the mixture that keeps a never-reset bettor as its first process
needs at most $\log(\pi^2/6)<0.498$ more log-wealth than that bettor alone and
alarms no later than the summable restart; with a declared cap~$J$, it alarms
no later than the capped restart.
}%

\rev{%
\subsection{Calibration blocks from the traffic the monitor already labels}
\label{app:traffic}

The paper's scope requires the
adjudicated answer of every query the monitor bets on. Those labels can also
pay for calibration.

\begin{remark}[A shadow block is a fresh block]
\label{rem:traffic}
Suppose that after query $t_0$, a stopping time, the operator decides on
attempt $r$: the state $\theta_r$, the block size $n_r$, the order $k_r$ and
the level $\gamma_r\le\rcal$ are fixed from $\F_{t_0}$, and
$\Hcal_r=\F_{t_0}$. Let the
block be the next $n_r$ labelled queries, scored by $\theta_r$ in shadow while
the state in force goes on serving them:
$S_{r,i}=s_{\theta_r}(X_{t_0+i},Y_{t_0+i})$, $i=1,\dots,n_r$, with
$\theta_r$'s own deployment randomization. Compute $q_r$ and $u_r$ from this
block as in Proposition~\ref{prop:beta} and deploy $\theta_r$ from query $t_0+n_r+1$.

Assume the sufficient condition named in Assumption~\ref{ass:epoch}, extended back over the
window: from $t_0$ to the end of $\theta_r$'s epoch the questions are drawn
i.i.d.\ from one law, each independently of everything observed before it.
Then, given $\Hcal_r$, the $S_{r,i}$ are i.i.d.\ from the law $F_r$ of
$\theta_r$'s true-answer score on a fresh question (on each $\{t_0=k\}$, by
induction over the window), so Proposition~\ref{prop:beta} applies as stated and
$\Pbb(p_r>u_r\mid\Hcal_r)\le\gamma_r$ with $p_r=1-F_r(q_r)$. During the window
the state in force still satisfies \eqref{eq:onestep}: its question is independent
of $\Fpre$, which now also contains the earlier shadow scores. For a query
$t>t_0+n_r$ served by $\theta_r$, the question is independent of $\Fpre$, which
contains the block, the threshold and the wealth, so
$\E[M_t\mid\Fpre]=p_r\le u_r$ on $G_r$: this is \eqref{eq:onestep} again.
Theorem~\ref{thm:refresh} then applies with no change, with its bound $\edelta+\caldelta$.
The product form of Proposition~\ref{prop:tight}(b) is not available here,
because the block is drawn while queries are served: this is the case that
Remark~\ref{rem:t3sum} sets aside, and its example is a shadow block of one
query.

Five things are worth saying.
\begin{enumerate}[label=\textup{(\roman*)},leftmargin=2em]
\item \emph{What the condition costs.} The paper states Assumption~\ref{ass:epoch} per
epoch, so that a refresh is free to change the score law. A shadow block asks
for more: the question law must be the same over the block's window and the
epoch it certifies. With a shadow block at every refresh the windows and the
epochs overlap, so the question stream has to be stationary over the whole
deployment and the score law can change only through the state. A change of
population that is timed with a refresh still needs a block drawn from the new
population, labelled separately.

\item The block and the outcomes of the state in force on the same $n_r$
queries are dependent: they are functions of the same questions. The theorem
does not ask for their independence. Condition \eqref{eq:budget} is conditional
on $\Hcal_r$, which precedes the block, and condition \eqref{eq:onestep} is
conditional on $\Fpre$, which follows it.
\item No label is used beyond those the monitor needs for its own bets on
queries $t_0+1,\dots,t_0+n_r$. The costs of a shadow block are of other kinds:
latency, $n_r$ labelled queries between the decision and the deployment,
during which the old state goes on serving, also after an alarm; the shadow
scoring itself, which for a node-side change means that every node whose scores
change uploads the new state's scores for those queries; and the first certificate of a
deployment, which has no traffic to draw on. If the monitor labels only
queries picked by a rule that decides before the query is seen, the block is
the next $n_r$ labelled ones and nothing else changes.
\item Nothing fixed at $t_0$ may be revised inside the window. An attempt
abandoned during its window is charged like any discarded attempt; a revised
plan is a new attempt whose block starts after the revision; and several
candidate states shadowed on one window need an allocation each
(Corollary~\ref{cor:design}(c)). Stopping the block when its threshold first looks
favourable is an unbudgeted search.
\item A look-back block is different. The $n_r$ labelled queries \emph{before}
the decision are $\Hcal_r$-measurable, so Proposition~\ref{prop:beta} cannot be used with
$\Hcal_r=\F_{t_0}$. A smaller conditioning field works only if the time of the
decision, $\theta_r$, $n_r$, $k_r$, $\gamma_r$ and the allocation are all
independent of those queries, including of the misses and the wealth they
produced: a controller that reacts to the monitor has seen them. Without
latency, the block must come from labelled queries the monitor does not bet
on or from a block labelled offline, or the operator returns to a state that
is already certified (Remark~\ref{rem:returning}).
\end{enumerate}
\end{remark}
}%

\rev{%
\subsection{When Assumption~\ref{ass:epoch} fails inside an epoch, and what can be checked at runtime}
\label{app:runtime}

Assumption~\ref{ass:epoch} is used once, to pass from
Proposition~\ref{prop:beta} to \eqref{eq:onestep}. Theorem~\ref{thm:refresh} itself needs only
\eqref{eq:budget} and \eqref{eq:onestep}.

\begin{remark}[Three cases]
\label{rem:three-cases}
Let the score law drift inside an epoch, for any reason, including a node-side
change the hub did not certify, and write
$\tilde p_t=\E[M_t\mid\Fpre]$ for the true conditional miss probability at the
threshold in force.
\begin{enumerate}[label=\textup{(\roman*)},leftmargin=2em]
\item If $\tilde p_t\le b_t$ on $G_{r_t}$ for every $t$, then \eqref{eq:onestep}
holds and Theorem~\ref{thm:refresh} applies unchanged: the false-alarm bound is intact however
far the score law has moved.
\item If $\tilde p_t>b_t$ at some queries, the deployment null is false there.
An alarm is then not a false alarm, and the theorem makes no claim about it.
With the oracle bet for that query the expected log-wealth increment is
$\KL(\tilde p_t\,\|\,b_t)>0$.
\item The monitor does not, and is not meant to, detect a drift of type (i).
A drift that leaves the miss rate between the target $\alpha$ and the allowance
is of that type.
\end{enumerate}
\end{remark}

\begin{remark}[A stability test may trigger a refresh at no cost to validity]
\label{rem:trigger}
By Corollary~\ref{cor:design}(f) a controller may use all of $\F_t$ to decide when to
attempt a refresh. In particular the crossing of any statistic of past
scores, such as a conformal test martingale for exchangeability, may be used as
a trigger. Only the attempt it starts is charged, through $\rcal$. The wealth
process remains the sole alarm, so the trigger needs no error budget of its own.
\end{remark}

\begin{remark}[State attestation]
\label{rem:attest}
Let each node attach to its message at query $t$ a digest of the components of its state (corpus
version, retriever, model and encoder), and let the hub compare the digests with those recorded
when the certificate in force was issued, before it forms the answer set. The comparison uses
only the messages' headers, so its outcome is $\Fpre$-measurable in the filtration enlarged by
them, and by Corollary~\ref{cor:design}(f) the controller may respond to a mismatch or a missing
message by completing select--calibrate--certify--deploy before the query is served, or by
serving it with an action the certificate in force covers. Either way \eqref{eq:app_budget} and
\eqref{eq:app_one_step} are untouched, so the bound of Theorem~\ref{thm:refresh} is unchanged.
Attestation therefore turns every \emph{declared} change of a node's components, and every
outage, into an event the hub sees before the next outcome. It does not see a node that reports a
stale digest, which is a Byzantine fault outside the scope of Section~\ref{sec:setup}, nor a change
in the question stream, which is what the flag below and the monitor are for.
\end{remark}

\begin{remark}[A flag that needs no label]
\label{rem:flag}
Let $Z_t\in\{0,1\}$ be an indicator that the hub observes at query $t$ without
the label, a function of the question and of the state in force, with any
constant in its definition fixed from $\Hcal_r$. (The construction applies to any such
indicator, for instance whether the aggregated score of the option the hub ranks first is above
a constant fixed before the block; nothing below depends on which indicator is used.) On the second benchmark $Z_t=1$ when
no node's best passage reaches the gate; under rule~(i) that is the event
that no node answered. Under the paper's own mean over all nodes the hub sees
neither silence nor similarities, and this $Z$ is not available to it. For
attempt $r$ write $\zeta_r$ for the probability, given $\Hcal_r$, that $Z=1$
for the state $\theta_r$ on a fresh question from the attempt's calibration
law.

Suppose that attempt $r$ also fixes from $\Hcal_r$ a level $\gamma'_r\ge0$,
with $\sum_r\gamma'_r\le\delta'_{\mathrm{cal}}$ almost surely, and computes
from its block a bound $v_r\in(0,1]$ with
\begin{equation}
 \Pbb(\zeta_r>v_r\mid\Hcal_r)\le\gamma'_r .
 \label{eq:flagbudget}
\end{equation}
For a block of $n$ conditionally i.i.d.\ questions, $c$ of which have $Z=1$,
the one-sided Clopper--Pearson bound at a level $\gamma'_r>0$ is such a $v_r$;
it equals one when $c=n$. With a library certified at once
(Corollary~\ref{cor:design}(c)) the level is split over its members. Let
a second wealth process bet on $Z_t$: $W_0=1$,
$W_t=W_{t-1}\{1+\mu_t(Z_t-v_{r_t})\}$, with $\mu_t\in[0,1/v_{r_t}]$ measurable
with respect to $\Fpre$. Like $b_t$ and $G_{r_t}$, the bound $v_{r_t}$ and the
event $\{\zeta_{r_t}\le v_{r_t}\}$ are functions of $\Hcal_{r_t}$ and the
block, hence measurable with respect to $\Fpre$, and $Z_t$ is
$\F_t$-measurable. Suppose finally that
\begin{equation}
 \one_{\{\zeta_{r_t}\le v_{r_t}\}}\,\E[Z_t\mid\Fpre]\le\one_{\{\zeta_{r_t}\le v_{r_t}\}}\,v_{r_t},
 \label{eq:flagnull}
\end{equation}
which holds under the sufficient condition named in Assumption~\ref{ass:epoch}, because
$Z_t$ is a function of a fresh question and of the state, so that
$\E[Z_t\mid\Fpre]=\zeta_{r_t}$. Then \eqref{eq:flagbudget} and
\eqref{eq:flagnull} are \eqref{eq:budget} and \eqref{eq:onestep} with
$(Z_t,v_r,\gamma'_r,\{\zeta_r\le v_r\})$ in place of $(M_t,b_t,\rcal,G_r)$, and
while every bound is below one Theorem~\ref{thm:refresh} gives, for every
$\delta'\in(0,1)$,
\begin{equation}
 \Pbb\Bigl(\sup_tW_t\ge1/\delta'\Bigr)\le\delta'+\delta'_{\mathrm{cal}} .
 \label{eq:flagbound}
\end{equation}
A bound of one needs a word. The setting takes $b_t\in(0,1)$, and
Appendix~\ref{app:controller-scope} calls a bound of one uninformative and inadmissible
for the theorem, and forbids clipping it below one. The flag keeps it at one.
Every factor $1+\mu_t(Z_t-1)$ of such an epoch lies in $[0,1]$, so $W$ lies
below the process $W'$ that does not bet in such an epoch, and $W'$ is an
instance of the setting as stated: in such an epoch take the indicator $0$,
any predeclared allowance in $(0,1)$ and the bet $0$, the device of
Appendix~\ref{app:beta_proof} for a degenerate certificate; the event $\{\zeta_r\le v_r\}$ is
then the sure event, so the budget is as before. Hence \eqref{eq:flagbound}
holds for $W'$, and so for $W$.

The two certificates of an attempt, $u_r$ and $v_r$, may come from the same
block: each bound is a statement of its own given $\Hcal_r$, and their
dependence is not used. With both nulls in force, the probability that either
process ever crosses is at most
$\edelta+\caldelta+\delta'+\delta'_{\mathrm{cal}}$, by the union bound. Four
things are worth saying.
\begin{enumerate}[label=\textup{(\roman*)},leftmargin=2em]
\item \emph{It is not an alarm.} The flag does not test the miss rate, and
neither null implies the other. A node can fall silent with the miss rate
unmoved, because the hub falls back on what the student knows; and the miss
rate can rise with the rate of $Z$ unchanged, as after a change of student.
A crossing is evidence, at level $\delta'+\delta'_{\mathrm{cal}}$, that the
rate of $Z$ exceeds its bound. What raises that rate is a change in what the
nodes have to offer for the traffic: a node lost or changed its corpus, the
retriever changed, or the traffic moved to questions for which no node has
passages. The wealth process on the misses remains the sole alarm.
\item \emph{Its use.} By Remark~\ref{rem:trigger} any statistic of the past may
trigger a refresh at no cost to the alarm's bound. The flag is such a
trigger, and it also has what that remark did not ask for: a bound,
$\delta'+\delta'_{\mathrm{cal}}$, on the probability that it ever fires while
\eqref{eq:flagnull} holds, as it does while the nodes' retrieval is as at
calibration and the questions are drawn from the calibration law
independently of the past. A flag that is restarted after a trigger needs the
budgeted restart of Appendix~\ref{app:frozen}, like any restarted process.
\item \emph{Its cost.} No label. It can bet on every query, also on those the
monitor does not label, so its delay is counted in queries, not in labelled
queries.
\item \emph{What the bound assumes.} Only \eqref{eq:flagbudget} and
\eqref{eq:flagnull}. The equality $\E[Z_t\mid\Fpre]=\zeta_{r_t}$ is the
sufficient condition for the second: the questions are drawn from the law the
block was drawn from, independently of the past, and the nodes are as they
were. If the law or the nodes change and the rate of $Z$ rises above
$v_{r_t}$, a crossing is a detection of that change and the bound makes no
claim about it, exactly as with \eqref{eq:onestep} in case~(ii) of
Remark~\ref{rem:three-cases}. If the rate rises and stays at most $v_{r_t}$,
\eqref{eq:flagnull} still holds, the bound still applies, and the flag is not
meant to see the change: cases~(i) and~(iii) of that remark. Dependence is a
third way to leave the null: quiet questions that arrive in runs break
\eqref{eq:flagnull} with the rate unchanged, as dependent misses break
\eqref{eq:onestep}.
\end{enumerate}
\end{remark}
}%
 
\section{FC-RAG instantiation, training and betting}
\label{app:instantiation}

\subsection{Valid communication designs}
\label{app:communication}

At attempt $r$, place the implemented student, retrievers, corpora,
candidate rule, local score, hub aggregation, quantizer and threshold rule
inside $\theta_r$. There are two direct ways to certify communication.
\begin{enumerate}[leftmargin=1.6em,itemsep=1.5pt,topsep=2pt]
\item \emph{Direct calibration.} Calibrate the end-to-end compressed score at
a fixed communication action. Quantization and aggregation are then already
part of the score law, so Proposition~\ref{prop:beta} needs no separate
compression-failure event. This is the self-contained default of the Qwen
studies and the model swap, and it certifies one communication action per
epoch; the resource-policy studies of Appendices~\ref{app:modular}
and~\ref{app:earlier} instead let the allowance vary with bandwidth.
\item \emph{Certified action class.} A separately verified static FC-RAG
certificate supplies a base allowance $c_r$, and a uniform predictable
communication term $\Delta_{\mathrm{comm}}$ supplies
$b_t=c_r+\Delta_{\mathrm{comm}}(a_t)<1$ for every allowed action $a_t$. If
$\gamma_r$ denotes the calibration-certificate allocation and $\rho_r$ the
allocation charged to any stochastic reconstruction-error bound, require
$\gamma_r+\rho_r\le\rcal$.
\end{enumerate}
A third option is a finite simultaneously certified action library
(Corollary~\ref{cor:design}(c)). Changing bandwidth without a uniform
communication term or such a library requires a fresh calibrated epoch.

\paragraph{Why the communication term shrinks as $K^{-1/2}$.}
The communication term $\Delta_{\mathrm{comm}}$ of Section~\ref{sec:setup}
comes from the dither calculation inherited from FC-RAG. Under the subtractively
dithered quantizer of FC-RAG, the hub's aggregate score decomposes as
$s_t(y)=s_t^{\star}(y)+\bar\xi_t(y)$, where $s_t^{\star}$ aggregates the
unquantized node scores and $\bar\xi_t=K^{-1}\sum_{i\le K}\xi_{i,t}$; the
per-node reconstruction noise satisfies $\E[\xi_{i,t}\mid\Fpre,X_t]=0$,
$\E[\xi_{i,t}^{2}\mid\Fpre,X_t]\le v(B_i)=O(2^{-2B_i/\kappa})$ for a
protocol-specific constant $\kappa$, and is conditionally independent across nodes. Independence removes the cross terms, so
\[
 \E[\bar\xi_t^{2}\mid\Fpre,X_t]\ \le\ V_K:=K^{-2}\textstyle\sum_{i\le K}v(B_i).
\]
For the true-answer score and any $\Fpre$-measurable threshold $u$, assume
the conditional law of $s_t^{\star}$ given $(\Fpre,\bar\xi_t)$ has density at
most $f_{\max}$ between $u$ and $u-\bar\xi_t$. Comparing its CDF at these
two endpoints and averaging over the fresh query and the noise gives
\[
 \bigl|\Pbb(s_t\le u\mid\Fpre)-\Pbb(s_t^\star\le u\mid\Fpre)\bigr|
 \le f_{\max}\,\E[|\bar\xi_t|\mid\Fpre]
 \le f_{\max}\sqrt{V_K},
\]
by Cauchy--Schwarz and the conditional variance bound. Thus
$\Delta_{\mathrm{comm}}=f_{\max}\sqrt{V_K}$. With equal budgets $B_i\equiv B$
this is $f_{\max}\sqrt{v(B)/K}$: the $K^{-1/2}$ rate, whereas a single node
gains nothing from averaging. The density clause is what converts a second-moment bound on
the noise into a bound on the \emph{miss probability}; without anti-concentration,
a vanishing bound on the miss probability need not follow. Appendix~\ref{app:eseries} checks the arithmetic numerically over
$K\in\{1,2,4,\dots,128\}$.

Communication is reported either as an abstract per-query cost or in fully
specified logical scalar-message units. The GPT-2 frozen-system studies of
Appendix~\ref{app:earlier} use $\sum_iB_i$, the sum of configured per-node
quantizer widths, as their source-defined abstract cost. With $K_t$ active
nodes, $m$ answer scores and $B$ bits per quantized scalar, the logical score
payload is $K_tmB$ bits per query; it excludes answer identifiers, framing,
retransmission, transport and other protocol overhead unless explicitly
stated. The Qwen configuration, with $K_t=8$, $m=10$ and $B=12$, therefore has
a logical score payload of $960$ bits per query, not measured network traffic.
This is an inference-side charge: a calibration refresh adds its own uplink
$B^{\mathrm{cal}}$ on the rounds where one occurs, which the per-query figure
above does not include.

\subsection{Training, model replacement, and why validity is training-agnostic}
\label{app:training}

Federated Probe-Logit Distillation (FPLD) is a modular training layer: nodes
evaluate public probe contexts, compress and upload probe logits, and the hub
aggregates those messages to distill a shared student, without receiving
private examples, gradients or node model weights~\citep{fcrag2026}. If that
student is selected first and freshly calibrated, the certificate controls its
deployed miss rate under the score-law conditions of
Appendix~\ref{app:beta_proof}. Coverage is therefore agnostic to the
training algorithm: FPLD quality affects accuracy, answer-set size and
resource efficiency, but adds no training-error term to the
validity budget.

Reusing an old threshold after replacing a model is a different problem. It
requires an explicit old-to-new score-transfer theorem under additional
density and training assumptions; the original FC-RAG analysis gives one
possible route, whose strengthened conditional-density clause buys a two-term
$f_{\max}(\varepsilon_{\mathrm{KL}}+\sqrt{2\varepsilon_{\mathrm{KL}}})$ bound on
the transfer term, where $\varepsilon_{\mathrm{KL}}$ is FC-RAG's
high-probability bound on the expected Kullback--Leibler divergence between
the true answer distribution and the trained model. Without
that clause the chain breaks, and Markov truncation leaves only a weaker
$O(f_{\max}\varepsilon_{\mathrm{KL}}^{1/4})$
rate. That would change the scale of an old-threshold transfer bound and
nothing else, since neither our theorem nor any experiment in this paper
relies on such a transfer.

\subsection{Bettors, power, and the evidence scale}
\label{app:power}

The theorem permits any predictable bettor satisfying $0\le\lambda_t\le1/b_t$;
betting choices affect power but not validity. \rev{Unless stated otherwise, the
synthetic and Qwen studies use the
cumulative form of approximate growth-rate-adaptive-to-the-particular-alternative
(aGRAPA) betting~\citep{waudby2024betting}, which is given no post-change rate.
For long deployments the default is a mixture, which keeps detecting after long
benign periods (Appendix~\ref{app:x-restart}) and
is one wealth process because a fixed-weight mixture of admissible wealth
processes is the wealth of a predictable bet: the equal mix of the five-scale
aGRAPA mixture of Appendix~\ref{app:modular} (half-lives $32$, $128$, $512$ and
$2{,}048$ queries, plus the cumulative bettor) and the start-time mixture of
Proposition~\ref{prop:delaybound} (a component every $250$ queries with weight
$1/\{m(m+1)\}$, staking $.0125$, $.025$, $.05$, $.1$, $.2$ and $.4$ of the
largest admissible bet). The aGRAPA components keep the bet cap at half of
$1/b_t$.}

For power intuition, consider the idealized stationary alternative
$M_t\stackrel{\mathrm{iid}}{\sim}\operatorname{Bern}(p)$ with constant
$b<p<1$. The admissible oracle bet $\lambda^\star=(p-b)/\{b(1-b)\}$, which
satisfies $\lambda^\star\le1/b$ for every $p\le1$, makes the one-step wealth
factor equal to the Bernoulli likelihood ratio:
\[
 1+\lambda^\star(M_t-b)
 = \Bigl(\frac pb\Bigr)^{M_t}\Bigl(\frac{1-p}{1-b}\Bigr)^{1-M_t},
\]
so that, with $\E_p$ denoting expectation under $\operatorname{Bern}(p)$ and
$\operatorname{KL}$ the Kullback--Leibler divergence,
\[
 \E_p\log\bigl\{1+\lambda^\star(M_t-b)\bigr\}
 = p\log\frac pb+(1-p)\log\frac{1-p}{1-b}
 = \operatorname{KL}\bigl\{\operatorname{Bern}(p)\,\|\,\operatorname{Bern}(b)\bigr\}.
\]
Thus $\log(1/\edelta)/\operatorname{KL}$, the number of queries after which an
oracle bettor's expected log-wealth reaches $\log(1/\edelta)$, is the natural
evidence scale. This is
an oracle calculation within the admissible cap\rev{, for a stationary alternative
and a process that stands at unit wealth when the change comes. It does not
claim that an estimated bet attains oracle growth, and it is not a delay
characterization for a process that carries pre-change wealth and an estimated
bet across many epochs: Remark~\ref{rem:oracle} makes the pre-change wealth
explicit, and Proposition~\ref{prop:delaybound} gives a delay bound for a
monitor that needs neither the post-change rate nor the time of the change.}

\rev{%
\subsection{A delay bound for a late change}
\label{app:delaybound}

\paragraph{The oracle scale, with the pre-change wealth made explicit.}
\begin{remark}[The oracle scale, with the pre-change wealth made explicit]
\label{rem:oracle}
Suppose that from query $T_0+1$ on the
allowance is a constant $b$ in every epoch and the misses are i.i.d.\
$\Bern(p)$, $p>b$, however often the system is refreshed, and that from then on
the never-reset process uses the oracle bet
$\lambda^\star=(p-b)/\{b(1-b)\}$ of Appendix~\ref{app:power}. Its log-wealth then grows by
$\ell_s=M_s\log\frac pb+(1-M_s)\log\frac{1-p}{1-b}$ per query, with mean
$\KL(p\,\|\,b)$. Let $\nu$ be the number of post-change queries until the
alarm and $w_0=E_{T_0}\in(0,1/\edelta)$ the wealth carried into the change. By
Wald's identity, and because the overshoot of the threshold is at most one
upward step $\log(p/b)$,
\[
 \E[\nu\mid\F_{T_0}]\ \le\
 \frac{\log(1/\edelta)+\log(1/w_0)+\log(p/b)}{\KL(p\,\|\,b)} .
\]
Appendix~\ref{app:power} gives the scale $\log(1/\edelta)/\KL$, which leaves out both
other terms. The term $\log(1/w_0)$ matters most: a process that has lost
wealth before the change pays for it in delay. For $b=.15$, $\edelta=.025$ and
$w_0=1$ the scale of Appendix~\ref{app:power} is $628$ queries at $p=.19$ and $166$ at
$p=.23$; the exact mean delays are $643$ and $173$, the bound above gives $669$
and $185$, and the exact medians are $507$ and $136$ (Table~\ref{tab:never}).
The calculation assumes the oracle bet. An estimated bet may stay near zero for
a long time after a long benign period, which is what the late-onset experiments of
Appendix~\ref{app:x-restart} measure; no bound of this form is claimed for it.
\end{remark}

\paragraph{A never-reset monitor with a late-onset delay bound.}
Appendix~\ref{app:power}
gives the scale of the oracle bet's delay and Remark~\ref{rem:oracle} bounds
it; both need the post-change miss rate. This section gives a bound of the
same form for a monitor that needs neither that rate nor the time of the
change, and that carries whatever happened before the change.

\paragraph{The start-time mixture.}
Fix an integer $L\ge1$, stakes $0<\kappa_1<\dots<\kappa_J<1$, and the weights
$w_m=1/\{m(m+1)\}$, $m\ge1$, which sum to one. Component $(m,j)$ starts with
query $s_m=(m-1)L+1$: it holds wealth $1$ before $s_m$ and from then on stakes
the fixed fraction $\kappa_j$ of the largest admissible bet,
\[
 E^{(m,j)}_t=\prod_{s=s_m}^{t}\Bigl\{1+\frac{\kappa_j}{b_s}\,(M_s-b_s)\Bigr\}
 \quad(t\ge s_m),\qquad E^{(m,j)}_t=1\quad(t<s_m),
\]
and the monitor is
\[
 W_t=\sum_{m\ge1}w_m\,\frac1J\sum_{j=1}^{J}E^{(m,j)}_t .
\]
It is the variant of Proposition~\ref{prop:dominate} that
Remark~\ref{rem:mixture} describes: restarts every $L$ queries with budgets
$\edelta/\{m(m+1)\}$, the equal mix of the $J$ fixed stakes as the bettor of
each restart, and every component kept running instead of being frozen at the
next restart. All components with the same stake multiply by the same factor
at a query, so $W_t$ is computed from one running product per stake and one
accumulated sum; the work per query is $O(J)$ however long the deployment.

For $\kappa\in(0,1)$, $b\in(0,1)$ and $0<\varepsilon\le1-b$ write
\[
 g(\kappa;b,\varepsilon)=(b+\varepsilon)\log\Bigl(1+\kappa\,\frac{1-b}{b}\Bigr)
 +(1-b-\varepsilon)\log(1-\kappa)
\]
for the expected log-growth of the stake $\kappa/b$ when the miss probability
is $b+\varepsilon$.

\begin{proposition}[A never-reset monitor with a late-onset delay bound]
\label{prop:delaybound}
\begin{enumerate}[label=\textup{(\alph*)},leftmargin=2em]
\item $W_0=1$ and $W_t>0$, and $W$ is the wealth of a predictable bet: there
is a bet $\bar\lambda_t$, fixed before query $t$, with
\[
 W_t=W_{t-1}\{1+\bar\lambda_t(M_t-b_t)\},
 \qquad 0\le\bar\lambda_t\le\kappa_J/b_t .
\]
Hence Theorem~\ref{thm:refresh} applies to $W$ with the single threshold $1/\edelta$, under
any refresh policy the theorem allows.
\item Let $\nu\ge0$ be a finite $(\F_t)$-stopping time, $\rho\ge\nu$ a
stopping time that may be infinite, $\varepsilon>0$, and $A\in\F_\nu$ an event
on which, for every $t$ with $\nu<t\le\rho$,
\[
 \E[M_t\mid\Fpre]\ \ge\ b_t+\varepsilon
 \qquad\text{and}\qquad \underline b\le b_t\le\bar b,
\]
with constants $0<\underline b\le\bar b<1-\varepsilon$. Nothing is assumed
about the queries up to $\nu$, about the number of refreshes between $\nu$ and
$\rho$, or about how the allowance moves within $[\underline b,\bar b]$. Let
$m^\star=\min\{m:s_m>\nu\}$, so that $m^\star\le\nu/L+2$, let
$g_j=\inf_{b\in[\underline b,\bar b]}g(\kappa_j;b,\varepsilon)$ and
$c_j=\log\{1+\kappa_j(1-\underline b)/\underline b\}$, and let
$\tau=\inf\{t:W_t\ge1/\edelta\}$. Then for every $j$ with $g_j>0$, on $A$,
\[
 \E[\,(\tau\wedge\rho-\nu)^+\mid\F_\nu]\ \le\ L-1+
 \frac{\log(1/\edelta)+\log\{m^\star(m^\star+1)\}+\log J+c_j}{g_j}.
\]
With $\rho=\infty$ the excess is sustained for ever and the left side is
$\E[(\tau-\nu)^+\mid\F_\nu]$; with $\rho=\tau\vee\nu$ it need only last until the
alarm; a finite $\rho$ is the time at which a refresh repairs the state, and
the bound is then on the time until the alarm or the repair, whichever comes
first.
\end{enumerate}
\end{proposition}

\begin{proof}
(a) Every factor is positive, since $1-\kappa_j>0$, so every component and $W$
are positive, and $W_0=\sum_mw_m=1$. At query $t$ only the components with
$s_m\le t$ change, so
\begin{align*}
 W_t-W_{t-1}&=\sum_{m:\,s_m\le t}w_m\,\frac1J\sum_jE^{(m,j)}_{t-1}\,
 \frac{\kappa_j}{b_t}\,(M_t-b_t)=W_{t-1}\,\bar\lambda_t\,(M_t-b_t),\\
 \bar\lambda_t&=\frac{\sum_{m:\,s_m\le t}w_mJ^{-1}\sum_jE^{(m,j)}_{t-1}\kappa_j}
 {b_t\,W_{t-1}} .
\end{align*}
Here $E^{(m,j)}_{t-1}=1$ for a component that starts at $t$. Every quantity in
$\bar\lambda_t$ is $\Fpre$-measurable. Dropping from $W_{t-1}$ the components
with $s_m>t$, and bounding $\kappa_j\le\kappa_J$, gives
$0\le\bar\lambda_t\le\kappa_J/b_t<1/b_t$.

(b) Work on $A$. Put $\sigma=s_{m^\star}-1$, which is a function of $\nu$,
hence $\F_\nu$-measurable, with $\nu\le\sigma\le\nu+L-1$. Fix $j$ with $g_j>0$
and let $X_i=\log\{1+\kappa_j(M_{\sigma+i}-b_{\sigma+i})/b_{\sigma+i}\}$ and
$S_n=\sum_{i\le n}X_i=\log E^{(m^\star,j)}_{\sigma+n}$. On $\{\sigma=k\}$, which
lies in $\F_k$, condition on $\F^{\mathrm{pre}}_{k+i}$: the step $X_i$ takes
the value $\log\{1+\kappa_j(1-b)/b\}>0$ with probability $p$ and
$\log(1-\kappa_j)<0$ otherwise, where $b=b_{k+i}$ and
$p=\E[M_{k+i}\mid\F^{\mathrm{pre}}_{k+i}]$. If $k+i\le\rho$ then
$p\ge b+\varepsilon$, and since the conditional mean of the step increases
with $p$ it is at least $g(\kappa_j;b,\varepsilon)\ge g_j$. For
$B\in\F_{\sigma+i-1}$ the event $A\cap B\cap\{\sigma=k\}\cap\{k+i\le\rho\}$
lies in $\F_{k+i-1}\subseteq\F^{\mathrm{pre}}_{k+i}$, so summing over $k$ gives
$\E[\one_{A\cap B}\one\{\sigma+i\le\rho\}X_i]\ge g_j\,\Pbb(A\cap B\cap
\{\sigma+i\le\rho\})$. Thus, with $R=(\rho-\sigma)^+$, a stopping time of
$(\F_{\sigma+n})_{n\ge0}$, the stopped process $S_{n\wedge R}-(n\wedge R)g_j$
is a submartingale, with steps bounded above by $c_j$ and below by
$\log(1-\kappa_j)$.

Let $a=\log(1/\edelta)+\log(1/w_{m^\star})+\log J$, which is positive and
$\F_\nu$-measurable, and $N=\inf\{n\ge0:S_n\ge a\}$. Since
$W_t\ge w_{m^\star}J^{-1}E^{(m^\star,j)}_t$, we have $W_{\sigma+N}\ge1/\edelta$
and so $\tau\le\sigma+N$. For each $n$, optional stopping at the bounded time
$N\wedge R\wedge n$ gives
$g_j\,\E[N\wedge R\wedge n\mid\F_\sigma]\le\E[S_{N\wedge R\wedge n}\mid\F_\sigma]
\le a+c_j$, because $S$ is below $a$ before $N$ and one step adds at most
$c_j$. Letting $n\to\infty$, $\E[N\wedge R\mid\F_\sigma]\le(a+c_j)/g_j$.
Finally $(\tau\wedge\rho-\nu)^+\le(\sigma-\nu)+N\wedge R\le L-1+N\wedge R$,
and $\F_\nu\subseteq\F_\sigma$.
\end{proof}

\begin{remark}[Reading the bound]
\label{rem:delaybound}
(i) \emph{The rate.} For a constant allowance the best fixed stake is
$\kappa^\star=\varepsilon/(1-b)$, the oracle bet of Appendix~\ref{app:power} as a
fraction of the largest bet, with $g(\kappa^\star;b,\varepsilon)=\KL(b+\varepsilon\,\|\,b)$.
Since $g(\cdot\,;b,\varepsilon)$ is concave and vanishes at $0$, a grid whose
stakes double contains, for every excess between $\kappa_1(1-b)$ and
$\min(2\kappa_J,1)(1-b)$, a stake that grows at half the Kullback--Leibler rate
or better.

(ii) \emph{The price of a late onset.} The term $\log(1/w_0)$ of
Remark~\ref{rem:oracle}, for the wealth carried into the change, is controlled
by nothing; here it becomes $\log\{m^\star(m^\star+1)\}+\log J\le2\log(\nu/L+3)+\log J$,
whatever the path before $\nu$, because a component that starts after the change
has lost nothing. The bound grows with the logarithm of how late the harm begins.

(iii) \emph{Beside a restart.} A summable restart with the same starts needs
$\log\{\pi^2m^2/(6\edelta)\}$ from one process within its own $L$ queries,
which Proposition~\ref{prop:ceiling} caps; $W$ needs nearly the same,
$\log\{m(m+1)/\edelta\}+\log J$, from a process that is never discarded. That
is why the bound has no counterpart for a restart.

(iv) \emph{Mixing it into another bettor.} If $V$ is any other never-reset
wealth process, such as the five-scale mixture, $(V+W)/2$ is again the wealth of
a predictable bet and reaches $1/\edelta$ no later than $W$ reaches $2/\edelta$,
so part~(b) holds for it with $\log(2/\edelta)$ in place of $\log(1/\edelta)$.

(v) \emph{Reading the conditions.} The event $\{\tau>\nu\}$ lies in $\F_\nu$, so
on $A\cap\{\tau>\nu\}$ the bound is on $\E[\tau\wedge\rho-\nu\mid\F_\nu]$ itself,
and if the alarm came before $\nu$ the left side is zero. A change the operator
cannot observe is covered by enlarging every $\F_t$ and $\Fpre$ with the
information that it has occurred, so that $\nu$ is a stopping time and the bets stay predictable.

(vi) \emph{What is not claimed.} The bound is an upper bound for one monitor;
no lower bound, and no claim that this monitor is the fastest, is made.
\end{remark}

\begin{table}[ht]\rev{%
\caption{The bound of Proposition~\ref{prop:delaybound}(b) with $\rho=\infty$,
evaluated for a constant allowance $b=.175$, $\edelta=.025$, a component every
$L=250$ queries and the six stakes $.0125,.025,.05,.1,.2,.4$ (those of the
experiments). \emph{Nats}:
$\log(1/\edelta)+\log\{m^\star(m^\star+1)\}+\log J$, the level the first
component after the change has to reach, without $c_j$. \emph{Stake}: the
stake of the grid for which the bound is smallest (the bound holds for every
stake with $g_j>0$). \emph{Growth}: its $g_j$. \emph{Bound}: rounded up. \emph{Oracle
scale}: $\log(1/\edelta)/\KL$, the scale of Appendix~\ref{app:power} for a bettor that
knows the post-change rate and starts at unit wealth. The three excesses lie
almost on stakes of the grid ($\kappa^\star=.024$, $.048$, $.097$), so the
growth is close to $\KL$ here; between stakes Remark~\ref{rem:delaybound}(i)
guarantees half of it. Nothing is simulated.}
\label{tab:delaybound}
\begin{center}
\small
\begin{tabular}{@{}rrrrrrrrr@{}}
\toprule
Excess & Benign & $m^\star$ & Nats & Stake & Growth & $\KL$ & Bound & Oracle \\
 & queries & & & & & & & scale \\
\midrule
$+0.02$ & 500 & 3 & 7.97 & 0.025 & 0.00134 & 0.00135 & 6{,}256 & 2{,}741 \\
$+0.02$ & 10{,}000 & 41 & 12.93 & 0.025 & 0.00134 & 0.00135 & 9{,}948 & 2{,}741 \\
$+0.02$ & 50{,}000 & 201 & 16.09 & 0.025 & 0.00134 & 0.00135 & 12{,}298 & 2{,}741 \\
$+0.02$ & 1{,}000{,}000 & 4{,}001 & 22.07 & 0.025 & 0.00134 & 0.00135 & 16{,}743 & 2{,}741 \\
$+0.04$ & 500 & 3 & 7.97 & 0.05 & 0.00524 & 0.00524 & 1{,}810 & 703 \\
$+0.04$ & 10{,}000 & 41 & 12.93 & 0.05 & 0.00524 & 0.00524 & 2{,}758 & 703 \\
$+0.04$ & 50{,}000 & 201 & 16.09 & 0.05 & 0.00524 & 0.00524 & 3{,}361 & 703 \\
$+0.04$ & 1{,}000{,}000 & 4{,}001 & 22.07 & 0.05 & 0.00524 & 0.00524 & 4{,}502 & 703 \\
$+0.08$ & 500 & 3 & 7.97 & 0.1 & 0.02000 & 0.02001 & 667 & 184 \\
$+0.08$ & 10{,}000 & 41 & 12.93 & 0.1 & 0.02000 & 0.02001 & 916 & 184 \\
$+0.08$ & 50{,}000 & 201 & 16.09 & 0.1 & 0.02000 & 0.02001 & 1{,}074 & 184 \\
$+0.08$ & 1{,}000{,}000 & 4{,}001 & 22.07 & 0.1 & 0.02000 & 0.02001 & 1{,}373 & 184 \\
\bottomrule
\end{tabular}
\end{center}
}%
\end{table}
}%

\rev{%
\subsection{Hub rules other than the mean}
\label{app:hubrules}

\begin{remark}[Under direct calibration, any rule that turns the nodes' messages into one score per option is certified as it stands]
\label{rem:hubrules}
Section~\ref{sec:setup} lists the hub's aggregation among the components of a
score state, and under direct calibration (Appendix~\ref{app:communication}, the first of its two
direct ways) Proposition~\ref{prop:beta} is applied to the scores the deployed state
produces, with the rule and its constants fixed from $\Hcal_r$ and the block
scored by the same protocol. Neither Proposition~\ref{prop:beta} nor Theorem~\ref{thm:refresh} refers
to how the hub forms the scores. The appendix's second way, a communication
term $\Delta_{\mathrm{comm}}=f_{\max}\sqrt{V_K}$ for the mean of $K$ nodes,
does not carry over as it stands: under the rules below the noise is that of
the nodes that are averaged, at worst of one. Its third option, a library of
actions certified at once, carries over as direct calibration does. The
second benchmark (Appendix~\ref{app:x-qasc}) uses
three rules besides the paper's mean over all $K$ nodes. Write $s_j(y)$ for
node $j$'s score of option $y$ and $\sigma_j$ for the similarity of node $j$'s
best retrieved passage to the query; $\sigma_j$ is a function of the question,
the node's corpus and its retriever.
\begin{enumerate}[label=\textup{(\roman*)},leftmargin=2em]
\item \emph{Silent nodes.} Node $j$ sends its scores only if
$\sigma_j\ge\tau$, for a gate $\tau$ that is part of the state. The hub takes
the mean of the scores it receives. If it receives none, one designated node
sends its scores computed with no evidence.
\item \emph{The most similar node alone.} Every node sends $\sigma_j$, and the
node with the largest sends its scores, ties going to the node first in a
fixed order.
\item \emph{The $m$ most similar nodes.} Every node sends $\sigma_j$; the $m$
nodes with the largest, ties broken by the same fixed order, send
$1-e^{-s_j(y)}$, and the hub takes the mean.
\end{enumerate}
Under each rule the set $A$ of messages that are averaged is a function of the
question and the state: it is fixed by the similarities, not by the dither.
For values inside the quantizer's range, each message is a node's value plus
a reconstruction error that is uniform on one cell of the quantizer of the
scale that is sent, and the errors are independent across nodes and of the
question and the past, hence of every $\sigma_j$ and of $A$. This is the
subtractively
dithered quantizer of Section~\ref{sec:setup}, whose Appendix~\ref{app:communication} uses of it a zero mean, a
variance bound and independence across nodes; the replay code uses its exact
law. Given the question, the hub's score of an option is therefore
\[
 \frac1{|A|}\sum_{j\in A}x_j\;+\;\frac{\Delta}{|A|}\Bigl(\mathrm{IH}_{|A|}-\frac{|A|}{2}\Bigr),
\]
with $x_j$ the undithered value node $j$ sends, $\Delta$ the quantizer step of
that scale, and $\mathrm{IH}_k$ the sum of $k$ independent standard uniforms:
the replay's exact law for the paper's mean of $K$ nodes, with $|A|$ in place
of $K$. The miss probability of a question at a threshold is then exact,
question by question, with that question's own $|A|$. What each rule shares:
under rule~(i) a node sends no more than under the paper's rule, its silence
tells the hub one bit per node and query, that $\sigma_j<\tau$, and the
fallback needs the hub to ask the designated node; rules~(ii) and~(iii) share
one more scalar per node and query and need the hub to tell the chosen nodes
to send.
\end{remark}
}%
 
An alarm is evidence that the certified conditional-mean allowance has become
untenable; it is not a general distribution-shift alarm. A score distribution can
move substantially while every miss stays unchanged, in which case the evidence should remain unchanged, and empirically does
(Appendix~\ref{app:modular}).

\section{Related work}
\label{app:related}

This is the paper's related-work section; Section~\ref{sec:intro} summarizes the
closest families in one sentence each, and Table~\ref{tab:capability} compares capabilities.
Here Table~\ref{tab:related} instead groups the adjacent literature by
inferential target, including conformal test martingales, sequential risk
monitoring, online and adaptive conformal prediction, conformal risk
control, post-selection calibration, conformal e-values, non-exchangeable
conformal, federated and conformal RAG, and selective acting. The last column
describes each method's role in this paper rather than asserting a capability
ordering.

\begin{table}[ht]
\caption{What each family of methods certifies; the last row is what this paper supplies, and Table~\ref{tab:related} compares the wider literature. Columns: valid at every time at once; prices the chosen bandwidth into the certified bound; calibrates across nodes; certifies states selected from past data; adapts without voiding the certificate.}
\label{tab:capability}
\begin{center}
\footnotesize
\renewcommand{\arraystretch}{1.08}
\setlength{\tabcolsep}{4pt}%
\begin{tabular}{@{}P{0.33\linewidth}ccccc@{}}
\toprule
Approach & \shortstack{Time-\\uniform} & \shortstack{Comm.-aware\\allowance} & \shortstack{Federated\\calibration} & \shortstack{History-sel.\\refresh} & \shortstack{Cert.-pres.\\control} \\
\midrule
FC-RAG$^{a}$ & $\times$ & \checkmark & \checkmark & $\times$ & $\times$ \\
CTM, WATCH$^{b}$ & \checkmark & $\times$ & $\times$ & $\times$ & $\times$ \\
Sequential risk monitoring$^{c}$ & \checkmark & $\times$ & $\times$ & $\times$ & $\times$ \\
\rev{\quad with state-matched certificates and error spending$^{g}$} & \rev{\checkmark} & \rev{$\times$} & \rev{$\times$} & \rev{\checkmark} & \rev{\checkmark} \\
Online / adaptive conformal$^{d}$ & partial$^{*}$ & $\times$ & $\times$ & $\times$ & ---$^{\S}$ \\
Post-selection conformal, CAP$^{e}$ & $\times$ & $\times$ & $\times$ & partial$^{\dagger}$ & $\times$ \\
Federated conformal$^{f}$ & $\times$ & partial$^{\ddagger}$ & \checkmark & $\times$ & $\times$ \\
\textbf{Anytime-FC-RAG (this work)} & \checkmark & \checkmark & \checkmark & \checkmark & \checkmark \\
\bottomrule
\end{tabular}\\[2pt]
{\scriptsize $^{*}$\,Long-run average coverage via an adaptive step, not
per-step time-uniform.\quad
$^{\dagger}$\,Calibrates after an adaptive pick, without budgeting a
trajectory of refreshes.\quad
$^{\ddagger}$\,Compresses calibration summaries, not deployed score
messages.\quad
$^{\S}$\,No certificate to preserve.\\
$^{a}$\,\citet{fcrag2026}.\quad
$^{b}$\,\citet{vovk2021retrain,prinster2025watch}.\quad
$^{c}$\,\citet{podkopaev2022tracking,timans2025monitoring}.\quad
$^{d}$\,\citet{gibbs2021adaptive,angelopoulos2024online}.\quad
$^{e}$\,\citet{bao2025cap}.\quad
$^{f}$\,\citet{lu2023federated,plassier2023conformal,wen2026gcfcp,xu2025fedccp}.\quad
\rev{$^{g}$\,An instance of Theorem~\ref{thm:refresh}, which is what makes it valid
when the model changes; neither cited paper gives that case, and both monitors
are run this way in Appendix~\ref{app:x-restart}.}}
\end{center}
\end{table}

\begin{table}[ht]
\caption{Nearby methods have different inferential targets.}
\label{tab:related}
\begin{center}
\footnotesize
\begin{tabular}{@{}P{0.24\linewidth}P{0.25\linewidth}P{0.41\linewidth}@{}}
\toprule
Method & Primary target & Relation to this work \\
\midrule
Conformal test martingales (CTM) / weighted-conformal monitoring
(WATCH)~\citep{vovk2021retrain,prinster2025watch,shaer2026conditional} &
Score-stream exchangeability, or adaptive and fixed-reference
distribution-shift monitoring & Useful deployment monitors with different
nulls; our certificate interface instead targets excess conditional miss
means across history-selected score states \\
\addlinespace[0.3ex]
Sequential risk monitoring~\citep{podkopaev2022tracking,timans2025monitoring} &
Time-uniform evidence for deployed risk increases or bounded-risk violations &
\rev{The closest target, and the source of the objective of alarming on harmful
risk increases while ignoring benign shift. Both are stated for one fixed
model; \citet{timans2025monitoring} leave model updating to future work. Given
state-matched certificates and error spending, such a monitor is an instance
of Theorem~\ref{thm:refresh}, and for a fixed state we claim no advantage over
it (Appendix~\ref{app:x-restart})} \\
\addlinespace[0.3ex]
Active anytime-valid risk control~\citep{xu2024active} & Anytime-valid
risk-controlled prediction sets with active label acquisition & Handles
adaptive label collection under a label budget; our interface instead composes
fresh state certificates across refreshes \\
\addlinespace[0.3ex]
Non-exchangeable conformal~\citep{barber2023beyond} & Coverage under
non-exchangeability and drift & A complementary coverage target; ours is
trajectory-level certificate accounting across refreshes \\
\addlinespace[0.3ex]
Sample-conditional split conformal~\citep{duchi2025sampleconditional} &
Sample-conditional split-conformal coverage & A fixed-state form of the
statement, holding with high probability over the calibration draw, that
Equations~\eqref{eq:app_budget}--\eqref{eq:app_one_step} require \\
\addlinespace[0.3ex]
Conformal e-values~\citep{gauthier2025evalues} & Conformal e-value
constructions & A complementary conformal-inference target; our monitor tests
certified conditional miss-mean nulls across state changes \\
\addlinespace[0.3ex]
Adaptive/online conformal and
calibration~\citep{gibbs2021adaptive,angelopoulos2024online,huang2026online} &
Long-run coverage or online calibration under shift & Adjusts prediction sets
or forecasting models rather than budgeting fresh state certificates or
testing a fixed conditional-mean null \\
\addlinespace[0.3ex]
Conformal risk control~\citep{angelopoulos2024risk} and risk-controlling
prediction sets~\citep{bates2021rcps} & A general monotone risk, controlled in
expectation and with high probability over the calibration draw respectively &
Complementary risk-control targets, distinct from budgeting every certificate
attempt and running one sticky monitor across refreshes; for a fixed state,
Equations~\eqref{eq:app_budget}--\eqref{eq:app_one_step} require the
high-probability statement that \citet{bates2021rcps} control \\
\addlinespace[0.3ex]
CAP~\citep{bao2025cap} & Post-selection false coverage rate & Directly
motivates calibration honesty after adaptive selection \\
\addlinespace[0.3ex]
Anytime-valid conformal risk control~\citep{hultberg2026anytime} & Risk control from
a cumulatively growing held-out calibration set & A complementary calibration
regime, distinct from separately budgeted fresh-state certificates \\
\addlinespace[0.3ex]
Conformal selective acting~\citep{khosravi2026csa} & Threshold-specific
selective-action risk & A complementary selective-decision problem rather than
federated score-state refresh \\
\addlinespace[0.3ex]
Federated conformal / conformal
RAG~\citep{lu2023federated,plassier2023conformal,wen2026gcfcp,xu2025fedccp,li2024traq,chakraborty2025principled}
& Distributed calibration; RAG answer reliability & Supply federation and RAG
context. The cited conformal-RAG works assume one model, corpus and calibration set;
the cited federated works take the score as given. FC-RAG~\citep{fcrag2026} supplies
the fixed-state communication layer composed here \\
\bottomrule
\end{tabular}
\end{center}
\end{table}

\section{Experimental designs and full results}
\label{app:experiments}

\subsection{Evidence structure and reporting conventions}
\label{app:conventions}

Three kinds of study answer the two questions Section~\ref{sec:evidence} separates: whether each condition of the protocol
is load-bearing, and how the
resulting monitor performs. \emph{Modular synthetic studies} answer the first:
they isolate certificate multiplicity, post-selection calibration, monitoring
power, predictable dynamic allowances, delayed feedback and a resource policy,
because they can violate one condition of the protocol at a time. A
\emph{cached replay} of Qwen2.5-1.5B-Instruct scores answers the second,
evaluating fixed-state monitoring at a $5\%$ total budget, and a GPT-2
\emph{frozen-system study}, run before the Qwen studies, checks the full
system end to end. \rev{One experiment swaps the deployed student, with
fresh state-matched calibration and continued monitoring
(Appendix~\ref{app:liverefresh}). The studies of
Appendices~\ref{app:x-scoreboard}--\ref{app:x-refresh} refresh the student,
the retrieval policy, the hub's rule, the corpora and the bit budget along the
path, on cached scores of real models for two benchmarks. None deploys over a
network: every real-score study is a replay of cached scores, and we say
``cached replay'' for it throughout.} Monte Carlo evidence illustrates the
theorem and its failure modes and does not substitute for its proof.

\rev{\paragraph{Units of replication.}
The two kinds of real-score study differ in what a replication is, and results
should be read accordingly. In the fixed-pool Qwen studies
(Appendices~\ref{app:qwen} and~\ref{app:liverefresh}) the score states,
thresholds, calibration blocks and question pools are fixed, and paths differ
in the order of questions and in the dither; their paths are replays that
share those objects, not independent deployments or calibration experiments,
and their rates are conditional on the fixed pools and blocks. In the
population studies (Appendices~\ref{app:x-cost}--\ref{app:x-refresh} and the
real-score rows of Appendix~\ref{app:x-restart}), questions are drawn i.i.d.\ with replacement
from a fixed population of $9{,}981$ (MMLU-Pro) or $8{,}134$ (QASC) questions,
for calibration blocks and for deployment alike, and every replication draws
its own blocks. Assumption~\ref{ass:epoch} then holds by construction; the
miss probability of every state at every threshold under every law is exact
(the population and the dither's law are known); each certificate's
correctness is observed; and each path is labeled by the \emph{exact excess} of
the miss rate of the state in force over its allowance, so no audit is needed
to call a path harmful and no boundary is fitted by Monte Carlo.

\paragraph{Designs specified in advance.}
Every new experiment of Appendices~\ref{app:x-restart}--\ref{app:x-refresh} was written into a protocol file before it
ran (on QASC, before its run on the evaluation split), with its design and primary outcomes and, for most, the outcome
expected of it; any change of plan was recorded, dated, before the run it concerns. The protocol is in the supplement
(Appendix~\ref{app:repro}), and Appendix~\ref{app:x-failed} summarizes the
expectations against the outcomes.}

Unless a row states otherwise, ``detection'' means unconditional post-change
detection by the stated horizon, so a pre-change false alarm counts against
it. Where power is reported among pre-alarm survivors, the text says so. The
risk-audit confidence statements and the monitor's false-alarm guarantee
belong to separate inferential families and use separate error budgets.

\subsection{What breaks when a refresh condition is dropped}
\label{app:refresh-stress}

The first stress test uses an exact independent-certificate model, with
$50{,}000$ Monte Carlo trials as a numerical check.
Table~\ref{tab:refresh_budget} reports the analytic probability of at least
one certificate failure across $R$ deployment segments. Each attempt used
calibration size $n_r=200$, target miss rate $.10$, and a bet cap equal to
one half of the admissible $1/b_t$ cap; each epoch contained $100$ queries.
With $.05$ spent at each refresh the failure probability approaches one,
whereas both trajectory allocations stay at or below $.05$, and a frozen
certificate stays at $.05$ regardless of how many deployment epochs reuse it.
The $50{,}000$-trial Monte Carlo check agreed with the exact values within
$.002$ everywhere.

\begin{table}[ht]
\caption{Exact any-certificate-failure probability across $R$ deployment
segments at $\caldelta=.05$. The refresh schemes use $R$ certificates:
``unallocated'' spends $.05$ on each, Bonferroni spends $.05/R$, and summable
uses $6(.05)/(\pi^2r^2)$. Frozen reuses one certificate. Bonferroni and
summable additionally admit dependence-robust union bounds of at most $.05$.}
\label{tab:refresh_budget}
\begin{center}
\small
\begin{tabular}{@{}rrrrr@{}}
\toprule
$R$ & Frozen & Unallocated & Bonferroni & Summable \\
\midrule
$1$   & $.0500$ & $.0500$ & $.0500$ & $.0304$ \\
$5$   & $.0500$ & $.2262$ & $.0490$ & $.0440$ \\
$20$  & $.0500$ & $.6415$ & $.0488$ & $.0478$ \\
$100$ & $.0500$ & $.9941$ & $.0488$ & $.0490$ \\
\bottomrule
\end{tabular}
\end{center}
\end{table}

The second stress test generates independent uniform score blocks for each of
$L$ candidate states, selects the smallest empirical threshold, and compares
reuse of that selection block with fresh post-selection calibration. Each
cell uses $50{,}000$ trials; Table~\ref{tab:postselection} reports the sweep.
The exact no-selection marginal miss rate for this design is $.0995$. At
$L=100$ the unrounded means were $.1592$ under reuse and $.0993$ after fresh
calibration. Each reported value averages over calibration trials and
therefore estimates marginal post-selection miss rate, not a guarantee for
every realized calibration sample. These are mechanism stress tests, not
estimates of any deployed model's risk; they isolate the theorem's two load-bearing certificate requirements: budget every attempt, and pick,
then calibrate.

\begin{table}[ht]
\caption{Mean marginal post-selection miss rate over $50{,}000$ trials at
target level $.10$.}
\label{tab:postselection}
\begin{center}
\small
\begin{tabular}{@{}rrr@{}}
\toprule
Candidate states $L$ & Reuse selection block & Fresh after selection \\
\midrule
$1$   & $.100$ & $.100$ \\
$5$   & $.125$ & $.100$ \\
$20$  & $.142$ & $.100$ \\
$100$ & $.159$ & $.099$ \\
\bottomrule
\end{tabular}
\end{center}
\end{table}

\subsection{The composed guarantee under repeated refresh}
\label{app:composed}

Table~\ref{tab:refresh_budget} isolates the certificate side of the guarantee; its design also exercises the composed statement of
Theorem~\ref{thm:refresh} directly. On the first $5{,}000$ of the same trials,
fresh Bernoulli outcomes drove one sticky wealth process across all $R$ epochs
of $100$ queries each, using the epoch-specific allowance supplied by
that epoch's own certificate. The scores are uniform, so the true conditional
miss probability in an epoch is exactly one minus its realized threshold. Here
$\edelta=\caldelta=.05$, so the composed bound is $.10$ and the alarm
threshold is $1/\edelta=20$. Table~\ref{tab:joint_composition} reports how
often the monitor alarms or some certificate on the trajectory
fails; Remark~\ref{rem:target} bounds that union by
$\edelta+\caldelta$ as well.

\begin{table}[ht]
\caption{Rate of alarm or certificate failure over $5{,}000$ null
paths of $100R$ queries, at $\edelta=\caldelta=.05$ and composed bound
$.10$. One sticky wealth process runs across all $R$ epochs with
epoch-specific allowances.}
\label{tab:joint_composition}
\begin{center}
\small
\begin{tabular}{@{}rrrrr@{}}
\toprule
$R$ & Frozen & Unallocated & Bonferroni & Summable \\
\midrule
$1$   & $.0580$ & $.0580$ & $.0580$ & $.0360$ \\
$5$   & $.0526$ & $.2282$ & $.0492$ & $.0438$ \\
$20$  & $.0504$ & $.6554$ & $.0498$ & $.0486$ \\
$100$ & $.0524$ & $.9918$ & $.0490$ & $.0512$ \\
\bottomrule
\end{tabular}
\end{center}
\end{table}

Every allocated scheme stayed between $.0360$ and $.0580$, well inside the
composed $.10$ bound, at every $R$; at $R=100$ the Bonferroni rate was $.0490$
with a pointwise $95\%$ Wilson interval of $[.0434,.0553]$. Under the
unallocated scheme the rate instead reached $.9918$, and the excess was driven by certificate
failures rather than by the monitor, whose null alarm rate was at most $.0262$
anywhere in this sweep. These are Monte Carlo rates over $5{,}000$
paths, with a standard error of about $.001$ for rates near one, which is why the
unallocated $R=100$ entry, $.9918$, sits slightly below the exact
certificate-failure probability $.9941$ of Table~\ref{tab:refresh_budget}
although its event is larger. This is a null-simulation diagnostic of the composed
guarantee under an exact independent-certificate model; it illustrates
Theorem~\ref{thm:refresh} rather than substituting for its proof.

\rev{\paragraph{Certificate failure is not an alarm.}
Table~\ref{tab:joint_composition} reports a union. Table~\ref{tab:composed_split}
(Appendix~\ref{app:x-scoreboard}) separates its two events at $R=100$, from the
same retained paths. Under the unallocated scheme some certificate is wrong on
$.9918$ of paths, the alarm fires on $.002$ of them, and on no path does it
fire with every certificate correct: the near-one rate is a rate of wrong
certificates and should not be read as a false-alarm rate. The $.0262$ quoted
above is the largest alarm rate anywhere in the sweep; it belongs to the frozen
scheme at $R=100$, where one certificate serves $10{,}000$ queries, and $.0238$
of it comes from paths whose certificate is wrong, on which the state is above
its allowance and an alarm is a correct detection. With every certificate
correct the alarm fires on at most $.0024$ of paths in any arm, against
$\edelta=.05$. A wrong certificate of the kind Proposition~\ref{prop:beta}
issues leaves a small, persistent excess, which is detected slowly, not
rarely: in the frozen arm the alarm fired on $.01$, $.07$, $.24$ and $.48$ of
the wrong-certificate paths in runs of $100$, $500$, $2{,}000$ and $10{,}000$
queries.}

\subsection{Testing each guarantee separately}
\label{app:eseries}

A deterministic, analytically known allowance is the $\caldelta=0$ case of
Section~\ref{sec:interface} ($G_r=\Omega$, no calibration-failure event), so the
whole budget is $\edelta$ and each guarantee can be tested on its own. All six
studies below run on CPU over Bernoulli or Gaussian variance-surrogate streams with the
default allowance $b=.155$, the target $\alpha=.10$ plus FC-RAG's additive
slack terms for a calibration size of $200$; none involves a language model.
Table~\ref{tab:eseries} collects them. The envelope check uses the fixed
$15$-point mixture $\eta_j=2^{j/2}$, $j=-8,\dots,6$, with equal weights;
its boundary is the numerical root in Appendix~\ref{app:envelope}.

\begin{table}[ht]
\caption{Each guarantee tested separately at the default $\edelta=.05$ with a
deterministic known allowance ($\caldelta=0$).}
\label{tab:eseries}
\begin{center}
\small
\begin{tabular}{@{}P{0.16\linewidth}P{0.24\linewidth}P{0.52\linewidth}@{}}
\toprule
Study & Target & Result \\
\midrule
Alarm validity & Theorem~\ref{thm:refresh} at the exact boundary and in the
interior & $2{,}000$ paths of length $5{,}000$: alarm rate $.0210$ at
$\E[M_t\mid\Fpre]=b_t$ and $.0025$ strictly inside, against the $\edelta=.05$
budget; median and $95$th-percentile $\sup_tE_t$ are $1.32$ and $10.28$ at the
boundary; a separate interior check at $p=.15<b$ gives $.0105$. \\
\addlinespace[0.3ex]
Detection and delay & Power of the same alarm under a sustained excess &
$500$ paths of length $8{,}000$, change at $2{,}000$: detection
$.310/.996/.998/.998/.998/.998$ at excess $+.02$ through $+.15$, with median
post-change delay $5{,}076\to687$ ($3{,}047$ at $+.04$). No universal delay rate is claimed; the
KL scale of Appendix~\ref{app:power} is the informal guide. \\
\addlinespace[0.3ex]
Envelope coverage & Corollary~\ref{cor:design}(d) & Breach rates $.0025/.0005$ on $2{,}000$ boundary/interior paths each
($6/4{,}000$ overall); the median maximum excursion of $A_t-h_t(\edelta)$
is $-1.68$. \\
\addlinespace[0.3ex]
Communication averaging & The $K^{-1/2}$ shrinkage of
Section~\ref{sec:setup} & Over $K\in\{1,2,4,\dots,128\}$ at $B_i=8$ bits, the
log-log slope of the predicted $\Delta_{\mathrm{comm}}$ against $K$ is $-0.5000$ and the
communication term falls from $3.91\times10^{-3}$ to $3.45\times10^{-4}$, a factor
$\sqrt{128}\approx11.3$. The Gaussian variance-surrogate draws give empirical
mean-absolute-noise slope $-0.5002$, checking the variance model. \\
\addlinespace[0.3ex]
Bettor ablation & Does power need a tuned stake? & $1{,}000$
paths: the adaptive bettor gives null $.009/.002$ and detection
$.340/.992$ on small/large excess, against $.017/.001$ and $.054/.980$ for
$\lambda{=}.323$, $.036/.010$ and $.000/.000$ for $\lambda{=}1.613$, and
$.029/.015$ and $.000/.000$ for $\lambda{=}3.226$. Every fixed stake loses
power and, on small excess, has a higher null rate than the adaptive bettor. \\
\addlinespace[0.3ex]
Hyperparameter sensitivity & Are $\alpha$, $\edelta$ and the bet cap delicate? & Sweeping
$\alpha\in\{.05,.10,.15,.20\}$, $\edelta\in\{.01,.025,.05,.10\}$ and the bet cap
over $\{.25,.50,.75,1.00\}\times$ the admissible $1/b_t$: null rates track
$\edelta$ ($.0020/.0060/.0180/.0520$) and detection stays in $.814$--$.984$,
its minimum at $\alpha=.20$ with half-cap; the full cap at the default $\alpha=.10$ gives $.832$. \\
\bottomrule
\end{tabular}
\end{center}
\end{table}

\subsection{Qwen/MMLU-Pro fixed-state monitoring}
\label{app:qwen}

Table~\ref{tab:protocol_details} specifies the protocol. Of the $12{,}032$ rows in
the pinned MMLU-Pro test split~\citep{wang2024mmlupro}, $9{,}981$ had exactly
ten answer options and were eligible. The complete grid crossed $14$ subjects
with five fixed calibration/pool splits, producing $70$ risk cells and six
paired monitoring paths per cell. The replay assigned $\caldelta=.025$ to the
state certificate and $\edelta=.025$ to evidence, so the alarm threshold was
$1/\edelta=40$ and the per-monitor total was $.05$. Under that allocation the
independent risk audit classified $35$ cells as harmful, $32$ as compatible
and three as ambiguous; the five frozen allowances were $.1514$, $.0754$,
$.1428$, $.1989$ and $.1578$.

\begin{table}[ht]
\caption{Protocol details for the Qwen replay and the matched synthetic
baselines.}
\label{tab:protocol_details}
\begin{center}
\small
\begin{tabular}{@{}P{0.18\linewidth}P{0.74\linewidth}@{}}
\toprule
Component & Fixed specification \\
\midrule
Qwen assets & Pinned Qwen2.5-1.5B-Instruct model, retriever, MMLU-Pro test
split and prompt revisions. Rows without exactly ten options were removed
before splitting. \\
\addlinespace[0.3ex]
Score messages & Under each of eight subject contexts, pinned MiniLM retrieved
the top three of five passages, and the shared student scored all ten choices
by full-option mean token negative log-likelihood. Node scores were clipped to
$[0,50]$, independently subtractive-dithered at $12$ bits and averaged; the
logical payload was $960$ bits/query, not measured traffic. \\
\addlinespace[0.3ex]
Threshold and allowance & In each of five fixed splits, $80$ calibration
examples (ten per context) set the threshold at target miss rate $.10$. An
independent $1{,}000$-draw pre-deployment audit under the same pre-change
fixed-pool query/dither law set the allowance to a one-sided exact
Clopper--Pearson upper bound at certificate-error level $.025$. \\
\addlinespace[0.3ex]
Shift and paths & The exhaustive design crossed $14$ prespecified post-change
subject concentrations with five splits and six paired paths per cell. Each
subject used a fixed $30$-question pool; change occurred after query $500$ and
the horizon was $T=5{,}000$. \\
\addlinespace[0.3ex]
Risk classification & A separate $10{,}000$-draw query/dither audit per
split--subject cell used two-sided exact Clopper--Pearson intervals with a
Bonferroni correction over both tails of all $70$ cells at familywise level
$.05$. A cell was harmful if its lower endpoint exceeded $b$, compatible if
its upper endpoint was at most $b$, and otherwise ambiguous. \\
\addlinespace[0.3ex]
Monitor & The protocol-frozen cumulative aGRAPA monitor, not supplied with any
post-change rate. The evidence allocation $\edelta=.025$ gives alarm threshold
$40$. \\
\addlinespace[0.3ex]
Synthetic baselines & Each comparison used $2{,}000$ paired streams, a change
at query $2{,}000$ and horizon $5{,}000$. A full-history randomized-rank CTM
and independently null-calibrated CUSUM and Shiryaev--Roberts procedures
received the same streams; CUSUM/SR were also given the true post-change
Bernoulli rate. \\
\bottomrule
\end{tabular}
\end{center}
\end{table}

For subjects with a unanimous audit class the separation is clean: six subjects (Computer Science, Health, History, Other, Philosophy and
Psychology) are audited harmful in
\emph{all five} splits, and all six reach detection $1.000$ by $T=5{,}000$,
three of them already by $T=1{,}500$. Table~\ref{tab:qwen_full} gives the
complete fixed grid. The detection columns
average all $30$ paths for a subject and therefore mix its harmful and
compatible cells; the relevant aggregation is by audit class, and it
gives $210/210$ harmful cell--path pairs
detected by $T=5{,}000$ ($.857$ at $T=1{,}500$, $.957$ at $T=3{,}000$), zero
alarms among the $192$ compatible cell--path pairs, and zero among the $30$
distinct no-drift paths, at every horizon. Ambiguous cells count as
neither power nor specificity evidence; none of their $18$ paths alarmed.

\begin{table}[ht]
\caption{Complete Qwen/MMLU-Pro grid at $\caldelta=\edelta=.025$. ``Risk'' is
the mean fixed-pool audited risk over the five splits; ``H/C/A'' counts the
splits in which the independent audit called the cell harmful, compatible or
ambiguous; detection averages all $30$ paths for that subject.}
\label{tab:qwen_full}
\begin{center}
\small
\begin{tabular}{@{}lccccc@{}}
\toprule
Post-change subject & Risk & H/C/A & $T=1{,}500$ & $T=3{,}000$ & $T=5{,}000$ \\
\midrule
Biology          & $.119$ & $2/3/0$ & $.067$ & $.267$ & $.400$ \\
Business         & $.067$ & $0/5/0$ & $.000$ & $.000$ & $.000$ \\
Chemistry        & $.086$ & $0/4/1$ & $.000$ & $.000$ & $.000$ \\
Computer Science & $.236$ & $5/0/0$ & $.967$ & $1.000$ & $1.000$ \\
Economics        & $.119$ & $1/3/1$ & $.100$ & $.200$ & $.200$ \\
Engineering      & $.052$ & $1/4/0$ & $.200$ & $.200$ & $.200$ \\
Health           & $.280$ & $5/0/0$ & $1.000$ & $1.000$ & $1.000$ \\
History          & $.233$ & $5/0/0$ & $.667$ & $.900$ & $1.000$ \\
Law              & $.067$ & $0/5/0$ & $.000$ & $.000$ & $.000$ \\
Math             & $.061$ & $0/4/1$ & $.000$ & $.000$ & $.000$ \\
Other            & $.323$ & $5/0/0$ & $.967$ & $1.000$ & $1.000$ \\
Philosophy       & $.484$ & $5/0/0$ & $1.000$ & $1.000$ & $1.000$ \\
Physics          & $.102$ & $1/4/0$ & $.033$ & $.133$ & $.200$ \\
Psychology       & $.319$ & $5/0/0$ & $1.000$ & $1.000$ & $1.000$ \\
\bottomrule
\end{tabular}
\end{center}
\end{table}

\paragraph{Discovery-disjoint confirmation.}
Before any confirmation question was scored, a selection rule that uses audited
risks only was frozen: among the subjects harmful in all five discovery splits
(the grid of Table~\ref{tab:qwen_full}), select the one whose smallest
simultaneous lower-bound margin above its split-specific allowance is largest. This selected Philosophy uniquely, at a minimum discovery margin of $.279$
against $.079$ for the next-best subject, without using detection outcomes. Of $350$ eligible Philosophy questions, the discovery
pools had used $127$ distinct rows; a fixed label-free allocation seed placed
$150$ of the $223$ untouched rows into five mutually disjoint $30$-question
pools with zero discovery overlap. Separately seeded $10{,}000$-draw
query/dither audits gave conditional risks from $.231$ to $.469$ (mean
$.400$), and with two-sided exact Clopper--Pearson tails Bonferroni-allocated
over all five cells at familywise level $.05$, every simultaneous lower bound
exceeded its frozen allowance, the smallest margin being $.144$. As a
secondary operating characteristic, all $30$ newly seeded shifted paths
alarmed by $T=1{,}500$ with median and $90$th-percentile delays of $149$ and
$200$ queries, while $0/30$ paired no-drift paths alarmed through
$T=5{,}000$. The five score states, thresholds, calibration
blocks and pre-change pools remained frozen from discovery, and the allowances
are the five listed above, at $\caldelta=.025$. The
$150$-question evidence and the seeded query/dither draws are fresh.

\rev{\paragraph{Units of replication, and what the counts measure.}
The $210$ and $222$ paths are not independent deployments or calibration
experiments. They are $70$ cells ($14$ subjects by five splits) replayed on
six seeded paths each, and $30$ no-drift paths. The six paths of a cell share
its score state, threshold, calibration block, allowance and $30$-question
pool, and within a split and seed all $15$ scenarios are identical until query
$500$, so the grid has $30$ distinct pre-change streams. The five allowances
are per split ($.1514$, $.0754$, $.1428$, $.1989$, $.1578$), so a cell's
excess depends on its split as much as on its subject: two cells are audited
harmful with a risk below the target $.10$, because their split's allowance is
$.0754$, and eight compatible cells have a risk above $.10$. What divides the
alarmed cells from the silent ones is their audited excess over the allowance:
no cell lies between $+.0055$ (the largest excess with no detection) and
$+.0228$ (the smallest with detection), so the two counts are the two ends of a
dose--response. Appendix~\ref{app:x-first} fills in that range, with questions
drawn from the population and a fresh calibration block in every replication:
detection then rises smoothly with the exact excess, low just above the
allowance and near one beyond four points. The deployed scorer is weak (top-1
accuracy $.132$ on these pools, against a chance level of $.10$), which is why
its sets are large; Appendix~\ref{app:x-qasc} adds a benchmark on which
retrieval helps and the sets are small.}

\paragraph{Estimand and scope.}
At $T=5{,}000$ the no-drift answer sets had mean size $8.63$ of ten options
across the five splits. The sets are therefore inefficient, and these
results are monitoring evidence rather than a competitive question-answering
result. They are conditional on the fixed pools and thresholds, and the $10{,}000$
draws, taken uniformly with replacement from each fixed $30$-question pool,
estimate that pool's query/dither risk, not population risk over MMLU-Pro or
risk over $10{,}000$ distinct questions. The run is a static cache-replay
monitoring stress test rather than measured network federation, and the grid
is a deterministic replay of immutable cached scores, pools, seeds and paths.

\subsection{One state-matched model swap}
\label{app:liverefresh}

\rev{Among the fixed-pool studies this is the one that changes the deployed model.}
On the same Qwen2.5/MMLU-Pro stack at $K=2$, the deployed student is
replaced mid-trajectory with state-matched calibration, and one wealth
process runs across the swap without being reset.
Table~\ref{tab:liverefresh} summarizes the result.

\rev{\paragraph{What this experiment shows, and its exact-rate rerun.}
It is one special case, on one pair of calibration blocks: the incoming model
and its block are fixed before the replay, and the evidence chooses only the
time of the swap. Its $200$ paths per arm share the two blocks $D_1$ and $D_2$,
so they repeat the query order and the dither and not the calibration
randomness; its calibration blocks and deployment pools are disjoint and
stratified, and its boundary arms are fitted by Monte Carlo. It shows that a
threshold must be matched to the incoming model.
Appendix~\ref{app:x-first} reruns it with calibration and deployment questions
drawn i.i.d.\ from one population, every miss rate exact and both blocks
redrawn on every path, which tests the certified transition itself: the
default monitor alarms with both certificates correct on at most $.0140$ of a
cell's paths (budget $.025$) and detects the post-swap drift on $.88$ to $.89$
of paths (aGRAPA on $.61$ to $.66$; the $.925$ and $.955$ of
Table~\ref{tab:liverefresh} belonged to this run's two blocks), and the
downgrade's miss rate averages $.099$ under the legal order against $.164$
under calibrate-then-replace, of which this run's $.089$ against $.154$ is one
draw. The contrast is exact here too: replaying this run with every one of its
$2{,}400$ paths retained gives a paired difference in miss rate of $+.0650$
($95\%$ interval $[+.0641,+.0658]$, conditional on the two blocks), with the
illegal ordering worse on $200$ of $200$ paths.}

\paragraph{Protocol.}
Two retrieval nodes carry disjoint corpora (Biology, History), and
Philosophy is held out as the drift target; the audit of
Appendix~\ref{app:qwen} classified Philosophy as harmful in all five splits of that design, so the drift target was fixed by that prior
evidence rather than by anything measured here. Each of the two calibration
subjects is cut into three mutually disjoint parts of $50/50/100$ items, so the
calibration blocks $D_1$ and $D_2$ hold $100$ items each and the served pool holds $100$ per subject: $D_1$ calibrates the first epoch, $D_2$ calibrates the
second, and neither is ever served.
Item identity, used for these disjointness checks, is the pair (question stem, option tuple), because MMLU-Pro reuses
stems across rows with different option sets. Epoch one deploys
Qwen2.5-0.5B-Instruct behind a threshold and the
Proposition~\ref{prop:beta} Beta allowance computed from $D_1$. The incoming
Qwen2.5-1.5B-Instruct is preselected and its threshold and allowance are
computed from $D_2$ alone before replay, then activated at the swap. The \emph{downgrade} arms run the
same trajectory with the two students exchanged. Budgets are
$\edelta=\caldelta=.025$, so each of the two attempts is allocated
$\gamma_r=.0125$ and the specified total budget is $.05$; $\alpha=.10$ gives an
allowance of $u_r=.175$ at $n_r=100$. The horizon is $T=4{,}000$ queries, with $200$
paths per arm.

\paragraph{Evidence selects the refresh time under drift.}
In the drift arms, the controller swaps on the query after wealth reaches
$10\%$ of the alarm threshold, or at query $2{,}001$, whichever comes first.
The incoming state and its disjoint calibration block are preselected;
the swap time uses only past evidence. Wealth triggered the swap on $186$ of
$200$ paths in each legal drift arm, with median swap queries of
$1{,}145$ and $1{,}222.5$ in the two arms. No-drift and boundary-targeted arms swap at query $2{,}001$.

\paragraph{Why a boundary-targeted arm is needed.}
The no-drift arm sits empirically far inside the allowance: a realized miss
rate near $.09$ against an allowance of $.175$ leaves the predictable
bettor staking on about one step in a thousand, so its wealth barely moves and a
zero alarm count would reflect near-zero stakes, not validity. The
\emph{boundary} rows therefore target the allowance. Before deployment,
$4{,}000$ query/dither draws estimate the miss rate in each of the calibration-subject
and drift pools. The served stream then mixes the two pools with a plug-in weight chosen
so that its estimated miss rate equals the epoch's allowance, clamped to $[0,1]$. Monte Carlo fitting does not establish exact equality or
certify the null inequality, so these rows are boundary-targeted stress tests.
In the legal upgrade, the mean pathwise maximum of $\log E_t$ is $0.59$ in
the boundary-targeted arm, against $0.09$ in the no-drift arm and an alarm
threshold of $\log 40=3.689$.

\begin{table}[ht]
\caption{One state-matched model swap at $K=2$ on Qwen2.5/MMLU-Pro ($0.5$B
and $1.5$B Instruct), $200$ paths per row, allowance $u_r=.175$ ($\alpha=.10$,
$n_r=100$), budget $.05$; Miss is the mean before $\to$ after the swap. \emph{Legal} is
select--calibrate--certify--deploy; \emph{illegal} is calibrate-then-replace,
scoring $D_2$ with the outgoing student and deploying the incoming one behind
it. \emph{Boundary} targets the allowance by Monte Carlo mixing;
\emph{drift} lets the held-out subject take over after query $500$ and
the evidence trigger the refresh.}
\label{tab:liverefresh}
\begin{center}
\footnotesize
\setlength{\tabcolsep}{5pt}
\begin{tabular}{@{}llrcrrcr@{}}
\toprule
Swap & Deployment & \multicolumn{3}{c}{Legal ordering} & \multicolumn{3}{c}{Illegal ordering} \\
\cmidrule(lr){3-5}\cmidrule(lr){6-8}
 & & Alarm & $95\%$ CI & Miss & Alarm & $95\%$ CI & Miss \\
\midrule
$0.5\mathrm{B}\!\to\!1.5\mathrm{B}$ & no drift   & $.000$ & $[.000,.019]$ & $.119\!\to\!.075$ & $.000$ & $[.000,.019]$ & $.119\!\to\!.054$ \\
$1.5\mathrm{B}\!\to\!0.5\mathrm{B}$ & no drift   & $.000$ & $[.000,.019]$ & $.089\!\to\!.089$ & $.000$ & $[.000,.019]$ & $.089\!\to\!.154$ \\
\addlinespace[0.3ex]
$0.5\mathrm{B}\!\to\!1.5\mathrm{B}$ & boundary   & $.025$ & $[.011,.057]$ & $.176\!\to\!.175$ & $.000$ & $[.000,.019]$ & $.176\!\to\!.139^{\ddagger}$ \\
$1.5\mathrm{B}\!\to\!0.5\mathrm{B}$ & boundary   & $.015$ & $[.005,.043]$ & $.176\!\to\!.175$ & $.010$ & $[.003,.036]$ & $.176\!\to\!.175$ \\
\addlinespace[0.3ex]
$0.5\mathrm{B}\!\to\!1.5\mathrm{B}$ & drift      & $.925$ & $[.880,.954]$ & $.207\!\to\!.210$ & $.000$ & $[.000,.019]$ & $.207\!\to\!.140^{\ddagger}$ \\
$1.5\mathrm{B}\!\to\!0.5\mathrm{B}$ & drift      & $.955$ & $[.917,.976]$ & $.203\!\to\!.220$ & $1.000$ & $[.981,1.000]$ & $.203\!\to\!.303$ \\
\bottomrule
\end{tabular}\\[2pt]
{\scriptsize Intervals
are $95\%$ Wilson intervals for the alarm rate over the $200$ paths of
that arm. The miss columns are arm means over those same paths. The
retained summaries keep the arm mean rather than the per-path miss rate,
so no dispersion is quoted for them; the text of this appendix instead
reports the two epoch-two arm-mean thresholds.
$^{\ddagger}$\,A stale threshold from the \emph{weaker} student is too loose for
the better one, so these cells stay below their allowance and alarm at
$.000$: this conservatism is uncontrolled and is no evidence of safety.}
\end{center}
\end{table}

\paragraph{What the rows say.}
Table~\ref{tab:liverefresh} reports the six rows. \emph{Monitoring continues across a real
refresh}: boundary-targeted alarm rates are $.025$ and $.015$ with the bettor
active\rev{, and in the exact-rate rerun the alarm with both certificates
correct stays at or below $.0140$ (Appendix~\ref{app:x-first})}. \emph{State-matched calibration held the downgrade's miss rate steady
across the swap}: on the same $200$ paired paths,
differing in which student scores $D_2$, the legal ordering holds a
downgrade's miss rate at $.089$ before and after the swap, while
calibrate-then-replace raises it to $.154$\rev{, one draw from the means $.099$ and $.164$}.

\rev{The original summaries kept the arm mean and not the per-path miss rate;
the replay with every path retained supplies the interval (first paragraph of
this section).} The retained arm-mean
thresholds show
the two students' different score scales. Across the four no-drift arms their epoch-two means
take two values, according to which student scored $D_2$:
$5.582$ when the $0.5$B student did, $4.921$ when the $1.5$B student did. The
legal downgrade therefore deploys the $0.5$B student behind a state-matched threshold averaging $5.582$ and
holds the miss rate at $.089$; calibrate-then-replace deploys the same student
behind an outgoing-model threshold averaging $4.921$, tighter by $0.661$, and the miss rate rises
to $.154$. The same two values, in the opposite order, explain the
$^{\ddagger}$ cells: a threshold of $5.582$ inherited from the weaker
student is too loose for the better one, so the estimated pool mixtures do not
reach the allowance. The allowance itself is $u_r=.1753$ in every arm, so none of this
is a change in the allowance, only in which student supplied the
calibration scores.

\emph{The refresh did
not hide the distribution shift}: after the refresh under drift (triggered by
the evidence on $186$ of $200$ paths), the refreshed state's empirical miss rate still exceeded its allowance on $200/200$
paths in both arms, and the monitor still fired on $.925$ and $.955$ of
paths; in the downgrade arm the illegal ordering shortened the
median post-change delay from $1{,}086$ to $763.5$ queries
while driving the mean pathwise maximum $\log E_t$ from $16.8$ to $97.5$.

\paragraph{A second model pair, and what it does not reach.}
The same runner on an MMLU stack with DistilGPT-2 and GPT-2 ($K=2$, blocks of
$80$, allowance $.185$) yields a legal-downgrade post-swap miss rate of $.176$
in the boundary-targeted arm; $198/200$ mixture estimates are clamped below
the allowance. Its sole alarm occurs at query $6$, before the swap. The illegal
arm reaches $.184$; neither ordering establishes an exact boundary null.
The ordering effect is present but small ($.131\to.131$ legal against
$.131\to.137$ illegal, and $27$ against $130$ allowance breaches under drift),
plausibly because the two students' score scales are much closer than the two
Qwen students'. Its held-out drift subject is not independently audited, so its drift
arms say nothing about detection and are not reported as a null.

\paragraph{Scope.}
This executes a refresh of the deployed \emph{student}; the retriever, corpora
and communication design are unchanged across the swap, and nothing runs on a
network. Model forward passes are cached per (model, revision, node, corpus,
item), so the $200$-path sweep replays them rather than recomputing them,
exactly as the Qwen studies of Appendix~\ref{app:qwen} do. One trajectory per
configuration is additionally executed with the cache bypassed, scoring every
query as it arrives; for the Qwen swap, that run scored $8{,}000$ option
sequences in $800$ language-model forward calls, including calibration, in
$297$ seconds, with no alarm over its $200$-query horizon. Deployment draws with replacement from fixed $100$-item pools,
so these are conditional statements about those pools, as in
Appendix~\ref{app:qwen}. Disjoint, stratified calibration blocks and separate
fixed deployment pools do not by themselves establish the premise of
Assumption~\ref{ass:epoch} that deployed scores are conditionally i.i.d.\ with
the calibration law; the Beta allowances
here implement the certificate rule, while the model experiment tests
state matching and continuous monitoring across a swap. Detection here is below that at the stronger alternatives of
Appendix~\ref{app:baselines}, plausibly because the excess over the allowance is smaller;
this study is evidence about \emph{refresh}, not about power.

\subsection{Matched sequential baselines}
\label{app:baselines}

The matched synthetic comparison used $2{,}000$ paired continuous-score
streams, a change at query $2{,}000$ and horizon $5{,}000$. Scores were
$\operatorname{Unif}(0,1)$ before the change, the threshold was $.9$, and
post-change scores followed $\operatorname{Beta}(a,1)$ for
$a\in\{1.25,1.5,2,3,5\}$; threshold crossings formed the binary misses. Here
the null allowance was known exactly, so $\caldelta=0$ and the full monitor
budget was $\edelta=.05$. CTM saw the continuous scores and mixed power
martingales at exponents $.1,.3,.5,.7,.9$; it therefore tests exchangeability.
CUSUM~\citep{page1954cusum} and Shiryaev--Roberts~\citep{roberts1966control}
received the true post-change Bernoulli rate and used thresholds calibrated on
$5{,}000$ independent null streams at the same horizon and target level $.05$.
Table~\ref{tab:baselines} reports the resulting detection rates.

\begin{table}[ht]
\caption{Unconditional post-change detection on matched streams.
$^\star$ denotes a procedure given the true post-change Bernoulli rate and an
independently null-calibrated threshold at target $.05$.}
\label{tab:baselines}
\begin{center}
\small
\begin{tabular}{@{}rrrrr@{}}
\toprule
$p_{\mathrm{post}}$ & Ours & CTM & CUSUM$^\star$ & SR$^\star$ \\
\midrule
$.123$ & $.603$ & $.343$ & $.917$ & $.936$ \\
$.146$ & $.987$ & $.970$ & $.982$ & $.980$ \\
$.190$ & $.988$ & $.975$ & $.982$ & $.983$ \\
$.271$ & $.988$ & $.975$ & $.979$ & $.978$ \\
$.410$ & $.988$ & $.975$ & $.987$ & $.985$ \\
\bottomrule
\end{tabular}
\end{center}
\end{table}

On a separate held-out null our monitor alarmed at $.0175$ (pointwise $95\%$
Wilson interval $[.0126,.0242]$) and CTM at $.0270$; the oracle CUSUM/SR null
rates ranged from $.038$ to $.055$ across their five separately calibrated
thresholds. Oracle knowledge materially improved the near-boundary case and
bought speed elsewhere: median post-change delays were $24$--$442$ for the
oracle procedures against $139$--$1{,}101$ for ours over the four stronger
alternatives. This is an operational comparison at each method's own stated
target, not an identical-null ranking; CTM's null rate pertains to its
exchangeability target\rev{, and its alarms under score changes that preserve
every miss are correct behavior under that null. Two things this table does
not show are in Appendix~\ref{app:x-baselines}: monitors of the \emph{same}
risk null on these streams (the default monitor, the rate of
\citet{timans2025monitoring} and a running-risk test in the style of
\citet{podkopaev2022tracking}), and delays beside detection. The oracle
detectors' advantage deserves the emphasis of the first row: near the
allowance they detect on $.917$ and $.936$ of streams where aGRAPA detects on
$.603$ and the default monitor on $.758$. That
their finite-horizon thresholds exceed the budget when retained at a longer
horizon illustrates finite-horizon calibration; it is no advantage of ours
over a time-uniform risk monitor, which has no horizon either}.

\subsection{Online conformal: a different question on the same scores}
\label{app:oc}

Online conformal prediction is a further comparison class, and it is not a
monitor: it adapts the threshold to hold long-run coverage rather than testing
whether coverage has broken. Running it on the \emph{identical} Qwen per-query
scores makes the division of labor concrete. Under the no-drift sampler its
post-change miss rate is $.101$ against a $.10$ target, with mean answer sets
of $8.84$ of ten options, against $.126$ and $8.63$ for our fixed threshold. Under
the audited-harmful Philosophy concentration it reports $.107$, within $.006$ of its no-drift value, by
re-adjusting its threshold to
sets of $8.89$, while the fixed threshold that our certificate covers admits
only $5.65$ options and misses $.486$ of the time. On Computer Science the
pattern repeats at $.102$ with sets of $8.54$, against our $.236$ and $7.71$; on the
compatible Business shift both stay below $.10$, at $.097$ and $.067$, and our
alarm stays silent.

\begin{table}[ht]
\caption{Online conformal on identical Qwen scores re-adjusts its threshold to
hold its $.10$ target through a harmful shift, keeping sets of $8.54$--$8.89$ of
ten options where the fixed threshold's sets shrink to $5.65$--$7.71$; its miss rate
barely moves, while the certified fixed state's rises and our alarm fires.
$^{\dagger}$\,Audited harmful.}
\label{tab:oc}
\begin{center}
\scriptsize
\setlength{\tabcolsep}{3.5pt}
\begin{tabular}{@{}lccc@{}}
\toprule
Shift & \shortstack{Online\\conformal} & \shortstack{Certified\\state}
 & \shortstack{Our\\alarm} \\
\midrule
No drift                  & $.101$ & $.126$          & silent \\
Business                  & $.097$ & $.067$          & silent \\
Comp.\ Sci.$^{\dagger}$   & $.102$ & $.236$          & fires \\
Philosophy$^{\dagger}$    & $.107$ & $\mathbf{.486}$ & fires \\
\bottomrule
\end{tabular}
\end{center}
\end{table}

This is online conformal working as designed: it holds coverage. But its
coverage output is close to invariant across a shift
that roughly quadruples the fixed-state miss rate, so an operator watching only
that output sees little sign that anything has changed. The alarm answers the complementary
question and detects every path on both harmful subjects. The two consume one set
of scores and can run together at no extra retrieval cost.

The GPT-2-small/MiniLM system of Appendix~\ref{app:earlier} showed the same
division of labor at a coarser resolution in a separately seeded head-to-head run: alarm
rate $.27/1.00/.00$ on MMLU, DBpedia and AG News against online-conformal
post-change miss rates $.118/.121/.108$, with our monitor firing on every
DBpedia path of that run.

\subsection{Additional modular stress tests}
\label{app:modular}

Table~\ref{tab:power} condenses the power and operating-cost results behind the takeaways of
Section~\ref{sec:evidence}; Tables~\ref{tab:additional}
and~\ref{tab:operational_additional} give the statistical and operational
modular studies in full. All conclusions are scoped to the listed designs.

\begin{table}[ht]
\caption{Power and operating cost at each study's stated budget, condensed from
Tables~\ref{tab:eseries}, \ref{tab:additional} and~\ref{tab:operational_additional};
each value from one completed study. \emph{Excess} is the post-change miss rate minus
the allowance; \emph{survivor power} is power among the paths that had not
already alarmed; timing schedules fix the allowance before each query; the
five-scale mixture is a precommitted mixture of bettors with different
half-lives; sentinels bet only on scheduled high-bandwidth queries.}
\label{tab:power}
\begin{center}
\small
\setlength{\tabcolsep}{0pt}
\begin{tabular}{@{}p{\linewidth}@{}}
\toprule
\textbf{Detection, delay, late labels.} At excess ${+}.02/{+}.04/{+}.15$: detection
$.310/.996/.998$, median delay $5{,}076/3{,}047/687$; survivor power
$.945$--$.955$ under seven delay laws at $+.04$ and $.855$--$.922$ under six
legal timing schedules at $+.02$. \\
\addlinespace[1.0ex]
\textbf{Betting and sensitivity.} The adaptive bettor reaches $.340/.992$ on
small/large excess at null $.009/.002$, against $.054/.980$ for the best fixed stake;
sweeping the target miss rate $\alpha$, $\edelta$ and the bet cap keeps the null rate tracking $\edelta$ and detection
in $.814$--$.984$; a five-scale mixture lifts mean survivor power $.925\to.958$ and worst-cell
survivor power $.627\to.789$ at null $.035$ (paired unconditional gain $+.154$ at the near-boundary alternative), reaching $1.000$ in all five
cells at $T=20{,}000$ on the same horizon-free threshold. \\
\addlinespace[1.0ex]
\textbf{Evidence-triggered control.} Escalating at $10\%$ of the threshold adds
$+.125$ detection over fixed-low at $p_{\mathrm{post}}=.23$ ($95\%$ interval $[.064,.185]$) at
$69\%$ less cost than fixed-high, though at $p_{\mathrm{post}}=.28$ fixed-low already detects every path and is cheapest; three
sentinel controllers (periods $5,10,20$) add $+.112/.127/.123$ over
budget-matched randomized policies at $p_{\mathrm{post}}=.17$. \\
\addlinespace[1.0ex]
\textbf{Node removal and operating cost.} At a $.10$ total budget, with any one retrieval node removed: $0/40$
audited cells harmful, $0/240$ paths alarming, harmful Computer Science and
History shifts still caught; a $250\to20{,}000$-draw audit cuts allowance
overhead $.0397\to.0039$ ($10.2\times$); the monitor costs $24.0\,\mu$s per update. \\
\bottomrule
\end{tabular}
\end{center}
\end{table}

\begin{table}[ht]
\caption{Statistical modular studies.}
\label{tab:additional}
\begin{center}
\small
\begin{tabular}{@{}P{0.18\linewidth}P{0.74\linewidth}@{}}
\toprule
Question & Result \\
\midrule
Multiscale betting & A precommitted five-component analytic mixture
(half-lives $32$, $128$, $512$, $2{,}048$, plus cumulative; equal fixed wealth,
no component, rate or horizon selection) had null alarm rate $.035$, mean
power $.958$ and worst-cell power $.789$ among pre-alarm survivors on the
fixed grid, against $.027/.925/.627$ for the cumulative bettor and
$.020/.873/.368$ for CTM. At the near-boundary cell the gain is statistically significant: paired against
the cumulative bettor at $p_{\mathrm{post}}=.1234$, the unconditional improvement is $+.154$ ($95\%$
Hoeffding interval $[.068,.240]$; $168$ multiscale-only against $14$
cumulative-only detections). The analytic threshold is horizon-free, and at
$T=20{,}000$ the mixture reached power $1.000$ in all five cells. \\
\addlinespace[0.3ex]
Horizon transfer & In a nested-path study the unchanged threshold of $20$ gave
alarm rates $.030/.036$ for our monitor and $.015/.015$ for CTM at
$T=5{,}000$ and $T=20{,}000$. In this nested-path sample CUSUM/SR were already at $.050$--$.077$ at their
own design horizon, and once those $T=5{,}000$ thresholds were frozen they
alarmed at $.193$--$.274$ by $T=20{,}000$; horizon-specific recalibration
brought them to $.039$--$.067$ but required $2{,}500$ null paths for each of five
design rates ($12{,}500$ maximum-horizon paths, $250$ million generated null
observations). At the near-boundary alternative $p_{\mathrm{post}}=.1234$,
extending the horizon raised our survivor power from $.616$ to
$1.000$ without being given the post-change rate. \\
\addlinespace[0.3ex]
Target specificity & Three monotone post-change score transforms had
Kolmogorov--Smirnov distances $.095/.225/.425$ and Wasserstein-1 distances
$.058/.135/.243$ over six million post-change scores each, while preserving
every miss exactly. Our alarm times were bit-for-bit identical to the
unshifted reference; the post-change alarm rate was $.003$ in each condition and
the whole-horizon null rate remained $.026$, whereas CTM detected the score
changes at $.142/.964/.969$. Harmful positive controls reverse the ordering:
among pre-alarm survivors our monitor detected $.634$ at
$p_{\mathrm{post}}=.123$ against CTM's $.374$, and $1.000$ at every stronger
alternative. \\
\addlinespace[0.3ex]
Predictable allowances & Across $2{,}000$ paths of length $10{,}000$,
six legal pre-outcome schedules (fixed-low, fixed-high, $50$-query alternating
blocks, independent pre-query randomization, a lagged previous-miss rule and
next-query evidence-triggered escalation) had exact-null whole-horizon alarm rates $.023$--$.029$, with
Bonferroni--Clopper--Pearson simultaneous upper endpoints at most $.0404$. An
intentionally illegal same-query rule, which selects the allowance after
observing $M_t$, alarmed at $.150$ (Wilson $[.135,.166]$). At conditional
excess $.02$, power among pre-alarm survivors was $.855$--$.922$; it was
$1.000$ at excess $.04$ and $.08$. \\
\addlinespace[0.3ex]
Exogenous delays & Seven delay laws (immediate; fixed $1$, $5$, $20$, $100$;
$\operatorname{Unif}\{0,\dots,100\}$; capped geometric with mean $25$) kept
null alarm rates at $.019$--$.022$. At the query horizon, power among paths
surviving to the first post-change feedback arrival was $.945$--$.955$ at
excess $.04$ and $1.000$ at excess $.12$. A fixed delay left the detection delay
counted in source queries unchanged and added exactly that lag in calendar time. \\
\bottomrule
\end{tabular}
\end{center}
\end{table}

\begin{table}[ht]
\caption{Operational and resource modular studies.}
\label{tab:operational_additional}
\begin{center}
\small
\begin{tabular}{@{}P{0.18\linewidth}P{0.74\linewidth}@{}}
\toprule
Question & Result \\
\midrule
Synthetic resource policy & Across $2{,}000$ paired streams with pre-change
risk $.13$ and allowances $b_{\mathrm{low}}=.185$, $b_{\mathrm{high}}=.150$,
policies shared the exogenous Bernoulli outcomes within each seed. At
$p_{\mathrm{post}}=.23$, fixed-low, evidence-triggered ($w=.10$, meaning
escalation once the wealth reaches $w/\edelta$) and fixed-high had detection
$.039/.163/.998$ at mean costs $1.00/1.23/4.00$; pre-drift alarms were at most
$.0025$ throughout. The adaptive policy raised detection over fixed-low by
$.125$ (paired $95\%$ Hoeffding interval $[.064,.185]$; $249$ adaptive-only
and no low-only detections) for $.234$ more cost per query
($[.052,.417]$), while spending $69\%$ less than fixed-high. At
$p_{\mathrm{post}}=.28$ fixed-low already saturates at $1.000$ and is
cheapest. The complete study covered four warning policies and eight drift
magnitudes; bandwidth changed the allowance and the cost, not the generated
errors. \\
\addlinespace[0.3ex]
Evidence-triggered sentinels & Twelve prespecified sentinel controllers
updated wealth only on high-bandwidth queries at a fixed period,
with the bet set to zero on low-bandwidth queries, and a warning changed
bandwidth only from the next query. Across both an interior and a
least-favorable boundary null, the largest of all $24$ simultaneous
null-alarm upper endpoints was $.0401$. At $p_{\mathrm{post}}=.17$, the
$w=.10$ controllers with sentinel periods $5,10,20$ improved unconditional
detection by $.112/.127/.123$ over randomized policies whose high-bandwidth rate
was matched to each controller's expected budget in an independent pilot run,
with $95\%$ intervals simultaneous over a family of $360$ comparisons:
$[.014,.209]$, $[.029,.224]$ and $[.025,.220]$. \\
\addlinespace[0.3ex]
Node removal & At $\caldelta=\edelta=.05$ (total $.10$), eight node-removal scenarios (one per node) crossed with five
splits yielded $40$ audited cells, none classified harmful; the paired
independent risk-audit changes relative to the healthy system were only
$-.00456$ to $+.00210$. Neither our monitor nor CTM alarmed on any of the
$240$ node-removal paths through $T=5{,}000$, while active nodes fell from eight to
seven and the logical payload from $960$ to $840$ bits/query; paired realized
$T=5{,}000$ path miss-rate changes lay in $[-.0042,+.0024]$. Audited-harmful
Computer Science and History concentrations were still detected in
every path by $T=5{,}000$, even with the node holding that subject's corpus removed. This is evidence for the
synchronous protocol under node removal only, not for asynchronous or Byzantine
operation. \\
\addlinespace[0.3ex]
Audit size & At $\caldelta=\edelta=.05$ (total $.10$), independent audits of $250/1{,}000/5{,}000/20{,}000$ draws gave
mean allowances $.1685/.1418/.1354/.1311$ and mean Clopper--Pearson overheads
$.03970/.01816/.00789/.00389$, a $10.2\times$ reduction from the smallest to
the largest audit. Each audit size is a standalone design, not a data-dependent
stopping rule. Every cell the independent audits called compatible, and every
pre-change path, remained alarm-free. \\
\addlinespace[0.3ex]
Monitor overhead & In a benchmark of six methods (the five below and online-conformal threshold adaptation) on one Intel Xeon Gold 6226 core,
the cumulative aGRAPA loop processed one million updates in $24.0$ seconds
($24.0\,\mu$s/update); the five-component mixture took $75.0$ seconds and the
full-history CTM $116.1$ seconds, against $0.6$ and $1.7$ seconds for CUSUM
and Shiryaev--Roberts. Process-wide peak resident set size across the whole
benchmark was about $100$\,MB ($100{,}171{,}776$ bytes). The memory figure is not attributable
to any one method, and these are absolute single-implementation measurements
rather than a cross-method speed ranking. \\
\bottomrule
\end{tabular}
\end{center}
\end{table}

The delayed-feedback and controller studies support their deliberately limited
formal counterparts; they do not justify outcome-dependent delay handling or
same-query decisions. The resource-policy and node-removal studies describe observed
cost--evidence tradeoffs, not causal effects of bandwidth on model quality, and
no communication optimality is claimed.

\subsection{GPT-2 frozen-system study}
\label{app:earlier}

This study evaluated a shared GPT-2-small student with
MiniLM retrieval on MMLU~\citep{hendrycks2021mmlu},
DBpedia~\citep{lehmann2015dbpedia} and AG News~\citep{zhang2015agnews} before
the Qwen studies were run. It kept the model and retriever fixed and used
three calibration splits crossed with five deployment paths per split. Its
allocations were $\edelta=\caldelta=.05$, a total level of $.10$, so it is
reported separately from the Qwen replay's $.025+.025=.05$.
Figure~\ref{fig:earlier} and Table~\ref{tab:earlier_system} are an
end-to-end check of the full system, not evidence for a model swap with fresh
state-matched calibration.

\begin{figure}[!t]
\begin{center}
\includegraphics[width=\linewidth]{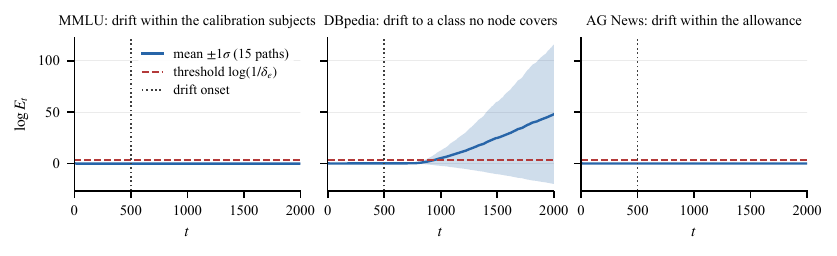}
\end{center}
\caption{The alarm tracks misses against the certified allowance, not
distribution shift as such. The fixed-state GPT-2-small/MiniLM system: $\log
E_t$ under a sudden change at $t=500$, alarm threshold dashed, $\pm1\sigma$ over
$15$ paths. Two shifts leave the miss rate under its allowance and keep the monitor
silent: redistribution \emph{within} the calibration subjects (MMLU, $.149$
against a $.201$ allowance) and drift to the AG News class no node covers, which stays within its
allowance ($.127$ against $.346$). Drift to a class no node covers (DBpedia, $.370$ against
$.252$) drives the mean evidence through the threshold; $5$ of its $15$ paths
alarm in this run (Table~\ref{tab:earlier_system}).}
\label{fig:earlier}
\end{figure}

\begin{table}[ht]
\caption{GPT-2 frozen-system results. The monitoring study used $T=2{,}000$
and the controller study $T=3{,}000$, both with a configured drift index of
$500$; the source generator switched at $t=500$. The monitoring post-change
mean is computed over queries $t=501,\dots,2{,}000$. The final column is the
source implementation's mean per-query $\sum_iB_i$, an abstract bandwidth-cost
unit.}
\label{tab:earlier_system}
\begin{center}
\small
\begin{tabular}{@{}lrrr@{}}
\toprule
Study and benchmark & Alarm rate & Mean post-change miss / allowance & Mean $\sum_iB_i$ \\
\midrule
Monitoring: MMLU        & $.00$  & $.149/.201$ & $40.00$ \\
Monitoring: DBpedia     & $.33$  & $.370/.252$ & $32.00$ \\
Monitoring: AG News     & $.00$  & $.127/.346$ & $24.00$ \\
Controller, DBpedia: fixed-low   & $1.00$ & n/a         & $32.00$ \\
Controller, DBpedia: adaptive    & $1.00$ & n/a         & $41.45$ \\
Controller, DBpedia: fixed-high  & $1.00$ & n/a         & $48.00$ \\
\bottomrule
\end{tabular}
\end{center}
\end{table}

Across the three benchmarks the evidence-triggered controller, in its own
runs at $T=3{,}000$, costs $34.4$, $41.5$ and $24.0$ mean per-query
$\sum_iB_i$ on MMLU, DBpedia and AG News at alarm rates $.33$, $1.00$ and
$.00$: it escalates most on DBpedia, where the shift is harmful, and never on
AG News. The sudden-drift arm of the separately seeded drift-type
sweep below alarms at the same $.33$, $1.00$ and $.00$; both differ from the
monitoring run of Table~\ref{tab:earlier_system}, which alarms on $5$ of its
$15$ DBpedia paths.

Three ablations on the same frozen system are worth recording, because each
tracks a dependence the allowance predicts. Under \emph{drift type} on DBpedia,
alarm rates rank by severity (no drift $.00$, sudden $1.00$, gradual $.67$,
periodic $.33$): gradual onset slows evidence accumulation, and a periodic
schedule returns to the calibration subjects and dilutes it. MMLU alarms at $.33$
under sudden drift and stays near zero otherwise, while AG News stays near zero
throughout. MMLU runs nearest its allowance, so this independently seeded sweep
alarms on one of its three splits where the run behind
Table~\ref{tab:earlier_system} alarms on none. Under \emph{heterogeneous bandwidth} on
MMLU, in another separately seeded sweep, alarms occur only at the highest bandwidth ($.27$ at $B_i\equiv10$, and $.00$ for
$(3,10,10,10)$, $(3,3,10,10)$ and $B_i\equiv4$), which is the direction Section~\ref{sec:setup} predicts, since
more bandwidth shrinks the communication term and tightens $b_t$. Under
\emph{multi-refresh} on MMLU, where two simulated student refreshes lower the
standard deviation of a noise perturbation on the student's logits along
$0.5\to0.25\to0$ at $t\in\{500,1500\}$, the allowance falls at each refresh and
the evidence growth rate adjusts to it; the final miss rate is $.120$ (standard
deviation $.027$) over $15$ paths. These $K\le4$ studies keep the model and
retriever fixed, apart from that simulated perturbation, and are reported
as mechanism illustrations, not as evidence for a certified model swap.

Under the scalar-message accounting of Appendix~\ref{app:communication}, the
MMLU/DBpedia/AG News monitoring rows correspond to $160/160/96$ bits per
query, and the DBpedia fixed-low/adaptive/fixed-high controller rows to
$160/207.26/240$ bits per query. The monitor stayed silent on MMLU and
AG News, whose measured post-change miss rates were below their source-reported
mean allowances, and alarmed on $5$ of $15$ DBpedia paths, where the miss rate was above the
allowance. On the severe DBpedia shift all three controller policies alarmed;
predictable escalation reduced mean $\sum_iB_i$ by $13.6\%$ relative to
fixed-high but remained
more expensive than fixed-low, so it is an intermediate operating point and
not a Pareto improvement. This is an end-to-end mechanism illustration with
$15$ correlated paths per benchmark.

\rev{%
\subsection{Every budget beside its measured rate}
\label{app:x-scoreboard}

The new experiments counted here draw a fresh calibration block in every replication (or know the
allowance) and know every miss rate exactly, so a wrong certificate and an alarm with every
certificate correct are observed, not inferred. One table in the supplement puts each budget beside
its measured rate, row by row: one row per experiment, quantity and monitor, one cell per
configuration. A cell counts as above its bound when its rate exceeds the bound with a one-sided
exact binomial $p$-value under $.01$.

\paragraph{Upper bounds.}
$394$ rows have a bound that no cell should exceed: an alarm although every certificate in force
was correct (bound $\edelta$), an alarm or a wrong certificate (bound $\edelta+\caldelta$), some
certificate wrong over a path (bound $\caldelta$), an alarm before a late change, and an alarm
where nothing was changed, with or without a refresh at a label-free flag, and a label-free flag
firing where nothing was changed. Over their $13{,}536$ cells, \emph{no cell is above its bound}. The
largest rates sit close to their bounds where the design puts every epoch exactly at its allowance:
$.0230$ against $.025$ for an alarm with every certificate correct (synthetic restart study; the cells of
the rate of \citet{timans2025monitoring}, which the scoreboard leaves out, reach $.0262$,
Appendix~\ref{app:x-restart}), and
$.0504$ against $.05$ for an alarm or a wrong certificate in one boundary cell of $5{,}000$ paths
($p=.46$). On real scores the largest are $.0185$ against $.025$ (an alarm before any query above
the allowance) and $.0455$ against $.05$ (an alarm or a wrong attempted certificate), both in the
long boundary null arm of the refresh study on QASC (Appendix~\ref{app:x-refresh}). The largest cells of
all are certificate cells: $.0295$ (synthetic, the study of many refreshes in
Appendix~\ref{app:x-cost}) and $.0285$ (real scores, QASC's design split of
Appendix~\ref{app:x-qasc}, a drift after about $118$ attempts) against $.025$ ($p=.11$ and $.18$);
the smallest $p$ on real scores, $.031$, is that of $.0280$ over $10{,}000$ paths (blocks taken from
the labeled traffic at one label in ten).

\paragraph{Exact levels.}
$24$ rows count the certificates of single attempts, each wrong with probability exactly its level
$\gamma_r$ (Proposition~\ref{prop:beta}), so about one cell in a hundred lies above its level at
$p<.01$ by chance. Over their $3{,}830$ cells, $13$ do, where at most about $38$ are expected; pooled,
the certificates are wrong on $152{,}793$ of $10{,}826{,}942$ draws, where their levels give $152{,}911$.

\paragraph{Clopper--Pearson levels.}
$4$ rows count the bounds that the second label-free flag of Appendix~\ref{app:x-node} certifies on
the rate of each of its two indicators, at level $.0125$. Such a bound is conservative, but attained
up to the binomial's discreteness, so again about one cell in a hundred may lie above its level at
$p<.01$ by chance. Over their $1{,}280$ cells, $2$ do, where at most about $13$ are expected; pooled,
the bounds are below the rate on $28{,}513$ of $2{,}560{,}000$ draws, where the level allows $32{,}000$.

\begin{table}[ht]\rev{%
\caption{The composed study of Appendix~\ref{app:composed} at $R=100$, with its two events
kept apart ($5{,}000$ null paths; $\edelta=\caldelta=.05$). \emph{Certificate}: some certificate
on the path is wrong. \emph{Alarm, correct}: the alarm fires although every certificate is
correct.}
\label{tab:composed_split}
\begin{center}
\small
\begin{tabular}{@{}lrrrrr@{}}
\toprule
Scheme & Certificate & Alarm & Alarm, correct & Alarm, a certificate wrong & Either \\
\midrule
Frozen (one certificate) & $.0500$ & $.0262$ & $.0024$ & $.0238$ & $.0524$ \\
Unallocated (full level each) & $.9918$ & $.0020$ & $.0000$ & $.0020$ & $.9918$ \\
Bonferroni & $.0488$ & $.0002$ & $.0002$ & $.0000$ & $.0490$ \\
Summable & $.0500$ & $.0012$ & $.0012$ & $.0000$ & $.0512$ \\
\bottomrule
\end{tabular}
\end{center}}%
\end{table}
}%

\rev{%
\subsection{The composition against budgeted restarts and adapted risk monitors}
\label{app:x-restart}

\paragraph{Design.}
Synthetic miss streams with exact rates: a benign stretch at the allowance's nominal, middle or
near-allowance level, then a change of $+.02$, $+.04$ or $+.08$ that persists, or that a refresh
repairs; refreshes every $100$ to $2{,}500$ queries or triggered by evidence; certificates on blocks
of $200$ or a known allowance; $116$ cells of $5{,}000$ paths, then $24$ cells with the change off
the refresh grid and $16$ cells with a late change near the allowance. Each bettor runs in four
compositions on the same paths: never reset (the protocol); restarted at each refresh with
summable shares of $\edelta$ ($6\edelta/(\pi^2m^2)$, nothing declared); restarted with Bonferroni
shares, which need the number of refreshes declared in advance (the \emph{capped} restart); and
restarted with the full $\edelta$ each time, which is not valid. Bettors: the cumulative aGRAPA
bettor of the original experiments, the five-scale mixture, the start-time mixture of
Proposition~\ref{prop:delaybound}, the equal mix of these two mixtures (the default), and two
monitors of prior work run inside the composition: the betting rate of
\citet{timans2025monitoring}, following their released tracker (windows of $50$ and $365$; the
window truncates the estimates and never the wealth), and a running-risk test in the style of
\citet{podkopaev2022tracking} (the conjugate-mixture empirical-Bernstein bound their drift
experiments use). For the default, the never-reset mixture of a restart's own processes runs beside
the restart, with each process frozen when the next starts (Proposition~\ref{prop:dominate}) or
kept betting (the \emph{kept mixture}, Proposition~\ref{prop:kept}). Every design was specified,
with its expected outcome, before it ran.

\paragraph{Validity.}
On null and boundary streams every valid composition stays within its bounds: the largest share of
paths with an alarm although every certificate was correct is $.0262$ against $\edelta=.025$ (exact
one-sided $p=.31$, the largest of $264$ comparisons), and with an alarm or a wrong certificate
$.0528$ against $.05$ ($p=.19$). The restart that spends the full $\edelta$ each time alarms on up to
$.68$ of boundary paths.

\begin{table}[ht]\rev{%
\caption{One never-reset process against budgeted restarts of the same bettor, on the same paths:
share of paths detected within $5{,}000$ queries of the change ($20{,}000$ for the late changes), and
median delay in queries where at least half are detected. \emph{Ceiling}: the bound of
Corollary~\ref{cor:ceiling} on any restarted bettor in that cell.}
\label{tab:x_restart}
\begin{center}
\footnotesize
\setlength{\tabcolsep}{3.5pt}
\begin{tabular}{@{}P{0.36\linewidth}P{0.14\linewidth}ccc@{}}
\toprule
Cell & Bettor & Never reset & \shortstack{Restart,\\summable} & \shortstack{Restart,\\capped} \\
\midrule
Refresh every $100$, change after $500$, excess $.04$ & default & $.999$ ($1{,}254$) & $.006$ & $.045$ \\
\quad the same, allowance known & default & $1.000$ ($1{,}089$) & $.009$ & $.060$ \\
\quad ceiling for any restarted bettor & & & $.122$ & $.524$ \\
\quad the same change off the refresh grid ($730$) & default & $.999$ ($1{,}374$) & $.004$ & $.046$ \\
Refresh every $1{,}000$, change after $50{,}000$, excess $.04$ & default & $.998$ ($1{,}894$) & $.270$ & $.906$ ($5{,}921$) \\
\quad the same & aGRAPA & $0$ & $.269$ & $.910$ ($5{,}826$) \\
\quad the same & Timans, $S{=}365$ & $.999$ ($1{,}166$) & $.210$ & $.885$ ($6{,}710$) \\
\quad the same & running risk & $0$ & $.016$ & $.516$ ($18{,}998$) \\
\quad the same, refresh by evidence & default & $.999$ ($1{,}872$) & $.008$ & --- \\
Refresh every $1{,}000$, change after $50{,}000$ near the allowance, excess $.08$ & default & $.997$ ($886$) & $.997$ ($777$) & $.998$ ($530$) \\
\bottomrule
\end{tabular}
\end{center}
}%
\end{table}

\paragraph{What one never-reset process buys, with the bettor held fixed.}
Over the $84$ cells with harm, the never-reset five-scale mixture detects more than its summable
restart in $58$ (paired $95\%$ intervals above zero) and fewer in none, and is faster in $18$ of the
other $26$; against the capped restart it detects more in $18$ of $57$ cells and fewer in none.
With the change off the refresh grid the summable restart of the five-scale mixture or of the
default never leads. On real scores, after a drift with blocks still drawn from the old law and
refreshes chosen along each path (Appendix~\ref{app:x-refresh}), the never-reset default detects
$.400$, $.213$ and $.565$ of the paths above their allowance (QASC's evaluation and design splits,
Appendix~\ref{app:x-qasc}, and MMLU-Pro) where its summable restart detects $.031$, $.135$ and
$.215$; the restart that spends the full budget each time detects $.219$, $.322$ and $.541$ and is
not valid: at the boundary it alarms on $.0255$, $.0213$ and $.0278$ of paths with about five
attempts each, and on $.35$ to $.38$ with about $70$.

\paragraph{Where another bettor or a restart is sooner.}
For a fixed state the composition adds nothing to the bettor, as expected: the rate of \citet{timans2025monitoring} with its
window of $365$, run inside the composition, never detects less than the five-scale mixture,
detects more in $3$ cells and has the shorter median delay in $54$ of the other $81$; its window of
$50$ is weaker at small excess and fails after a long benign period at the middle level. The
running-risk test is valid and about twice as slow, and stops detecting after long benign periods,
as its null (the average risk since deployment) implies. The advantage over restarts belongs to the
default and not to the cumulative aGRAPA bettor: never reset, aGRAPA detects on no path after
$50{,}000$ benign queries well below the allowance, and a capped restart of it is ahead in every
cell with a change after $10{,}000$ or more benign queries ($16$ of $16$ on the grid, $8$ of $8$ off
it). Even for the default, a valid restart is sooner in $19$ combinations of a cell and a restart
rule on the grid ($13$ distinct cells; in $18$ of the $19$ the change falls exactly on a refresh), by
up to $294$ queries in the median paired difference (the last row of Table~\ref{tab:x_restart} is one such
cell); with only two
attempts and a change at the refresh, a restart that gives each attempt half the budget is level or
slightly ahead on two of three real-score populations. With the change off the grid, a capped
restart detects slightly more than the never-reset process in $4$ of $48$ restart--bettor
comparisons with certificates and $6$ of $48$ with a known allowance, all at an excess of $.08$, by
$8$ to $24$ paths of $5{,}000$: the difference in alarms raised before the change, which count as
misses.

\paragraph{Keeping a restart's own processes.}
In every combination where a valid restart is sooner than the default and the kept mixture can be
run ($18$ of the $19$), the kept mixture of that restart is sooner than the default (by $3$ to $318$
queries in the median paired difference) and not later than the restart on any path (sooner in
$15$, level in $3$). No mixture alarms later than its restart in any of $3{,}064{,}721$
restart alarms. Over the $65$ cells with harm that refresh on a schedule, however, neither the kept
mixture nor the default is ahead everywhere: with summable weights the mixture is sooner in $37$,
level in $11$ and later in $17$; with the restart's declared cap, sooner in $20$, level in $12$ and
later in $33$, all of them cells with an early change, where its median delay is up to $49\%$
longer. With the change off the grid, the corner in which a capped restart leads is closed by the
kept mixture with the same cap and not by the summable one. A declared cap is information that can
be wrong: declared five times too large it costs the capped restart its lead in all five cells where
it had one, and declared too small and spent before the change it leaves the restart with no alarm
on any of $39{,}962$ paths, while the kept mixture with the same wrong cap detects at least $.997$ of
them. We keep the default, which needs neither a schedule nor a declaration, and offer the kept
mixture where refreshes are scheduled.

\paragraph{Late onset.}
After $500$, $2{,}000$, $10{,}000$ and $50{,}000$ benign queries at the nominal level and an excess of
$.04$, the default detects on $.998$ of paths or more, with median delays of $1{,}094$, $1{,}417$,
$1{,}706$ and $1{,}894$ queries (five-scale mixture $1{,}133$, $1{,}328$, $1{,}524$, $1{,}694$;
start-time mixture $1{,}094$, $1{,}577$, $2{,}307$, $3{,}169$). For $50{,}000$ benign queries
Proposition~\ref{prop:delaybound} bounds the start-time mixture's mean delay by $3{,}361$ at an
allowance of $.175$ ($3{,}827$ at the cell's $.215$); its measured median is $3{,}169$.
}%

\rev{%
\subsection{The matched baselines at the default operating point, and the Qwen grid replayed}
\label{app:x-baselines}

\paragraph{The baseline comparison, rerun with monitors of the same null.}
Table~\ref{tab:x_baselines} reruns the study of Appendix~\ref{app:baselines} on its own $2{,}000$
paired streams (allowance $.10$ known, change after $2{,}000$ of $5{,}000$ queries, budget $.05$),
with the cumulative aGRAPA bettor, the five-scale mixture, the start-time mixture, the default (the
equal mix of the two mixtures), the rate of \citet{timans2025monitoring} with and without its
window, and a running-risk test in the style of \citet{podkopaev2022tracking}, each one never-reset
process tested against $1/.05$. The cumulative aGRAPA bettor reproduces the original study's row
alarm for alarm.

\begin{table}[ht]\rev{%
\caption{The matched baselines on the streams of Appendix~\ref{app:baselines}: share of streams
with an alarm after the change (an alarm before it counts as a miss) and median delay in queries.
\emph{Null}: alarm rate on $2{,}000$ fresh streams held at the allowance for $5{,}000$ and for
$20{,}000$ queries. $^{\star}$Given the post-change rate; thresholds calibrated for $T=5{,}000$ on
separate null streams.}
\label{tab:x_baselines}
\begin{center}
\footnotesize
\setlength{\tabcolsep}{2.4pt}
\begin{tabular}{@{}P{0.215\linewidth}cccccc@{}}
\toprule
Monitor & $p_{\mathrm{post}}=.123$ & $.146$ & $.190$ & $.271$ & $.410$ & \shortstack{Null, $5{,}000$;\\$20{,}000$} \\
\midrule
Cumulative aGRAPA (original) & $.603$ ($2{,}057$) & $.987$ ($1{,}101$) & $.988$ ($514$) & $.988$ ($257$) & $.988$ ($139$) & $.022$; $.027$ \\
Five-scale mixture & $.745$ ($1{,}615$) & $.981$ ($652$) & $.981$ ($244$) & $.981$ ($103$) & $.981$ ($51$) & $.028$; $.034$ \\
Default (equal mix) & $.758$ ($1{,}625$) & $.977$ ($653$) & $.977$ ($219$) & $.977$ ($74$) & $.977$ ($28$) & $.035$; $.043$ \\
Start-time mixture & $.761$ ($1{,}684$) & $.973$ ($674$) & $.973$ ($215$) & $.973$ ($70$) & $.973$ ($25$) & $.037$; $.043$ \\
Timans et al., window $365$ & $.764$ ($1{,}315$) & $.979$ ($533$) & $.979$ ($242$) & $.979$ ($122$) & $.979$ ($67$) & $.023$; $.032$ \\
Timans et al., no window & $.648$ ($1{,}995$) & $.987$ ($1{,}053$) & $.988$ ($493$) & $.988$ ($246$) & $.988$ ($134$) & $.012$; $.019$ \\
Running risk & $.560$ ($2{,}106$) & $.996$ ($1{,}179$) & $.996$ ($545$) & $.996$ ($271$) & $.996$ ($148$) & $.005$; $.009$ \\
CTM (exchangeability) & $.343$ ($1{,}684$) & $.970$ ($875$) & $.975$ ($394$) & $.975$ ($217$) & $.975$ ($135$) & --- \\
CUSUM$^{\star}$ & $.917$ ($1{,}128$) & $.982$ ($442$) & $.982$ ($156$) & $.979$ ($57$) & $.987$ ($24$) & --- \\
Shiryaev--Roberts$^{\star}$ & $.936$ ($1{,}123$) & $.980$ ($435$) & $.983$ ($156$) & $.978$ ($58$) & $.985$ ($24$) & --- \\
\bottomrule
\end{tabular}
\end{center}
}%
\end{table}

\paragraph{Paired differences.}
The streams are shared, so differences are paired. Near the allowance ($p_{\mathrm{post}}=.123$) the
default is ahead of aGRAPA by $+.156$ ($95\%$ interval $[+.138,+.173]$; $338$ streams on which only
the default detects, $27$ on which only aGRAPA does) and behind CUSUM by $.159$ $[.142,.175]$ and
Shiryaev--Roberts by $.178$ $[.161,.195]$; exact McNemar $p$-values are below $10^{-68}$ for all
three. The rate of \citet{timans2025monitoring} with its window is ahead of aGRAPA by $+.162$
$[+.141,+.182]$, and the running-risk test behind it by $.043$ $[.030,.056]$. The rerun kept
discordant counts only against three reference monitors, so $16$ of its $36$ pairs per cell have
bounds rather than intervals, and no paired difference in delay can be computed.

\paragraph{Reading.}
Among monitors of the same risk null none dominates: the windowed rate of
\citet{timans2025monitoring} is as good as the default near the allowance and faster at a change to
$.146$, the mixtures are faster at large changes, and the running-risk test has the lowest null
rate and is slowest. The detectors that are told the post-change rate are well ahead near the
allowance and faster everywhere; their finite-horizon thresholds exceed the budget when kept to
$20{,}000$ queries ($.193$ to $.274$ in the original study), which shows what a horizon costs and is
not an advantage of ours over time-uniform risk monitoring.

\paragraph{The Qwen grid replayed with the default.}
The Qwen grid of Appendix~\ref{app:qwen}, replayed from its caches with the default bettor in place
of the cumulative aGRAPA bettor, gives the same separation (all $210$ paths of audited-harmful cells
alarm, none of the other $240$) at a median delay of $124$ queries against $346$ (five-scale mixture
$155$). That grid's limits (shared pools, no cell between $+.0055$ and $+.0228$ of excess) are those of
Appendix~\ref{app:qwen}.
}%

\rev{%
\subsection{What repeated certification costs}
\label{app:x-cost}

\paragraph{Allowance, labels and sets over a hundred refreshes.}
Under the summable schedule attempt $r$ is certified at level $6\caldelta/(\pi^2r^2)$, so with a
fixed block the allowance grows with $r$. Table~\ref{tab:x_cost} gives it exactly, with what it does
to the answer set of a useful scorer and to detection. The guarantee does not weaken with the
number of refreshes; the sensitivity does unless the blocks grow, so the block should be sized for
the number of attempts a deployment expects: with blocks of $100$, the size used in
Appendix~\ref{app:liverefresh}, the allowance is twice the target by the third attempt, while
blocks of $1{,}000$ keep it at $.150$ or below over a hundred attempts.

\begin{table}[ht]\rev{%
\caption{Repeated certification, summable schedule, target $.10$, $\caldelta=.025$. Allowance:
exact. Sets: the 7B student on QASC (evaluation split; design split in brackets;
Appendix~\ref{app:x-qasc}), mean set of eight at the attempt's threshold, under the conformal order
of Section~\ref{sec:construct} or the \emph{strict order}, the smallest order whose allowance is at
most the target, which holds the allowance at the target and lets the set grow (a dash: no such
order). Detection: synthetic streams, the default monitor, a shift to a miss rate of $.20$ at the
attempt's epoch (epochs of $500$ queries), share detected within $5{,}000$ queries and median
delay.}
\label{tab:x_cost}
\begin{center}
\footnotesize
\setlength{\tabcolsep}{4pt}
\begin{tabular}{@{}lcccc@{}}
\toprule
Attempt & $1$ & $5$ & $20$ & $100$ \\
\midrule
Level of the attempt & $.015$ & $6.1\times10^{-4}$ & $3.8\times10^{-5}$ & $1.5\times10^{-6}$ \\
Allowance, blocks of $100$ & $.172$ & $.218$ & $.251$ & $.286$ \\
Allowance, blocks of $1{,}000$ & $.121$ & $.133$ & $.141$ & $.150$ \\
Set, conformal order, blocks of $1{,}000$ & $2.49$ ($1.06$) & $2.49$ ($1.06$) & $2.49$ ($1.06$) & $2.49$ ($1.06$) \\
Set, strict order, blocks of $1{,}000$ & $2.85$ ($1.13$) & $3.06$ ($1.20$) & $3.22$ ($1.25$) & $3.38$ ($1.29$) \\
Set, strict order, blocks of $100$ & $4.12$ ($1.60$) & $5.24$ ($2.66$) & $6.21$ ($4.02$) & --- \\
Detection of a shift to $.20$, blocks of $1{,}000$ & $1.000$ ($197$) & $.999$ ($501$) & $.998$ ($772$) & $.999$ ($1{,}090$) \\
Detection of a shift to $.20$, blocks of $100$ & $.145$ & $0$ & $0$ & $0$ \\
\bottomrule
\end{tabular}
\end{center}
}%
\end{table}

With blocks of $100$ a shift to $.20$ is within the allowance from the third attempt and is, correctly
by the null, not detected. The delay grows over a hundred refreshes for two reasons: the excess over
the allowance shrinks as the allowance grows, and a later change waits behind a longer benign
history (Proposition~\ref{prop:delaybound}); with the allowance held constant the delay still
roughly doubles. Many refreshes also call for a bettor that does not stall: with blocks of
$1{,}000$ the cumulative aGRAPA bettor detects nothing from about the twentieth epoch, where the
default still detects $.998$ and $.999$, so long deployments should use the default
(Appendix~\ref{app:x-restart}). On real scores with $70$ to $130$ attempts a path, certificates fail
at their levels attempt by attempt ($62$ of $4{,}000$ at the first attempt, at a level of $.015$;
$3$ from the twentieth on, where their levels give $2.6$), and labeling separately costs about
$70{,}000$ to $130{,}000$ questions a path over $100{,}000$ queries (Appendix~\ref{app:x-refresh}).

\paragraph{Between the target and the allowance.}
The monitor tests the allowance, so a miss rate above the target and below the allowance is not
detected, by design. On QASC's evaluation split, $.42$ of shifted paths end in that band with blocks
of $100$ and $.28$ with blocks of $1{,}000$. The strict order removes the band: the monitor then
tests the target itself, at a cost of $.3$ of an option for the 7B student with blocks of $1{,}000$
and $1.2$ with blocks of $100$, and it detects half to three quarters of the paths above the target.

\paragraph{Calibration blocks from the traffic the monitor already labels.}
Remark~\ref{rem:traffic} shows that the next $n$ labeled queries, scored in shadow by a state chosen
before them, form a block to which Proposition~\ref{prop:beta} applies as stated, provided the traffic
is stationary over the block and the epoch it certifies. In a simulation in which the block and the
monitor's own bets are functions of the same questions (uniform scores; a refresh every $1{,}000$
queries for $100$ epochs; $200{,}000$ paths for each block size), such blocks are wrong on $.0245$
and $.0250$ of paths against a bound of $.025$, and need no label beyond the first block, where
blocks labeled separately need $200$ or $1{,}000$ a refresh. When the traffic shifts, the first
block drawn wholly after the shift certifies the shifted law, so the state is repaired at the next
refresh: the queries served above the allowance fall by a factor of $3$ to $63$ against blocks
labeled separately from the old law. The costs are a latency of $n$ labeled queries per refresh,
stationarity across the block, and a block that straddles the shift, which can leave a state above
its allowance for one more epoch. Synthetic scores; the next paragraph has real ones.

\paragraph{On real scores, at a label budget.}
One deployment on QASC walks an upgrade path of ten states, each step changing one of student,
retrieval policy, corpora and bit budget, under the silent-node rule and under FC-RAG's mean. A
refresh is decided every $2{,}000$ labeled queries, $20$ refreshes a path; blocks hold $1{,}000$ or
$500$ questions; one query in $1$, $4$ or $10$ is labeled, and the monitors bet on labeled queries
only (a zero bet otherwise). A block is the next $n$ labeled queries, scored in shadow by the
incoming state (\emph{traffic}), $n$ questions labeled separately from the law in force
(\emph{reference}), or the same from the calibration law whatever the traffic (\emph{stale}); the
first block of a path is labeled separately. Under stationary traffic, in all $24$ cells of
$10{,}000$ paths, some certificate is wrong on $.0210$ to $.0280$ of paths (none above $.025$ at
$p<.01$), the default alarms before any state is above its allowance on at most $.0015$, and either
happens on at most $.0288$; cells that differ only in the label rate agree. Traffic blocks need $n$
labels beyond monitoring a path, separately labeled ones $21n$. After all of the traffic moves to the
topics no node holds, traffic blocks leave the state within its allowance at the horizon on all but
$6$ of $24{,}000$ paths, as reference blocks do (all but $5$), and stale blocks on $.71$ to $.92$ of
paths under the silent-node rule; the cost is the latency of $n$ labeled queries, which in queries
grows as the inverse label rate (Table~\ref{tab:x_traffic}). The default detects few of the exposed
paths ($.03$ to $.10$): a block from the shifted traffic repairs the state before the evidence is
in. Under FC-RAG's mean this shift hardly moves the miss rate on the evaluation split (at most
$.001$ of post-shift queries above the allowance). On the design split the same holds: some
certificate is wrong on $.0215$ to $.0267$ of paths in the $24$ stationary cells (none above $.025$
at $p<.01$), the default alarms before any state is above its allowance on at most $.0014$, and
after a full shift traffic blocks leave all but $2$ of $24{,}000$ paths within the allowance at the
horizon, where stale blocks leave $.64$ to $.79$ under the silent-node rule.

\begin{table}[ht]\rev{%
\caption{Blocks from the labeled traffic on real scores: QASC, silent-node rule, blocks of
$1{,}000$, evaluation split. Stationary traffic: labels per query and the share of $10{,}000$ paths on
which some certificate is wrong (budget $.025$). All traffic moved to the topics no node holds ($2{,}000$
paths): the share of post-shift queries served above the allowance, the share of paths whose state is
within its allowance at the horizon, and post-shift queries above the allowance per path.}
\label{tab:x_traffic}
\begin{center}
\footnotesize
\setlength{\tabcolsep}{4pt}
\begin{tabular}{@{}llccccc@{}}
\toprule
Labeled & Blocks & \shortstack{Labels\\per query} & \shortstack{Certificate\\wrong} & \shortstack{Queries above,\\share} & \shortstack{Within at\\the horizon} & \shortstack{Queries above,\\per path} \\
\midrule
every query & traffic & $1.024$ & $.0215$ & $.056$ & $1.000$ & $444$ \\
 & reference & $1.500$ & $.0230$ & $.044$ & $1.000$ & $350$ \\
 & stale & $1.500$ & --- & $.306$ & $.707$ & $2{,}447$ \\
\addlinespace[0.3ex]
one in four & traffic & $.256$ & $.0210$ & $.058$ & $1.000$ & $1{,}864$ \\
 & reference & $.375$ & $.0234$ & $.040$ & $1.000$ & $1{,}265$ \\
 & stale & $.375$ & --- & $.307$ & $.721$ & $9{,}835$ \\
\addlinespace[0.3ex]
one in ten & traffic & $.102$ & $.0227$ & $.055$ & $1.000$ & $4{,}380$ \\
 & reference & $.150$ & $.0223$ & $.041$ & $1.000$ & $3{,}311$ \\
 & stale & $.150$ & --- & $.311$ & $.717$ & $24{,}895$ \\
\bottomrule
\end{tabular}
\end{center}
}%
\end{table}
}%

\rev{%
\subsection{The first benchmark revisited: a better scorer, refreshes chosen from history, and the model swap redrawn}
\label{app:x-first}

\paragraph{The scorers, and why retrieval does not help here.}
Every ten-option MMLU-Pro question was scored under every node context by the scorer of
Appendix~\ref{app:qwen} (the full option's mean negative log-likelihood under Qwen2.5-1.5B) and by
a letter score (the log-probability of the option's letter) for three Qwen2.5 students. Under the calibration law (one of eight subjects uniformly, then
one of its questions), top-1 accuracy over eight nodes is $.137$ for that scorer, and $.137$,
$.243$ and $.177$ for the letter score of the 0.5B, 1.5B and 7B students; with no retrieved evidence
the letter score reaches $.166$, $.286$ and $.418$. Retrieved evidence lowers every student's accuracy
here, and sharply the 7B student's: the nodes' passages are worked solutions of other questions,
whose answer letters mislead the letter score. Eight nodes, four, two and one are within a point and a
half of each other. The most accurate system, the 7B student with no evidence, keeps $5.97$ of ten
options at $90\%$ coverage; the scorer of Appendix~\ref{app:qwen} keeps $8.84$. This benchmark
therefore tests the monitor and not the application; Appendix~\ref{app:x-qasc} adds one that tests
both.

\paragraph{Fixed state, with a fresh calibration block in every replication.}
$72$ cells of $2{,}000$ replications over four eight-node systems, under no shift and under a shift to
each of $14$ subjects (and $54$ cells for the three students with no evidence). The certificate is
wrong on $3{,}523$ of $144{,}000$ replications ($.0245$) and on $2{,}734$ of $108{,}000$ ($.0253$), at a
level of $.025$. With no shift the default alarms with the certificate correct on $.0010$ of
replications in both studies, at most $.0031$ and $.0046$ in a cell (bound $.025$), and an alarm or
a wrong certificate occurs in at most $.0340$ of a cell's replications (bound $.05$). Detection is a
function of the exact excess of the shifted miss rate over the allowance and of little else: for
the default, $.065$ below one point, $.489$ at one to two points, $.932$ at two to four, $.997$ at
four to eight, and $.995$ and $.996$ beyond; for the cumulative aGRAPA bettor, $.034$, $.388$,
$.885$ and $.996$. The certificate and the monitor behave the same for a scorer at $.137$ and one
at $.418$ top-1.

\paragraph{Refreshes chosen from history.}
A controller walks $24$ states (three students, two retrievers, two corpus assignments, two bit
budgets) with the triggers of Appendix~\ref{app:x-refresh}; every block is drawn after the state is
chosen from the law in force, and every replication draws its own. In the null arms ($20{,}000$ paths
each) some attempted certificate is wrong on $.0221$ and $.0222$ of paths (bound $.025$), the default
alarms with every certificate correct on $.0003$ and $.0153$ (bound $.025$), and either occurs on
$.0223$ and $.0376$ (bound $.05$). The legal order is wrong on $14{,}853$ of $600{,}000$ blocks at
$.025$ over $30$ ordered pairs of states. Each move the protocol forbids breaks its budget:
calibrating a stronger incoming student on the outgoing student's scores is wrong on at least
$.9999$ of draws; choosing the best of several states on the block that certifies it, on up to
$.0775$ at a level of $.025$; redrawing blocks until the set looks small, on $.130$ when the tries are
not charged and $.018$ when each is. With stale blocks after a drift the default detects $.565$ of
the paths above their allowance, and $.507$ when the drift comes after $20{,}000$ queries.

\paragraph{The model swap, redrawn.}
The model swap of Appendix~\ref{app:liverefresh} shared two calibration blocks over its $200$
paths. Rerun with both blocks redrawn on every path, from the same subjects, students and allowance,
the first certificate of the legal order is wrong on $9{,}050$ and the second on $9{,}031$ of
$720{,}000$ draws, at a level of $.0125$ each. The drift begins under the outgoing state, before the
swap, and after the swap the state in force is above its allowance on about $.8$ of paths, not all.
The monitor detects the drift on $.61$ to $.66$ of all paths with aGRAPA, $.86$ to $.89$ with the
five-scale mixture and $.88$ to $.89$ with the default, where Table~\ref{tab:liverefresh} reports
$.925$ and $.955$; of the paths whose state is above its allowance after the swap, the default
detects $.94$. In the null and boundary arms of
these swaps, every miss rate exact, the default alarms with both certificates correct on at most
$.0140$ of a cell's paths (budget $.025$), and an alarm or a wrong certificate occurs on at most
$.0399$ (budget $.05$). Over $20{,}000$ redrawn blocks, each scored by both models, the downgrade's
miss rate is $.099$ under the legal order and $.164$ under calibrate-then-replace, a mean paired
difference of $+.0648$; the original run's $.089$ against $.154$ was one draw. That run itself,
replayed with every path kept, gives a paired difference of $+.0650$ ($95\%$ interval
$[+.0641,+.0658]$, conditional on its two blocks), with the illegal order worse on $200$ of $200$
paths.

\paragraph{A node changes with no new certificate.}
On this benchmark a node-side change rarely matters: only a change of student moves the miss rate
much, and the other $62$ changes move it by at most $.03$, because retrieval does not help here. The
default alarms on $5$ of $125{,}580$ replications in which a change left the guarantee intact, and
detects $.726$ of $6{,}420$ harmful ones within $4{,}500$ queries.
}%

\rev{%
\subsection{A second benchmark on which retrieval matters}
\label{app:x-qasc}

\paragraph{Why a second benchmark, and its design.}
On MMLU-Pro with five worked solutions per node, retrieved evidence lowers the accuracy of every
student and eight nodes do no better than one (Appendix~\ref{app:x-first}), so that benchmark can
test the monitor but not the application. QASC~\citep{khot2020qasc} pairs eight-option science
questions with a corpus in which the evidence for every question exists. Its $7{,}671$ distinct
facts were clustered into twelve topics; node $j$ holds topic $j$ for the eight largest, and four
topics are held by no node. The calibration law picks one of the eight held topics uniformly and
then one of its questions. Students are Qwen2.5 0.5B, 1.5B and 7B Instruct~\citep{qwen2024qwen25},
scoring each option by its letter's log-probability; twelve bits per score. Two retrieval settings
(an all-mpnet-base-v2 encoder with ten facts retrieved by one query per option, and the MiniLM
encoder of Appendix~\ref{app:qwen} with five) and four hub rules: FC-RAG's mean over all nodes; the
\emph{silent-node rule}, in which a node answers only if its best fact has similarity at least
$.65$ to the query (the \emph{gate}) and the hub averages the nodes that answer (falling back to one
node with no evidence); the most similar node alone; and the mean of $1-p$ over the two most similar
nodes. Each rule's exact dither law is implemented and tested against a message-by-message
simulation (Remark~\ref{rem:hubrules}). Designs were piloted on QASC's development questions, the
\emph{design split} ($926$ questions), and are evaluated on its training questions, the
\emph{evaluation split} ($8{,}134$ questions); both are reported. The evaluation split is much
harder, for a reason the benchmark's authors give: they removed valid distractors from the
development and test splits only. A choice rule written down before the pilot was
read picked the two-most-similar-nodes rule; the silent-node rule, chosen with its gate on the
design split, leads, with that pick beside it. Every number below is on the evaluation split unless
stated.

\paragraph{What a hub can do under the same information-sharing constraints.}
Table~\ref{tab:qasc_systems} sets the federation's rules beside the alternatives that send no
more than score vectors, and beside three systems no federation may use.

\begin{table}[ht]\rev{%
\caption{QASC, evaluation split: top-1 accuracy, mean answer set at $90\%$ coverage (of eight
options; exact under the dither) and score vectors sent per query (answers, for the vote), lead
retrieval setting. The last three rows need documents or labels no federation has. Design split in
brackets for the 7B student.}
\label{tab:qasc_systems}
\begin{center}
\footnotesize
\setlength{\tabcolsep}{4pt}
\begin{tabular}{@{}P{0.40\linewidth}ccccc@{}}
\toprule
System & 0.5B & 1.5B & 7B & Set, 7B & Vectors \\
\midrule
No retrieval & $.294$ & $.502$ & $.606$ ($.775$) & $3.20$ & $1$ \\
One node alone (mean over the eight) & $.239$ & $.444$ & $.558$ ($.748$) & --- & $1$ \\
Plurality vote of the nodes' top answers & $.242$ & $.483$ & $.590$ ($.804$) & --- & $8$ \\
Mean over all eight nodes (FC-RAG's rule) & $.279$ & $.493$ & $.604$ ($.822$) & $3.03$ & $8$ \\
Silent-node rule (the lead) & $.392$ & $.579$ & $.671$ ($.898$) & $2.48$ & $1.58$ \\
Most similar node alone & $.404$ & $.591$ & $.679$ ($.897$) & $2.54$ & $1$ \\
Two most similar nodes (the choice rule's pick) & $.377$ & $.572$ & $.671$ ($.897$) & $2.66$ & $2$ \\
Silent-node rule, the corpus cut at random & $.358$ & $.554$ & $.653$ ($.872$) & $2.65$ & $4.11$ \\
\addlinespace[0.3ex]
\emph{One retriever over the eight pooled corpora} & $.387$ & $.591$ & $.690$ ($.900$) & $2.48$ & $1$ \\
\emph{One retriever over the complete corpus} & $.389$ & $.584$ & $.693$ ($.895$) & $2.50$ & $1$ \\
\emph{The question's two gold facts} & $.740$ & $.905$ & $.912$ ($.971$) & $0.97$ & $1$ \\
\bottomrule
\end{tabular}
\end{center}
}%
\end{table}

With intervals that treat each topic's questions as a sample from the topic, the silent-node rule
gains $+.098$ $[+.085,+.111]$, $+.077$ $[+.064,+.089]$ and $+.065$ $[+.054,+.076]$ over no retrieval
for the three students; FC-RAG's mean is level with no retrieval ($-.015$, $-.010$, $-.002$; on the design split
$+.071$, $-.005$, $+.046$). Routing to the most similar node alone is ahead of the silent-node rule by
about a point ($+.011$, $+.012$, $+.008$, each interval above zero), and one retriever over the
pooled corpora, which no federation may use, is level for the 0.5B student and $.012$ and $.019$
ahead for the other two. The data therefore support retrieving at the right node more than averaging
across nodes. The silent-node rule sends $1.58$ score vectors per query where FC-RAG's mean sends
eight. Its one constant matters: the gate set on the design split is where the evaluation split puts
its optimum ($.65$ to $.70$ within a point), and a gate a tenth lower costs $3.2$ to $3.7$ points, a
tenth higher $1.8$ to $3.3$.

\paragraph{Bits, shifts, certificates and detection.}
On real scores the number of bits hardly moves the sets under the lead rule (for the 7B student,
$2.48$, $2.48$, $2.49$, $2.50$ and $2.55$ options at $12$, $8$, $6$, $4$ and $3$ bits), so the
synthetic bandwidth study of Appendix~\ref{app:modular} has no counterpart here; under the
two-most-similar-nodes rule fewer bits do enlarge the sets ($2.66$ to $4.46$). A full shift of the
traffic to the topics no node holds moves the miss rate at the old threshold from $.10$ to $.120$,
$.150$ and $.146$ under the silent-node rule and to $.088$ to $.110$ under FC-RAG's mean. In $270$
monitoring cells of $2{,}000$ replications, each drawing its own calibration block, the certificate
is wrong on $.0252$ of $540{,}000$ replications at a level of $.025$, and no null alarm rate exceeds
its bound (Appendix~\ref{app:x-scoreboard}). Table~\ref{tab:qasc_dose} gives detection by the exact
excess of the shifted miss rate over the allowance, pooled over every system.

\begin{table}[ht]\rev{%
\caption{QASC, evaluation split: share of shifted paths detected within $4{,}500$ queries, by the
exact excess of the post-shift miss rate over the allowance, pooled over all systems and cells.
No path with a negative excess raised an alarm in more than $.001$ of cases.}
\label{tab:qasc_dose}
\begin{center}
\footnotesize
\begin{tabular}{@{}lrrrr@{}}
\toprule
Exact excess & Paths & Cumulative aGRAPA & Five-scale mixture & Default \\
\midrule
$0$ to $.01$ & $611{,}449$ & $.033$ & $.055$ & $.057$ \\
$.01$ to $.02$ & $403{,}200$ & $.361$ & $.424$ & $.450$ \\
$.02$ to $.04$ & $443{,}456$ & $.857$ & $.882$ & $.903$ \\
$.04$ to $.08$ & $278{,}081$ & $.989$ & $.989$ & $.990$ \\
$.08$ to $.15$ & $79{,}101$ & $.982$ & $.981$ & $.982$ \\
above $.15$ & $3{,}372$ & $.931$ & $.928$ & $.930$ \\
\bottomrule
\end{tabular}
\end{center}
}%
\end{table}

The few paths missed at a large excess are mostly ones that alarmed before the change, which counts
against detection. Between the target and the allowance the monitor is silent by design. With
blocks of $1{,}000$ the 7B student's allowance is $.119$ and its mean set $2.50$; the strict order of
Appendix~\ref{app:x-cost}, which sets the allowance at the target, costs $.31$ of an option
($2.81$), and with blocks of $100$ it costs $1.2$ ($2.59$ against $3.77$).
\paragraph{Decisions that use the answer set.}
Three decisions that use the set were compared with the same decisions made from the hub's
confidence (the dithered aggregate of its best option), matched to answer the same share and to pass
on lists of the same mean length, for all $15$ systems (the three students under the four hub rules
and with no retrieval). Answering when the set is a single option and passing the set to a person
otherwise handles more traffic correctly (answered rightly, or the right answer in the list passed
on) than the matched confidence rule in all $15$ systems, by $.014$ to $.045$, and in all $45$
comparisons with the traffic shifted to the topics no node holds; its wrong automatic answers can
never exceed the miss rate the monitor tests. Among the questions it answers, neither
rule is uniformly more accurate: they differ by more than $.03$ in $6$ of the $15$ systems, the set
rule ahead in $4$. A
cascade in which the 1.5B student answers alone on a single-option set and calls the 7B student
otherwise saves $.196$ of the 7B calls at the 7B student's accuracy, and a cascade driven by
confidence does as well. The sets hold the right answer more often than lists of the best options
of the same mean size in all $15$ systems, by $.022$ to $.052$.
}%

\rev{%
\subsection{A node changes with no new certificate}
\label{app:x-node}

\paragraph{Design.}
On QASC, a node's corpus is replaced by another topic's, extended, or lost; one or more nodes drop
out; the retriever falls back to the second setting at every node; the student changes; or the bit
budget falls. The change comes at query $500$ of $5{,}000$ and the old threshold and certificate
stay in force, so Assumption~\ref{ass:epoch} fails from then on. Sixteen families (the 1.5B and 7B
students, two or eight nodes, the four hub rules) give $320$ variants of $2{,}000$ replications on
each of QASC's two splits (Appendix~\ref{app:x-qasc}), every one labeled by the exact miss rate of
the changed system at the old threshold.

\paragraph{The monitor under the violation.}
It tests the miss rate against the allowance and nothing else, which is what
Remark~\ref{rem:three-cases} predicts. It alarms on $.0004$ of $427{,}100$ benign uncertified changes
on the evaluation split ($.0005$ of $335{,}365$ on the design split) and detects $.566$ of the
harmful ones within $4{,}500$ queries ($.720$ on the design split). Detection follows the exact
excess, family by family (evaluation split; design split in brackets): $.03$ to $.06$ ($.01$ to
$.05$) under one point, $.36$ to $.51$ ($.26$ to $.63$) at one to two points, and $.976$ or more
($.93$ or more) above four, wherever a family has more than three such replications. A harm confined
to one node's topic is diluted in what the monitor tests: when one node of eight loses its corpus
under a rule that listens to the nodes, the miss rate on that node's topic rises by $.017$ to $.107$
on the evaluation split, while the pooled rate moves by at most $.014$ and stays below the allowance
($.121$) at the threshold of exact $90\%$ coverage; there the null is still true and the monitor's
silence is correct by it, while one topic is served worse. At the replications' certified
thresholds, which scatter around that one, the changed state is above its allowance on up to $.23$
of replications. On the design split the lost topic's rate rises by up to $.370$ and the pooled rate
crosses the allowance in $12$ of $36$ variants.

\paragraph{A flag that needs no label, used as a refresh trigger.}
Under the rules in which the hub sees retrieval similarities, it can watch, without labels, how
often no node's best passage reaches the gate of the silent-node rule, against a Clopper--Pearson
bound on that rate taken from the calibration block; Remark~\ref{rem:flag} bounds the chance that
this flag ever fires while nothing changes. Used as a trigger for recertifying the state in force on
a fresh block of $1{,}000$, on the evaluation split it fires after $.999$ of node losses: within a
median of $17$ to $25$ queries with two nodes (largest delay $124$), and of $179$ to $442$ when one
node of eight is lost ($99\%$ of flags within $1{,}438$ queries; $3$ of $36{,}000$ replications never
flagged). The share of post-change queries served above the allowance falls from $.716$ to $.008$
with two nodes and from $.243$ to $.012$ with eight. With nothing changed it fires on $.0024$ of
replications, against its bound of $.05$, and over all $264{,}528$ refreshes it caused the refreshed
certificates are wrong on $.0036$, at a level of $.0038$. A score-law test that needs labels (a
conformal test martingale on the true-answer scores) is a weak trigger here: it refreshes after $.31$
of two-node and $.005$ of eight-node corpus losses.

\paragraph{A second flag, on the hub's own confidence.}
Under every rule, FC-RAG's mean included, the hub computes the score of the option it ranks first;
with the dither, call it $m$. Two indicators, $m$ above a constant (the best option weaker than
usual) and $m$ below another (stronger than usual), each carry a one-sided Clopper--Pearson bound
on their rate from the calibration block, at level $.0125$, and their own betting process (the
start-time mixture), at level $.0125$. The flag is the first crossing of either, so by
Remark~\ref{rem:flag} and the union bound it fires while nothing changes with probability at most
$.05$, and together with the first flag at most $.10$. The constants are the $.8$ and $.2$
quantiles of the dithered score on the design split, fixed before the run, which repeats the first
flag's run on the same draws with this flag added. With nothing changed the flag fires on $.0020$
of $32{,}000$ replications ($.0021$ on the design split), at most $.0040$ in a family ($.0045$); the
first of the two flags on $.0044$ ($.0051$), at most $.0065$ ($.0090$). Each indicator's bound is
below its rate on $.011$ of replications, at a level of $.0125$, and the certificates issued at a
refresh at the flag are wrong on $.0039$ ($.0040$), at a level of $.0038$. A harmful change of
student is flagged on $.998$ or more of its replications under every rule, a median of $30$ to $108$
queries after it ($99\%$ within $233$), and a refresh at the flag takes the share of post-change
queries served above the allowance from $.250$ to $.007$. Under FC-RAG's mean, where the first flag
does not exist, a refresh at this one takes that share from $.976$ to $.045$ after one of two nodes
drops out ($.569$ to $.031$). It flags benign changes too, $.264$ of the benign replications of the
changed variants ($.318$), each at the cost of a block of $1{,}000$ labels and nothing in validity.
It sees little of a change of retriever (Table~\ref{tab:x_hubflag}) and of the loss of one node's
corpus among eight ($.048$ of the harmful replications; $.040$). Three other sets of constants, run
beside it on the same draws, flag $.419$ to $.445$ of the harmful replications against its $.458$,
and none dominates family by family.

\begin{table}[ht]\rev{%
\caption{The second flag by kind of change, evaluation split (design split in brackets).
\emph{Flagged}: the share of the harmful replications it flags, every rule pooled. Then the share of
post-change queries served above the allowance: under the three rules where the first flag exists,
with no refresh and with a refresh at the first flag, at the second, and at the first of the two;
under FC-RAG's mean, where only the second exists, with none and with a refresh at it.}
\label{tab:x_hubflag}
\begin{center}
\scriptsize
\setlength{\tabcolsep}{3pt}
\begin{tabular}{@{}P{0.19\linewidth}ccccccc@{}}
\toprule
 & & \multicolumn{4}{c}{Three rules} & \multicolumn{2}{c}{FC-RAG's mean} \\
\cmidrule(lr){3-6}\cmidrule(l){7-8}
Change & Flagged & None & First flag & Second & Either & None & Second \\
\midrule
Corpus lost or replaced & $.344$ ($.512$) & $.463$ ($.702$) & $.011$ ($.039$) & $.348$ ($.410$) & $.011$ ($.039$) & $.358$ ($.475$) & $.216$ ($.247$) \\
Nodes drop out & $.556$ ($.486$) & $.449$ ($.657$) & $.009$ ($.034$) & $.297$ ($.433$) & $.009$ ($.034$) & $.394$ ($.311$) & $.031$ ($.055$) \\
Corpus extended & $.002$ ($.043$) & $.020$ ($.032$) & $.020$ ($.032$) & $.020$ ($.032$) & $.020$ ($.031$) & $.015$ ($.081$) & $.015$ ($.078$) \\
Retriever & $.147$ ($.261$) & $.282$ ($.660$) & $.282$ ($.660$) & $.269$ ($.606$) & $.269$ ($.606$) & $.241$ ($.708$) & $.204$ ($.450$) \\
Student & $.999$ ($.998$) & $.250$ ($.362$) & $.250$ ($.362$) & $.007$ ($.013$) & $.007$ ($.013$) & $.250$ ($.132$) & $.007$ ($.004$) \\
Bit budget & $.717$ ($.664$) & $.013$ ($.115$) & $.013$ ($.115$) & $.004$ ($.044$) & $.005$ ($.045$) & $.023$ ($.123$) & $.017$ ($.095$) \\
Every change & $.458$ ($.522$) & $.310$ ($.491$) & $.051$ ($.116$) & $.209$ ($.286$) & $.024$ ($.065$) & $.259$ ($.319$) & $.102$ ($.150$) \\
\bottomrule
\end{tabular}
\end{center}
}%
\end{table}

\paragraph{State attestation as a refresh trigger.}
The same run, on the same draws, was repeated with a third trigger: every node reports the digest
of its corpus version, retriever, model and encoder with its message (Remark~\ref{rem:attest}; the
supplement has a reference implementation), the hub compares it with the digest recorded at
certification, and a declared change is answered at the first query it affects by certifying the
state in force again on a fresh block of $1{,}000$; that first query is still served at the old
threshold. The digests name a changed node in each of the $304$ changed variants and in none of the
$16$ unchanged ones. Over every change and every rule the share of post-change queries served above
the allowance falls from $.298$ with no refresh to $.004$ (at most $.005$ in any rule and kind of
change), against $.182$ under a refresh at the second flag, and $.051$ and $.024$ at the first flag
and at the first of the two under the three rules where those exist; after a change of retriever,
which the flags rarely see, it falls from $.19$ to $.35$ to between $.002$ and $.005$. The refreshed
certificates are wrong on $2{,}214$ of $608{,}000$ refreshes, at a level of $.0038$ ($2{,}310$
expected), and in no cell does an alarm with every certificate correct and no query above the
allowance exceed $\edelta$ (largest $.004$). The cost is a block of $1{,}000$ labels at every
declared change, benign ones included ($.70$ of the refreshes), and sets larger by about half an
option after a corpus is lost or replaced or a node drops out ($3.75$ to $4.20$ and $3.62$ to $4.16$
options), the price of the coverage the old threshold had lost. On the design split the share above
the allowance falls from $.448$ to $.004$. Attestation sees only what nodes declare: a node that
reports a stale digest is a Byzantine fault, outside the scope, and a change in the question stream
alters no component, which is what the monitor and the flags are for.

\paragraph{Limits.}
The first flag does not exist under FC-RAG's mean, where the hub sees neither silence nor
similarities, and it sees a change of retrieval, not a change of harm: after a change of retriever
it triggers on $1$ and $6$ of $12{,}000$ replications on the evaluation split while $.285$ and
$.279$ of them are above their allowance at the horizon ($1$ and $3$, with $.68$ and $.64$, on the
design split), and it does not see a change of student or of bit budget. The second flag sees a
change of student under every rule and a fall of the bit budget under the silent-node and
most-similar rules (partly under FC-RAG's mean, not under the two most similar nodes), but a change
of retriever on only $.147$ of the harmful replications ($.261$ on the design split): under FC-RAG's
mean with two nodes, a refresh at it leaves $.210$ of post-change queries above the allowance after
a change of retriever, against $.233$ with no refresh ($.504$ against $.997$ on the design split).
Each flag costs a second error budget and $1{,}000$ labels for each trigger; a refresh at the first
adds $.3$ to $.9$ of an option to the set, and one at the second changes it by $-1.7$ to $+2.8$
options, depending on the change. On the design split with eight nodes the first flag is slower:
$.035$ of replications are never flagged within $4{,}500$ queries, and $.014$ are still above the
allowance at the horizon. State attestation is the complementary check for declared changes,
retriever changes included. On MMLU-Pro no change of corpus, nodes, retriever or bit budget moves
the miss rate by more than $.03$, because retrieval does not help there, while a change of student
does (Appendix~\ref{app:x-first}).
}%

\rev{%
\subsection{Refreshes chosen from history on real scores}
\label{app:x-refresh}

\paragraph{Design.}
A controller walks a library of $36$ QASC states: three students; three retrieval policies, each a
retrieval setting paired with a hub rule (the second setting with the silent-node rule, the lead
setting with it, the lead setting with the two most similar nodes), so that a change of retriever
and a change of hub rule are both refreshes; the nodes' own corpora, or nodes 1 and 2 extended by
topics 9 and 10; and $12$ or $6$ bits. It starts from the 1.5B student with the second setting.
Three triggers start an attempt. \emph{Evidence}: the never-reset wealth has grown by a warning
factor since the last refresh; the controller moves to the state with the lowest miss rate the path
has seen, else to the next stronger student, else certifies the state in force again.
\emph{Clock}: an exogenous clock, standing for new retriever and corpus versions, moves the retrieval
policy and the corpus in turn. \emph{Cost}: after a quiet stretch it tries a cheaper state and
discards the attempt, still charged, if the block shows larger sets. Every block holds $1{,}000$
questions drawn after the state is chosen, from the law then in force; every attempt takes the next
share of the summable schedule; one wealth process runs throughout. Arms: a \emph{natural null}, in
which the traffic keeps the calibration law, and a \emph{boundary null}, in which every new epoch is
put exactly at its allowance ($20{,}000$ paths each); a drift to a topic no node holds that a refresh
on the shifted law repairs, the same drift with \emph{stale} blocks, still drawn from the
calibration law, and a late drift; then two null arms of $100{,}000$ queries with about $130$
attempts a path in the natural null and $70$ at the boundary ($4{,}000$ paths each).

\paragraph{Both budgets hold.}
Table~\ref{tab:refresh_null} gives the null arms. Certificates fail at their levels attempt by
attempt (evaluation split: $.0152$ at a level of $.0152$ at the first attempt, $.0039$ at $.0038$ at
the second, $.0012$ at $.0017$ at the third; in the long natural null, $62$ of $4{,}000$ at the first
attempt and $3$ from the twentieth on, where their levels give $2.6$; in the long boundary null, $65$
and none), and within every trigger and component: in the boundary null the $3{,}421$
evidence-triggered attempts that change the retrieval policy or the corpus have $3$ wrong
certificates where their levels give $3.1$.

\begin{table}[ht]\rev{%
\caption{Refreshes chosen from history, null arms, QASC's evaluation and design splits: attempts per
path; the share of paths on which some attempted certificate is wrong (budget $.025$); with an alarm
although every deployed certificate is correct (budget $.025$); and with either (budget $.05$).
Default monitor; $95\%$ Wilson intervals in the supplement. Rows \emph{no clock}: every refresh
chosen from the path's history, the threshold rule included (see text).}
\label{tab:refresh_null}
\begin{center}
\footnotesize
\begin{tabular}{@{}llcccc@{}}
\toprule
Split & Arm & \shortstack{Attempts\\per path} & \shortstack{Certificate\\wrong} & \shortstack{Alarm, certificates\\correct} & Either \\
\midrule
Evaluation & natural null & $7.1$ & $.0222$ & $.0007$ & $.0229$ \\
Evaluation & boundary null & $5.3$ & $.0209$ & $.0167$ & $.0376$ \\
Design & natural null & $7.1$ & $.0232$ & $.0006$ & $.0237$ \\
Design & boundary null & $5.3$ & $.0223$ & $.0103$ & $.0327$ \\
Evaluation & long natural null & $130$ & $.0245$ & $.0010$ & $.0255$ \\
Evaluation & long boundary null & $69$ & $.0272$ & $.0182$ & $.0455$ \\
Design & long natural null & $131$ & $.0225$ & $.0005$ & $.0230$ \\
Design & long boundary null & $73$ & $.0240$ & $.0118$ & $.0357$ \\
\addlinespace[0.3ex]
Evaluation & natural null, no clock & $6.5$ & $.0228$ & $.0008$ & $.0237$ \\
Evaluation & boundary null, no clock & $5.9$ & $.0236$ & $.0156$ & $.0391$ \\
Design & natural null, no clock & $6.5$ & $.0215$ & $.0011$ & $.0227$ \\
Design & boundary null, no clock & $5.8$ & $.0213$ & $.0097$ & $.0309$ \\
Evaluation & long natural null, no clock & $86$ & $.0238$ & $.0003$ & $.0240$ \\
Evaluation & long boundary null, no clock & $98$ & $.0217$ & $.0175$ & $.0393$ \\
Design & long natural null, no clock & $110$ & $.0248$ & $.0008$ & $.0255$ \\
Design & long boundary null, no clock & $100$ & $.0248$ & $.0097$ & $.0345$ \\
\bottomrule
\end{tabular}
\end{center}
}%
\end{table}

\paragraph{What breaks it.}
For $30$ ordered pairs of states, $20{,}000$ blocks each, the legal order (the block scored by the
incoming state) is wrong on $15{,}130$ of $600{,}000$ blocks at a level of $.025$ (evaluation split;
$14{,}985$ on the design split), the same for every pair. Calibrating the incoming state on the
outgoing state's scores is wrong on every block when a stronger student comes in (exact miss rate
$.229$ to $.303$ against a target of $.10$), on $.58$ to $.74$ of blocks when one student moves from
the lead retrieval setting to the second (and on at most $.0002$ the other way), and never when a
weaker student comes in, where the inherited threshold keeps nearly every option. Choosing the best
of $3$, $9$, $18$ or $36$ states on the block that certifies it is wrong on $.0280$, $.0435$, $.0585$
and $.0395$; redrawing until the estimated set is small enough is wrong on $.1615$ uncharged and
$.0215$ with every try charged. In a two-attempt swap between students, both blocks redrawn on each
of $20{,}000$ paths per cell, each certificate of the legal order is wrong on $.011$ to $.014$ at a
level of $.0125$ and the alarm with both certificates correct fires on at most $.0146$; the illegal
order with a stronger student coming in is wrong on every path.

\paragraph{What the controller did, and its scope.}
The retrieval policy changes on $.98$ of paths and the corpus on $.91$, and both are certified
refreshes on real scores; but in the null arms their \emph{timing} comes from the exogenous clock
($80{,}276$ of $122{,}418$ attempts in the natural null of the evaluation split). What the path's
history decides is the time and target of the evidence refreshes (on $.014$ of paths in the natural
null, $.215$ at the boundary, and $.35$ to $.67$ under drift), whether a cheaper state is kept (the
cost rule certified and discarded $41{,}818$ attempts in the natural null, each charged), and the
state each refresh starts from. In these arms the threshold rule itself is never varied; the
threshold is refreshed from history only by certifying the state in force again. Calibration on
separately labeled blocks is expensive: about $7{,}100$ labeled questions per path over $6{,}000$
queries. Blocks taken from the labeled traffic remove most of that cost, since beyond a path's first
block they need no label that monitoring does not already collect, at the price of a latency of one
block per refresh and stationarity over it (Appendix~\ref{app:x-cost}); this controller was not run
with them.

\paragraph{Every refresh chosen from history, the threshold rule included.}
With the clock removed, a \emph{review} after $150$ misses observed since the last deployed refresh
reads the epoch's record and changes the threshold rule (to or from the strict order of
Appendix~\ref{app:x-cost}, by the epoch's observed miss rate), else the corpus (when the epoch's
misses fall more often than its block's on queries for which no node's best passage reached the
gate, which the hub sees without a label), else the retrieval policy, never the first two at two
reviews in a row; no refresh's time or target is then drawn outside the path. In the natural null of the evaluation split the reviews change the
retrieval policy on $.913$ of paths, the corpus on $.797$ and the threshold rule on $.518$ ($.908$,
$.813$ and $.517$ on the design split); both budgets hold on both splits over $6{,}000$ and
$100{,}000$ queries (rows \emph{no clock} of Table~\ref{tab:refresh_null}), and certificates are wrong
at their levels within every trigger, kind of change and order, though the boundary null's
evidence-triggered attempts, taken together, run high ($18$ wrong against $10.15$). With stale blocks
after a drift the default detects $.215$ of the paths above their allowance against $.021$ for its
summable restart, but $.053$ against $.070$ on the design split. The controller remains one
history-dependent rule, fixed in advance.

\paragraph{Drift.}
A move of the traffic to a topic no node holds is mild here (mean exact excess $-.004$ to $+.041$
by topic). When a refresh on the shifted traffic is allowed, the controller refreshes on $.985$ of
paths, a median of $582$ queries after the change, and the share of queries served above the
allowance is $.103$ against $.435$ with stale blocks. With stale blocks the monitor detects $.400$ of
the paths above their allowance at the change, and $.052$ when the change comes after $20{,}000$
queries; these excesses of two to four points
take thousands of queries to see.

\paragraph{Detection after many refreshes.}
The same drift, with stale blocks, after $1{,}000$, $20{,}000$, $50{,}000$ or $90{,}000$ queries, that
is after about $2$, $24$, $64$ and $118$ certificate attempts, each followed by $6{,}000$ queries
($2{,}000$ paths an arm, both splits; Table~\ref{tab:x_late}). Both budgets hold in every arm: some
attempted certificate wrong on $.020$ to $.026$ of paths (none above $.025$ at $p<.01$), an alarm
before any query above the allowance on at most $.002$. The allowance in force at the change grows
from $.124$ to $.150$, and at later changes fewer paths are above it after the drift. Of those that
are, the default detects $.41$ after about two attempts and $.03$ to $.07$ after $24$ to $118$; the
cumulative aGRAPA bettor detects $.24$ and then none. At two to four points of excess over the
allowance the fall is the same, $.44$, then $.14$, $.06$ and $.05$; in the other bins of excess it
comes mainly between about $2$ and $24$ attempts. This is the cost of a late change that
Proposition~\ref{prop:delaybound} bounds: the term $2\log(\nu/L+3)$ of its delay bound grows from
about $3.9$ nats at a change after $1{,}000$ queries to $11.8$ after $90{,}000$, while an excess of
one or two points earns at most about $.0004$ to $.0015$ nats a query; the number of attempts, the
elapsed time and the allowance move together here and are not separated. On the design split, with
the same arms, detection does not fall: $.20$ after about two attempts and $.31$ to $.34$ after $24$
to $118$, the paths above their allowance having larger excesses there at the later changes, and at
two to four points of excess $.19$, $.14$, $.16$ and $.21$; both budgets hold there too. With a
refresh on the shifted law allowed, the state at the
horizon is within its allowance on $.998$ and $.999$ of paths after about $2$ and $117$ attempts,
but the repair comes later, a median of $547$ and $976$ queries after the change.

\paragraph{Why the splits differ.}
A post hoc analysis of the per-path records, not registered in advance, explains the difference. For
each path above its allowance at the change, the state in force at every query of the $6{,}000$ after
it was rebuilt from the attempt records, with its allowance and its exact excess under the shifted
law; the rebuild reproduces the run's per-path counts exactly. The \emph{evidence} of a path is the
sum, over the queries served above the allowance, of the Kullback--Leibler divergence of the miss
rate from the allowance: the log-wealth that a bettor knowing the miss rate would expect to gain
(Table~\ref{tab:x_late}). The splits differ in this evidence, not in how the monitor uses it. After
$24$ to $118$ attempts its median is $2.8$ nats on the evaluation split and $12.9$ on the design
split, and $.60$ and $.22$ of the paths carry under $4$ nats, about what an alarm needs at
$\edelta=.025$ ($\log40\approx3.7$). At matched evidence the splits detect alike: $.17$ and $.20$
between $8$ and $16$ nats; at the design split's rate for each band of evidence ($0$--$2$, $2$--$4$,
$4$--$8$, $8$--$16$ and $16$--$32$ nats) the evaluation split's paths would be detected on $.036$
(observed $.049$), and at the evaluation split's rates the design split's paths below $32$ nats on
$.251$ (observed $.247$). The evidence differs with the populations: the four topics no node holds
have $358$ development questions against $2{,}217$ training questions, and at their thresholds of
exact $90\%$ coverage the $36$ states' miss rates after the move range from $0$ to $.33$ on the design
split but only from $.07$ to $.19$ on the evaluation split. At the change after $1{,}000$ queries the
evaluation split carries the more evidence ($6.6$ against $5.4$ nats) and detects more.

\begin{table}[ht]\rev{%
\caption{Detection after many refreshes on real scores (QASC, both splits; the controller of this
section; all traffic moves to the topics no node holds, stale blocks). \emph{Above}: the share of
paths whose state is above its allowance at the change. \emph{Evidence}: of those, the median sum of
the Kullback--Leibler divergence of the miss rate from the allowance over the queries served above it
in the $6{,}000$ after the change, in nats (post hoc; see text). \emph{Detected}: of those, the share
with an alarm within $6{,}000$ queries, with the median delay; last column at an exact excess of two
to four points.}
\label{tab:x_late}
\begin{center}
\footnotesize
\setlength{\tabcolsep}{4pt}
\begin{tabular}{@{}lccccccc@{}}
\toprule
Change after & \shortstack{Attempts\\before} & \shortstack{Allowance\\at change} & Above &
\shortstack{Evidence,\\median} & \shortstack{Detected,\\default} & \shortstack{Detected,\\aGRAPA} &
\shortstack{Default,\\excess $.02$--$.04$} \\
\midrule
\multicolumn{8}{@{}l}{\emph{Evaluation split}} \\
$1{,}000$ & $1.8$ & $.124$ & $.736$ & $6.6$ & $.407$ ($1{,}044$) & $.236$ ($1{,}940$) & $.439$ \\
$20{,}000$ & $24.4$ & $.142$ & $.408$ & $3.2$ & $.066$ ($3{,}474$) & $0$ & $.140$ \\
$50{,}000$ & $64.0$ & $.147$ & $.330$ & $2.6$ & $.041$ ($3{,}137$) & $0$ & $.056$ \\
$90{,}000$ & $117.8$ & $.150$ & $.276$ & $2.3$ & $.034$ ($4{,}102$) & $0$ & $.051$ \\
\addlinespace[0.4ex]
\multicolumn{8}{@{}l}{\emph{Design split}} \\
$1{,}000$ & $1.8$ & $.124$ & $.448$ & $5.4$ & $.199$ ($1{,}558$) & $.069$ ($2{,}116$) & $.187$ \\
$20{,}000$ & $24.5$ & $.142$ & $.363$ & $12.3$ & $.342$ ($1{,}098$) & $0$ & $.136$ \\
$50{,}000$ & $64.0$ & $.147$ & $.335$ & $11.5$ & $.311$ ($1{,}234$) & $0$ & $.164$ \\
$90{,}000$ & $117.7$ & $.150$ & $.263$ & $15.4$ & $.322$ ($1{,}899$) & $0$ & $.208$ \\
\bottomrule
\end{tabular}
\end{center}
}%
\end{table}
}%

\rev{%
\subsection{Expectations stated in advance}
\label{app:x-failed}

Every new experiment in Appendices~\ref{app:x-restart} to~\ref{app:x-refresh} was specified before
it ran (on QASC, before its run on the evaluation split), with its design, primary outcomes and, for
most, expected outcome; the supplement holds this record, with every change made before a run and
every outcome. No expectation about validity failed: none of the $13{,}536$ cells with a stated bound
exceeds it at the $1\%$ level (Appendix~\ref{app:x-scoreboard}). Those not met concern effect sizes,
delays and orderings; the most material follow (QASC numbers on its evaluation split unless stated).
\begin{itemize}
\item On QASC, top-1 under the silent-node rule is $.671$, $.579$ and $.392$ (7B,
1.5B, 0.5B), against $.87$--$.91$, $.72$--$.77$ and $.49$--$.55$ expected from the design split; its
gain over no retrieval is $.065$ for the 7B student, against at least seven points
(Appendix~\ref{app:x-qasc}).
\item After a node's loss with two nodes, FC-RAG's mean was expected within $.02$ of the target
$.10$ and the most similar node alone at $.15$ or more; the miss rates are $.108$ to $.153$ and
$.122$ to $.139$, and the two most similar nodes are harmed most ($.156$ to $.189$;
Appendix~\ref{app:x-node}).
\item No valid restart was expected to detect more than the never-reset process; off the refresh
grid a capped restart does in $4$ of $48$ comparisons with certificates and $6$ of $48$ with a known
allowance, by $8$ to $24$ paths of $5{,}000$, the difference in alarms before the change
(Appendix~\ref{app:x-restart}).
\item The first label-free flag, expected within about a hundred queries of a change to a node's
corpus, takes medians of $179$ to $442$ when one node of eight changes; a refresh at the second,
expected to at least halve the queries above the allowance after a change of retriever under FC-RAG's
mean, takes their share from $.233$ to $.210$ (Appendix~\ref{app:x-node}).
\item With stale blocks after a drift, most paths above their allowance were expected to be
detected; the default detects $.400$, and $.052$ when the drift comes after $20{,}000$ queries
($.213$ and $.362$ on the design split; Appendix~\ref{app:x-refresh}).
\end{itemize}
}%
 
\section{Reproducibility, compute and licenses}
\label{app:repro}

\paragraph{Theory.}
Every assumption is stated in Section~\ref{sec:setup}, Assumption~\ref{ass:epoch} and
Appendix~\ref{app:information}, with the covered and excluded regimes in
Table~\ref{tab:scope}. Lemma~\ref{lem:familywise}, Theorem~\ref{thm:refresh}
and Proposition~\ref{prop:beta} are proved in Appendix~\ref{app:proofs};
Corollary~\ref{cor:design} in
Appendix~\ref{app:corollaries}. No proof step is deferred and every constant
is named.

\paragraph{Experimental provenance.}
Every number in Section~\ref{sec:evidence} traces to a completed run, and the
rendered panels are produced by scripts that read the retained numerical
summaries directly, so no plotted value is recomputed at figure time. Every
retained summary comes from a completed run whose protocol was fixed in
advance; the six single-guarantee studies of Appendix~\ref{app:eseries} are
CPU-only and were run from pinned configurations, and their runner and
summaries are included in the artifact bundle. The modular studies
(Appendices~\ref{app:refresh-stress}--\ref{app:composed},
\ref{app:baselines}--\ref{app:modular}) run on CPU; each fixes its seeds in its
configuration file, and the tables above give its stream counts, horizons and budgets. The Qwen fixed-state
studies of Appendix~\ref{app:qwen} are cache-only replays: the base cache holds $17{,}744$ entries and the
confirmation cache $1{,}472$ entries ($1{,}200$ fresh target vectors plus
$272$ prespecified controls). Both were verified by SHA-256 before and after
use, opened read-only and never modified, and incurred zero cache misses and zero
model forward calls; the model swap of Appendix~\ref{app:liverefresh}
additionally executes one trajectory per configuration with the cache
bypassed. In the upstream cache audit, all $272$ prespecified
controls ($2{,}720$ scalar option scores) agreed exactly between Tesla
V100-PCIE-16GB and V100-PCIE-32GB scoring; the hardware claim is exact
agreement on those controls, and the checks verify cache integrity rather than
constituting a new model evaluation.
\rev{The new experiments of Appendices~\ref{app:x-restart}--\ref{app:x-refresh} were each specified
in a protocol file before they ran (on QASC, before their runs on the evaluation split), most with an expected outcome; the protocol, with every
outcome, is in the artifact bundle, and Appendix~\ref{app:x-failed} summarizes it. Every replication draws its own calibration blocks and every miss rate is
exact, each run is seeded from its configuration, and each summary is generated from its run's
results file, most of them recounting their tables from the per-path records.}

\paragraph{Compute.}
The synthetic sweeps run on commodity CPU nodes; the largest studies,
the predictable-allowance and sentinel sweeps, dominate the cost and remain
within a few hundred CPU-hours. Scoring the Qwen caches used single-GPU
V100-family nodes on a shared academic cluster; the cache studies replay
those scores on CPU, except the cache-bypassed model-swap trajectories of
Appendix~\ref{app:liverefresh}, which recompute them. The GPT-2 frozen-system study used single-GPU
nodes (mostly RTX~6000, with one single-seed bandwidth sweep on an A100) for
at most $60$ GPU-hours. No proprietary cloud or multi-node job was used.
The monitoring layer itself costs $24.0\,\mu$s per update on one CPU core
(Appendix~\ref{app:modular}). \rev{The score tables of the population studies were computed on single GPUs of the same
cluster (Quadro RTX~6000, V100, A100 and H200 cards, minutes to under an hour per job); every
experiment that replays them runs on CPU, on the cluster's CPU nodes or on a four-core laptop
(several on both, with equal results), in minutes to a few hours per run.}

\paragraph{Artifacts.}
The artifact bundle contains the \texttt{anytime\_fcrag} package, one
YAML configuration per experiment, the experiment runners and retained result JSONs,
the three figure-generation scripts, and a non-interactive verifier
that checks the retained summaries against their recorded SHA-256 digests.
Because the bundle is anonymized, its copies differ from the production
artifacts in filesystem paths, package metadata and the recording host; the
digests are recorded over those copies and attest to the bundle. No reported
quantity differs: across the retained summaries every numeric value is
identical to the production file, and the only changes are redacted strings.
\rev{The bundle also redistributes the three model-score caches behind the
Qwen2.5 runs, with a manifest row for every cached vector; the per-path results of the model
swap; and the score tables, code, configurations and retained results of every population study,
with the protocol and scripts that replay the results from these files alone on one CPU core.
Replayed from a copy of the bundle, the Qwen grid reproduces its retained summary (integers and
strings identical, floating-point values within a relative $5\times10^{-11}$) and the second
benchmark's main study reproduces on either split, each in minutes. Regenerated scores are
bit-exact only on the same GPU family with the same batching ($1{,}800$ of $1{,}800$ cached
vectors); across card families they differ by $.0025$ to $.0028$ per score on average, with the
same top choice on $99.5\%$ or more. No benchmark text is redistributed: passages are identified
by digest or corpus index, at pinned dataset revisions.}
A public repository link will be added when the paper is published.

\paragraph{Licenses for existing assets.}
Benchmarks are used as-is under their original licenses\rev{, and no benchmark text is
redistributed (the artifact bundle holds identifiers, digests and model scores only):
QASC (CC BY 4.0;~\citealp{khot2020qasc})}, MMLU-Pro (MIT;~\citealp{wang2024mmlupro}),
MMLU (MIT;~\citealp{hendrycks2021mmlu}), DBpedia ontology
classification (CC BY-SA 3.0 / GFDL;~\citealp{lehmann2015dbpedia}) and AG News
topic classification (a research-use academic corpus
released by the original maintainer;~\citealp{zhang2015agnews}). Pretrained checkpoints are used as-is
for inference: \texttt{Qwen/Qwen2.5-1.5B-Instruct} (Apache~2.0),
\texttt{Qwen/Qwen2.5-0.5B-Instruct} (Apache~2.0),
\rev{\texttt{Qwen/Qwen2.5-7B-Instruct} (Apache~2.0),}
\texttt{openai-community/gpt2} (MIT),
\texttt{distilbert/distilgpt2} (Apache~2.0),
\texttt{sentence-transformers/all-MiniLM-L6-v2} (Apache~2.0) \rev{and
\texttt{sentence-transformers/all-mpnet-base-v2} (Apache~2.0); BM25 is implemented in the code}. All datasets and
models are used in compliance with their terms.

\paragraph{Open questions.}
One concrete open problem is worth naming, because it is the natural
strengthening of Corollary~\ref{cor:design}(f): a communication-optimality
oracle inequality bounding a certified controller's expected cumulative uplink
against the best validity-preserving controller in the same class, up to an
$o(T)$ overhead. We do not attempt it. Appendix~\ref{app:liverefresh}
executes a swap of the deployed student\rev{, and Appendices~\ref{app:x-first} and~\ref{app:x-refresh} refresh the retriever, the hub rule, the corpus, the bit budget and the threshold rule from each path's history on cached scores of real models}; what remains untested is \rev{a deployment over a real network rather than a cached replay}. Beyond those lie observation models with outcome-dependent missing labels, formal privacy accounting for the score messages, and a decision-theoretic analysis of intervention and communication policies. None of these is needed for the guarantee proved here.

\end{document}